\pdfoutput=1
\documentclass{article}

	\usepackage{hyperref}

\hypersetup{
    colorlinks=true,
    linkcolor=red,
    citecolor=blue,
    filecolor=magenta,      
    urlcolor=cyan,
}

    \usepackage[T1]{fontenc}
	\usepackage[utf8]{inputenc}
	\usepackage[english]{babel}
	\usepackage{amsfonts}
	\usepackage{amssymb}
	\usepackage{amsthm}
	\usepackage{enumitem}
	\usepackage{bbm}
	\usepackage{mathtools}
	\usepackage{booktabs}	
	\usepackage{setspace}
	\usepackage{graphicx}
	\usepackage{fullpage}
	
	\usepackage{soul}
	\usepackage{xspace}
	\usepackage{csquotes}
	\usepackage{changepage}
	\usepackage{verbatim}
    \usepackage{placeins}
    
\usepackage{tikz}
\usepackage{subcaption}
\usepackage{textcomp}
\usepackage[ruled,vlined]{algorithm2e}
\usepackage{caption}
\usepackage{subcaption}
\AtBeginDocument{%
  \providecommand\BibTeX{{%
    \normalfont B\kern-0.5em{\scshape i\kern-0.25em b}\kern-0.8em\TeX}}}

\usepackage[suppress]{color-edits}
\addauthor[Amanda]{am}{blue}
\addauthor[Alex]{ac}{purple}
\addauthor[Ashesh]{as}{green}

\newcommand{\p}{\mathbb{P}}
\newcommand{\e}{\mathbb{E}}

\newcommand{\reals}{\mathbb{R}}
\newcommand{\abs}[1]{\left| #1 \right|}

\newcommand{\cD}{\mathcal{D}}
\newcommand{\cE}{\mathcal{E}}
\newcommand{\cF}{\mathcal{F}}
\newcommand{\cG}{\mathcal{G}}
\newcommand{\cH}{\mathcal{H}}
\newcommand{\cS}{\mathcal{S}}
\newcommand{\cU}{\mathcal{U}}
\newcommand{\cX}{\mathcal{X}}
\newcommand{\cY}{\mathcal{Y}}
\newcommand{\cZ}{\mathcal{Z}}

\newcommand{\indep}{\perp \!\!\!\perp}
\newcommand{\loss}{\operatorname{loss}}
\newcommand{\cost}{\operatorname{cost}}

\newcommand{\disp}{\operatorname{disparity}}
\newcommand{\best}{\operatorname{Best}}

\newcommand{\Qhat}{\hat{Q}}
\newcommand{\costhat}{\widehat{\cost}}
\newcommand{\disphat}{\widehat{\disp}}

\newcommand{\epsilonhat}{\hat \epsilon}
\newcommand{\hnuma}{\hat{\omega}_1}
\newcommand{\hnum}{\hat{\omega}}
\newcommand{\hnumo}{\hat{\omega}_0}
\newcommand{\hdena}{\hat{\bar{\omega}}_1}
\newcommand{\hdeno}{\hat{\bar{\omega}}_0}
\newcommand{\hden}{\hat{\bar{\omega}}}
\newcommand{\numa}{\omega_1}
\newcommand{\num}{\omega}
\newcommand{\den}{\bar{\omega}}
\newcommand{\numo}{\omega_0}
\newcommand{\dena}{\bar{\omega}_1}
\newcommand{\deno}{\bar{\omega}_0}

\newtheorem{theorem}{Theorem}

\newtheorem{lemma}{Lemma}
\newtheorem{assumption}{Assumption}
\theoremstyle{condition}

\newtheorem{definition}{Definition}

\newcommand\numberthis{\addtocounter{equation}{1}\tag{\theequation}}

\newcommand{\fairs}{FaiRS\xspace}

\usepackage{arxiv}
\usepackage{url}            
\usepackage{nicefrac}       
\usepackage{microtype}      

\title{Characterizing Fairness Over the Set of Good Models Under Selective Labels}
\date{\today}
\author{
    Amanda Coston\thanks{ Carnegie Mellon University, Heinz College and Machine Learning Department: \texttt{acoston@andrew.cmu.edu} }
    \And
    Ashesh Rambachan\thanks{ Harvard University, Department of Economics: \texttt{asheshr@g.harvard.edu} } 
    \And
    Alexandra Chouldechova\thanks{ Carnegie Mellon University, Heinz College: \texttt{achoulde@andrew.cmu.edu} }
}

\begin{document}
\maketitle
\begin{abstract}
Algorithmic risk assessments are used to inform decisions in a wide variety of high-stakes settings. Often multiple predictive models deliver similar overall performance but differ markedly in their predictions for individual cases, an empirical phenomenon known as the ``Rashomon Effect.'' These models may have different properties over various groups, and therefore have different predictive fairness properties. We develop a framework for characterizing predictive fairness properties over the set of models that deliver similar overall performance, or ``the set of good models.'' Our framework addresses the empirically relevant challenge of selectively labelled data in the setting where the selection decision and outcome are unconfounded given the observed data features. Our framework can be used to 1) replace an existing model with one that has better fairness properties; or 2) audit for predictive bias. We illustrate these uses cases on a real-world credit-scoring task and a recidivism prediction task.
\end{abstract}
\section{Introduction} \label{sec:intro}

Algorithmic risk assessments are used to inform decisions in high-stakes settings such as health care, child welfare, criminal justice, consumer lending and hiring \cite{CaruanaEtAl(15), ChouldechovaEtAl(18), kleinberg2018human, FusterEtAl(20), RaghavanEtAl(20)}.
Unfettered use of such algorithms in these settings risks disproportionate harm to marginalized or protected groups \cite{BarocasSelbst2016, dastin_2018, vigdor_2019}. 
As a result, there is widespread interest in measuring and limiting predictive disparities across groups.

The vast literature on algorithmic fairness offers numerous methods for learning anew the best performing model among those that satisfy a chosen notion of predictive fairness (e.g. \cite{ZemelEtAl(13)}, \cite{agarwal2018reductions}, \cite{AgarwalEtAl(19)-FairRegression}). 
However, for real-world settings where a risk assessment is already in use, practitioners and auditors may instead want to assess disparities with respect to the current model, which we term the \textit{benchmark model}. For example, the benchmark model for a bank may be an existing credit score used to approve loans. The relevant question for practitioners is: Can we improve upon the benchmark model in terms of predictive fairness with minimal change in overall accuracy?


We explore this question through the lens of the ``Rashomon Effect,'' a common empirical phenomenon whereby multiple models perform similarly overall but differ markedly in their predictions for individual cases \cite{Breiman(01)}. These models may perform differently over various groups, and therefore have different predictive fairness properties \cite{ChouldechovaGsell(17)}. We propose an algorithm, Fairness in the Rashomon Set (FaiRS), to probe predictive fairness properties over the set of models that perform similarly to a chosen benchmark model. We refer to this set as \textit{the set of good models} \cite{DongRudin(20)}. \fairs is designed to efficiently answer the following questions: What are the range of predictive disparities that could be generated over the set of good models? What is the disparity minimizing model within the set of good models?

A key empirical challenge in domains such as credit lending is that outcomes are not observed for all cases \cite{lakkaraju2017selective, kleinberg2018human}. 
This \emph{selective labels problem} is particularly vexing in the context of assessing predictive fairness. Our framework addresses the challenges of selectively labelled data in contexts where the selection decision and outcome are unconfounded given the observed data features.

Our methods are useful for legal audits of disparate impact. In various domains, decisions that generate disparate impact must be justified by ``business necessity" \cite{civil-rights-act, ECOA, BarocasSelbst2016}. For instance, financial regulators investigate whether credit lenders could have offered more loans to minority applicants without affecting default rates \cite{gillis}.
Our methods provide one possible formalization of the business necessity criteria.
An auditor can use \fairs to assess whether there exists an alternative model that reduces predictive disparities without compromising performance relative to the benchmark model. If possible, then it is difficult to justify the benchmark model on the grounds of business necessity.

Our methods can also be a useful tool for decision makers who want to improve upon an existing model. A decision maker may use \fairs to search for a prediction function that reduces predictive disparities without compromising performance relative to the benchmark model.
We emphasize that the effective usage of our methods requires careful thought about the broader social context surrounding the setting of interest \cite{SelbstEtAl(19), HolsteinEtAl(19)}.

\textbf{Contributions}: We (1) develop an algorithmic framework, Fairness in the Rashomon Set (FaiRS), to investigate predictive disparities over the set of good models; (2) provide theoretical guarantees on the generalization error and predictive disparities of \fairs [\textsection~\ref{section: reductions approach}]; (3) propose a variant of FaiRS that addresses the selective labels problem and achieves the same guarantees under oracle access to the outcome regression function [\textsection~\ref{section: selective labels}]; (4) use \fairs on a selectively labelled credit-scoring dataset to build a model with lower predictive disparities than the benchmark model [\textsection~\ref{section: consumer lending application}]; and (5) use \fairs to audit the COMPAS risk assessment, finding that it generates larger predictive disparities between black and white defendants than any model in the set of good models [\textsection~\ref{section: recidivism application}]. 

\section{Background and Related Work} \label{sec:lit}

\subsection{Rashomon Effect}
In a seminal paper on statistical modeling, \cite{Breiman(01)} observed that often a multiplicity of good models achieve similar accuracy by relying on different features, which he termed the ``Rashomon effect.” 
These models  may differ along key dimensions, and recent work considers the implications of the Rashomon effect for model simplicity, interpretability, and explainability
\cite{Rudin(19), FisherEtAl(19), MarxEtAl(19), DongRudin(20), SemenovaEtAl(20)}. 
Focusing on algorithmic fairness, 
we develop techniques to investigate the range of predictive disparities that may be generated over the set of good models. 

\subsection{Fair Classification and Fair Regression}

An influential literature on fair classification and fair regression constructs prediction functions that minimize loss subject to a predictive fairness constraint chosen by the decision maker \cite{Dwork2012, ZemelEtAl(13), hardt2016equality, MenonWilliamson(18), DoniniEtAl(18), agarwal2018reductions, AgarwalEtAl(19)-FairRegression, ZafarEtAl(19)}. In contrast, we construct prediction functions that minimize a chosen measure of predictive disparities subject to a constraint on overall performance. This is useful when decision makers find it difficult to specify acceptable levels of predictive disparities, but instead know what performance loss is tolerable. It may be unclear, for instance, how a lending institution should specify acceptable differences in credit risk scores across groups, but the lending institution can easily specify an acceptable average default rate among approved loans. Similar in spirit to our work, \cite{ZafarEtAl(19)} provide a method for selecting a classifier that minimizes a particular notion of predictive fairness, ``decision boundary covariance,'' subject to a performance constraint. Our method applies more generally to a large class of predictive disparities and covers both classification and regression tasks.

While originally developed to solve fair classification and fair regression problems, we show that the ``reductions approach'' used in \cite{agarwal2018reductions, AgarwalEtAl(19)-FairRegression} can be suitably adapted to solve general optimization problems over the set of good models. This provides a general computational approach that may be useful for investigating the implications of the Rashomon Effect for other model properties. 

In constructing the set of good models with comparable performance to a benchmark model, our work bears resemblance to techniques that ``post-process'' existing models.
Post-processing techniques typically modify the predictions from an existing model to achieve a target notion of fairness \cite{hardt2016equality, PleissEtAl(17), KimEtAl(19)}. By contrast, our methods only use the existing model to calibrate the performance constraint, but need not share any other properties with the benchmark model. While post-processing techniques require access to individual predictions from the benchmark model, our approach only requires that we know its average loss.

\subsection{Selective Labels and Missing Data}
In settings such as criminal justice and credit lending, the training data only contain labeled outcomes for a selectively observed sample from the full population of interest. This is a missing data problem \cite{little2019statistical}. Because the outcome label is missing based on a selection mechanism, this type of missing data is known as the \emph{selective labels problem} \cite{lakkaraju2017selective, kleinberg2018human}. One solution treats the selectively labelled population as if it were the population of interest, and proceeds with training and evaluation on the selectively labelled population only. This is also called the ``\emph{known good-bad}'' (KGB) approach \cite{ zeng2014rule, nguyen2016reject}. However, evaluating a model on a population different than the one on which it will be used can be highly misleading, particularly with regards to predictive fairness measures \cite{kallus2018residual, CostonEtAl(20)}. Unfortunately, most fair classification and fair regression methods do not offer modifications to address the selective labels problem, whereas our framework does.

Popular in credit lending applications, ``reject inference'' procedures incorporate information from the selectively unobserved cases (i.e., rejected applicants) in model construction and evaluation by imputing missing outcomes using augmentation, reweighing or extrapolation-based approaches \cite{LiEtAl(20), mancisidor2020deep}. 
These approaches are similar to domain adaptation techniques, and indeed 
the selective labels problem can be cast as domain adaptation since the labelled training data is not sampled from the target distribution. Most relevant to our setting are covariate shift methods for domain adaptation. Reweighing procedures have been proposed for jointly addressing covariate shift and fairness \cite{coston2019fair, SinghViolations}. While \fairs similarly uses iterative reweighing to solve our joint optimization problem, we explicitly use extrapolation to address covariate shift.
Empirically we find extrapolation can achieve lower disparities than reweighing.

\section{Setting and Problem Formulation}\label{section: set-up}
The population of interest is described by the random vector $(X_i, A_i, D_i, Y_i^*) \sim P$, where $X_i \in \cX$ is a feature vector, $A_i \in \{0,1\}$ is a protected or sensitive attribute, $D_i \in \cD$ is the decision and $Y_i^* \in \cY \subseteq [0, 1]$ is a discrete or continuous outcome. The training data consist of $n$ i.i.d. draws from the joint distribution $P$ and may suffer from a \textit{selective labels problem}: There exists $\cD^* \subseteq \cD$ such that the outcome is observed if and only if the decision satisfies $D_i \in \cD^*$. Hence, the training data are $\{(X_i, A_i, D_i, Y_i)\}_{i=1}^{n}$, where $Y_i = Y_i^* 1\{D_i \in \cD^*\})$ is the \textit{observed outcome} and $1\{\cdot\}$ denotes the indicator function.

Given a specified set of prediction functions $\cF$ with elements $f \colon \cX \rightarrow [0, 1]$, we search for the prediction function $f \in \cF$ that minimizes or maximizes a measure of predictive disparities with respect to the sensitive attribute subject to a constraint on predictive performance. We measure performance using average loss, where $l \colon \cY \times [0, 1]\rightarrow [0, 1]$ is the loss function and $\loss(f) := \e\left[ l(Y_i^*, f(X_i)) \right]$. The loss function is assumed to be $1$-Lipshitz under the $l_1$-norm following \cite{AgarwalEtAl(19)-FairRegression}. The constraint on performance takes the form $\loss(f) \leq \epsilon$ for some specified \textit{loss tolerance} $\epsilon \geq 0$. The set of prediction functions satisfying this constraint is the \textit{set of good models}.

The loss tolerance may be chosen based on an existing benchmark model $\tilde f$ such as an existing risk score, e.g., by setting $\epsilon = (1 + \delta) \loss(\tilde f)$ for some $\delta \in [0, 1]$. The set of good models now describes the set of models whose performance lies within a $\delta$-neighborhood of the benchmark model. When defined in this manner, the set of good models is also called the ``Rashomon set'' \cite{Rudin(19), FisherEtAl(19), DongRudin(20), SemenovaEtAl(20)}.

\subsection{Measures of Predictive Disparities}
We consider measures of predictive disparity of the form
    \begin{equation}\label{equation: definition of disparity}
        \disp(f) := \beta_0 \e\left[ f(X_i) | \cE_{i,0} \right] + \beta_1 \e\left[ f(X_i) | \cE_{i,1} \right],
    \end{equation}
where $\cE_{i,a}$ is a group-specific conditioning event that depends on $(A_i, Y_i^*)$ and $\beta_a \in \reals$ for $a \in \{0, 1\}$ are chosen parameters. Note that we measure predictive disparities over the \textit{full} population (i.e., not conditional on $D_i$).

For different choices of the conditioning events $\cE_{i,0}, \cE_{i,1}$ and parameters $\beta_0, \beta_1$, our predictive disparity measure summarizes violations of common definitions of predictive fairness.

\begin{definition} \label{definition: Statistical parity}
    \textbf{Statistical parity} (SP) requires the prediction $f(X_i)$ to be independent of the attribute $A_i$ \cite{Dwork2012, ZemelEtAl(13), FeldmanEtAl(15)}. By setting $\cE_{i,a} = \{A_i = a\}$ for $a \in \{0, 1\}$ and ${\beta_0 = -1},$ ${\beta_1 = 1}$, $\disp(f)$ measures the difference in average predictions across values of the sensitive attribute.
\end{definition} 

\begin{definition}
\label{definition: balance for positive and negative class}
Suppose $\cY = \{0, 1\}$. \textbf{Balance for the positive class} (BFPC) and \textbf{balance for the negative class} (BFNC) requires the prediction $f(X_i)$ to be independent of the attribute $A_i$ conditional on $Y_i^* = 1$ and $Y_i^* = 0$ respectively (e.g., Chapter 2 of \cite{barocas-hardt-narayanan}). Defining $\cE_{i,a} = \{Y_i^* = 1, A_i = a\}$ for $a \in \{0, 1\}$ and $\beta_0 = -1, \beta_1 = 1$, $\disp(f)$ describes the difference in average predictions across values of the sensitive attribute given $Y_i^* = 1$. If instead $\cE_{i,a} = \{Y_i^* = 0, A_i = a\}$ for $a \in \{0, 1\}$, then $\disp(f)$ equals the difference in average predictions across values of the sensitive attribute given $Y_i^* = 0$. 
\end{definition}

Our focus on differences in average predictions across groups is a common relaxation of parity-based predictive fairness definitions \cite{CorbettDaviesEtAl2017, MitchellEtAl(19)}.

Our predictive disparity measure can also be used for \textit{fairness promoting interventions}, which aim to increase opportunities for a particular group. 
For instance, the decision maker may wish to search for the prediction function among the set of good models that minimizes the average predicted risk score $f(X_i)$ for a historically disadvantaged group.

\begin{definition}\label{definition: affirmative action and qualified affirmative action}
    Defining $\cE_{i,1} = \{A_i = 1\}$ and $\beta_0 = 0, \beta_1 = 1$, $\disp(f)$ measures the average risk score for the group with $A_i = 1$. This is an \textbf{affirmative action}-based fairness promoting intervention. Further assuming $\cY = \{0, 1\}$ and defining $\cE_{i,1} = \{Y_i^* = 1, A_i = 1\}$, $\disp(f)$ measures the average risk score for the group with both $Y_i^* = 1$, $A_i = 1$. This is a \textbf{qualified affirmative action}-based fairness promoting intervention.
\end{definition}

Our approach can accommodate other notions of predictive disparities. For instance, in the Supplement, we show how to achieve \textit{bounded group loss}, which requires that the average loss conditional on each value of the sensitive attribute reach some threshold \cite{AgarwalEtAl(19)-FairRegression}. 

\subsection{Characterizing Predictive Disparities over the Set of Good Models}
We develop the algorithmic framework, Fairness in the Rashomon Set (\fairs), to solve two related problems over the set of good models. First, we characterize the range of predictive disparities by minimizing or maximizing the predictive disparity measure over the set of good models. We focus on the minimization problem 
    \begin{equation}\label{equation: disparity range problem}
        \min_{f \in \cF} \, \disp(f) \mbox{ s.t. } \loss(f) \leq \epsilon.
    \end{equation}
Second, we search for the prediction function that minimizes the absolute predictive disparity over the set of good models
    \begin{equation}\label{equation: absolute disparity min problem}
        \min_{f \in \cF} \, \left| \disp(f) \right| \mbox{ s.t. } \loss(f) \leq \epsilon. 
    \end{equation}
For auditors, (\ref{equation: disparity range problem}) traces out the range of predictive disparities that \textit{could} be generated in a given setting, thereby identifying where the benchmark model lies on this frontier. This is crucially related to the legal notion of ``business necessity'' in assessing disparate impact -- the regulator may audit whether there exist alternative prediction functions that achieve similar performance yet generate different predictive disparities \cite{civil-rights-act, ECOA, BarocasSelbst2016}. For decision makers, (\ref{equation: absolute disparity min problem}) searches for prediction functions that reduce absolute predictive disparities without compromising predictive performance. 

\section{A Reductions Approach to Optimizing over the Set of Good Models}\label{section: reductions approach}
We characterize the range of predictive disparities (\ref{equation: disparity range problem}) and find the absolute predictive disparity minimizing model (\ref{equation: absolute disparity min problem}) over the set of good models using techniques inspired by the reductions approach in \cite{agarwal2018reductions, AgarwalEtAl(19)-FairRegression}. Although originally developed to solve fair classification and fair regression problems in the case without selective labels, we show that the reductions approach can be appropriately modified to solve general optimization problems over the set of good models in the presence of selective labels.
For exposition, we first focus on the case without selective labels, where $\cD^* = \cD$ and the outcome $Y_i^*$ is observed for all observations. We solve (\ref{equation: disparity range problem}) in the main text and (\ref{equation: absolute disparity min problem}) in \textsection~\ref{section: computing abs disp min model} of the Supplement. We cover selective labels in \textsection~\ref{section: selective labels}.

\subsection{Computing the Range of Predictive Disparities}\label{section: range of predictive disparities}

We consider randomized prediction functions that select $f \in\cF$ according to some distribution $Q \in \Delta(\cF)$ where $\Delta$ denotes the probability simplex. Let $\loss(Q) := \sum_{f \in \cF} Q(f) \loss(f)$ and $\disp(Q) := \sum_{f \in \cF} Q(f) \disp(f)$. We solve
    \begin{equation}\label{equation: disp range, randomizaton predictor problem}
        \min_{Q \in \Delta(\cF)} \disp(Q) \mbox{ s.t. } \loss(Q) \leq \epsilon.
    \end{equation}
    
While it may be possible to solve this problem directly for certain parametric function classes, we develop an approach that can be applied to any generic function class.\footnote{Our error analysis only covers function classes whose Rademacher complexity can be bounded as in Assumption~\ref{assumption: error analysis for exp grad extremes}.}
A key object for doing so will be classifiers obtained by thresholding prediction functions. For cutoff $z \in [0, 1]$, define $h_f(x, z) = 1\{f(x) \geq z\}$ and let $\cH := \{h_f : f \in \cF\}$ be the set of all classifiers obtained by thresholding prediction functions $f \in \cF$. We first reduce the optimization problem (\ref{equation: disp range, randomizaton predictor problem}) to a constrained classification problem through a discretization argument, and then solve the resulting constrained classification problem through a further reduction to finding the saddle point of a min-max problem.

Following the notation in \cite{AgarwalEtAl(19)-FairRegression}, we define a discretization grid for $[0, 1]$ of size $N$ with $\alpha := 1/N$ and $\cZ_{\alpha} := \{j \alpha \colon j = 1, \hdots, N\}$. Let $\tilde \cY_{\alpha}$ be an $\frac{\alpha}{2}$-cover of $\cY$. The piecewise approximation to the loss function is $l_{\alpha}(y, u) := l(\underline{y}, [u]_\alpha + \frac{\alpha}{2})$, where $\underline{y}$ is the smallest $\tilde y \in \tilde \cY_\alpha$ such that $|y - \tilde y|\leq \frac{\alpha}{2}$ and $[u]_\alpha$ rounds $u$ down to the nearest integer multiple of $\alpha$. For a fine enough discretization grid, $\loss_\alpha(f) := \e\left[l_\alpha(Y_i^*, f(X_i)) \right]$ approximates $\loss(f)$.

Define $c(y, z) := N \times \left( l(y, z + \frac{\alpha}{2}) - l(y, z - \frac{\alpha}{2}) \right)$ and $Z_\alpha$ to be the random variable that uniformly samples $z_\alpha \in \cZ_\alpha$ and is independent of the data $(X_i, A_i, Y_i^*)$. For $h_f \in \cH$, define the cost-sensitive average loss function as $\cost(h_f) := \e\left[ c(\underline{Y}^*_i, Z_\alpha) h_f(X_i, Z_\alpha) \right]$. Lemma 1 in \cite{AgarwalEtAl(19)-FairRegression} shows $\cost(h_f) + c_0 = \loss_\alpha(f)$ for any $f \in \cF$, where $c_0 \geq 0$ is a constant that does not depend on $f$. Since $\loss_\alpha(f)$ approximates $\loss(f)$, $\cost(h_f)$ also approximates $\loss(f)$. For $Q \in \Delta(\cF)$, define $Q_h \in \Delta(\cH)$ to be the induced distribution over threshold classifiers $h_f$. By the same argument, $\cost(Q_h) + c_0 = \loss_\alpha(Q)$, where $\cost(Q_h) := \sum_{h_f \in \cH} Q_h(h) \cost(h_f)$ and $\loss_\alpha(Q)$ is defined analogously.

We next relate the predictive disparity measure defined on prediction functions to a predictive disparity measure defined on threshold classifiers. Define $\disp(h_f) := \beta_0 \e\left[ h_f(X_i, Z_\alpha) \mid \cE_{i,0} \right] + \beta_1 \e\left[ h_f(X_i, Z_\alpha) \mid \cE_{i,1} \right].$
\begin{lemma}\label{lemma: disparity classifier relates to disparity prediction function}
    Given any distribution over $(X_i, A_i, Y_i^*)$ and $f \in \cF$, $\left| \disp(h_f) - \disp(f) \right| \leq \left(|\beta_0| + |\beta_1|\right) \alpha$.
\end{lemma}
\noindent Lemma \ref{lemma: disparity classifier relates to disparity prediction function} combined with Jensen's Inequality imply ${|\disp(Q_h) - \disp(Q)| \leq \left(|\beta_0| + |\beta_1|\right) \alpha}$.

Based on these results, we approximate (\ref{equation: disp range, randomizaton predictor problem}) with its analogue over threshold classifiers
    \begin{equation}\label{equation: discretized problem for extremes}
        \min_{Q_h \in \Delta(\cH)} \disp(Q_h) \mbox{ s.t. } \cost(Q_h) \leq \epsilon - c_0.
    \end{equation}
We solve the sample analogue in which we minimize $\widehat{\disp}(Q_h)$ subject to $\widehat{\cost}(Q_h) \leq \hat{\epsilon}$, where $\hat{\epsilon} := \epsilon - \hat{c}_0$ plus additional slack, and $\hat{c}_0, \widehat{\disp}(Q_h)$, $\widehat{\cost}(Q_h)$ are the associated sample analogues. We form the Lagrangian $L(Q_h, \lambda) := \widehat{\disp}(Q_h) + \lambda (\widehat{\cost}(Q_h) - \hat{\epsilon})$ with primal variable $Q_h \in \Delta(\cH)$ and dual variable $\lambda \in \reals^{+}$. Solving the sample analogue is equivalent to finding the saddle point of the min-max problem $ \min_{Q_h \in \Delta(\cH)} \max_{0 \leq \lambda \leq B_\lambda} L(Q_h, \lambda)$, where $B_\lambda \geq 0$ bounds the Lagrange multiplier. We search for the saddle point by adapting the exponentiated gradient algorithm used in \cite{agarwal2018reductions, AgarwalEtAl(19)-FairRegression}. The algorithm delivers a $\nu$-approximate saddle point of the Lagrangian, denoted $(\hat{Q}_h, \hat{\lambda})$. Since it is standard, we provide the details of and the pseudocode for the exponentiated gradient algorithm in \textsection~\ref{sec: exp grad for range} of the Supplement.

\subsection{Error Analysis}

The suboptimality of the returned solution $\hat{Q}_h$ can be controlled under conditions on the complexity of the model class $\cF$ and how various parameters are set. 

\begin{assumption}\label{assumption: error analysis for exp grad extremes}
    Let $R_n(\cH)$ be the Radermacher complexity of $\cH$. There exists constants $C, C^\prime, C^{\prime \prime} > 0$ and $\phi \leq 1/2$ such that $R_n(\cH) \leq C n^{-\phi}$ and $\hat{\epsilon} = \epsilon - \hat{c}_0 + C^\prime n^{-\phi} - C^{\prime \prime} n^{-1/2}$.
\end{assumption}

\begin{theorem}\label{theorem: error analysis for exp grad extremes}
Suppose Assumption \ref{assumption: error analysis for exp grad extremes} holds for $C' \geq 2C + 2 + \sqrt{2 \ln(8 N/\delta)}$ and $C^{\prime \prime} \geq \sqrt{\frac{-\log(\delta / 8)}{2}}$. Let $n_0, n_1$ denote the number of samples satisfying the events $\cE_{i,0}, \cE_{i,1}$ respectively.

Then, the exponentiated gradient algorithm with $\nu \propto n^{-\phi}$, $B_\lambda \propto n^\phi$ and $N \propto n^\phi$ terminates in $O(n^{4\phi})$ iterations and returns $\hat{Q}_h$, which when viewed as a distribution over $\cF$, satisfies with probability at least $1-\delta$ one of the following: 1) $\Qhat_h \neq null$, $\loss(\hat{Q}_h) \leq \epsilon + \tilde{O}(n^{-\phi})$ and $\disp(\hat{Q}_h) \leq \disp(\tilde{Q}) + \tilde{O}(n_0^{-\phi}) + \tilde{O}(n_1^{-\phi})$ for any $\tilde{Q}$ that is feasible in (\ref{equation: disp range, randomizaton predictor problem}); or 2) $\Qhat_h = null$ and (\ref{equation: disp range, randomizaton predictor problem}) is infeasible.\footnote{The notation $\tilde{O}(\cdot)$ suppresses polynomial dependence on $\ln(n)$ and $\ln(1/\delta)$}
\end{theorem}

Theorem \ref{theorem: error analysis for exp grad extremes} shows that the returned solution $\hat{Q}_h$ is approximately feasible and achieves the lowest possible predictive disparity up to some error. Infeasibility is a concern if no prediction function $f \in \cF$ satisfies the average loss constraint. Assumption~\ref{assumption: error analysis for exp grad extremes} is satisfied for instance under LASSO and ridge regression. If Assumption 1 does not hold, \fairs delivers good solutions to the sample analogue of Eq.~\ref{equation: discretized problem for extremes} (see Supplement \textsection~\ref{section: solution quality for algorithm exp grad}).

A practical challenge is that the solution returned by the exponentiated gradient algorithm $\hat{Q}_h$ is a stochastic prediction function with possibly large support. Therefore it may be difficult to describe, time-intensive to evaluate, and memory-intensive to store. Results from \cite{CotterEtAl(19)-ALT} show that the support of the returned stochastic prediction function may be shrunk while maintaining the same guarantees on its performance by solving a simple linear program. 
The linear programming reduction reduces the stochastic prediction function to have at most two support points and we use this linear programming reduction in our empirical work (see \textsection~\ref{section: linear program reduction} of the Supplement for details).

\section{Optimizing Over the Set of Good Models Under Selective Labels} \label{section: selective labels}

We now modify the reductions approach to the empirically relevant case in which the training data suffer from the selective labels problem, whereby the outcome $Y_i^*$ is observed only if $D_i \in \cD^*$ with $\cD^* \subset \cD$.
The main challenge concerns evaluating model properties over the target population when we only observe labels for a selective (i.e., biased) sample.
We propose a solution that uses outcome modeling, also known as extrapolation, to estimate these properties.

To motivate this approach, we observe that average loss and measures of predictive disparity (\ref{equation: definition of disparity}) that condition on $Y_i^*$ are not identified under selective labels without further assumptions. 
We introduce the following assumption on the nature of the selective labels problem for the binary decision setting with $\cD = \{0, 1\}$ and $\cD^* = \{1\}$. 
\begin{assumption}\label{assumption: selection on obs and positivity}
    The joint distribution $(X_i, A_i, D_i, Y_i^*) \sim P$ satisfies 1) \textbf{selection on observables}: $D_i \indep Y_i^* \mid X_i$, and 2) \textbf{positivity}: $\p\left(D_i = 1 \mid X_i = x\right) > 1$ with probability one.
\end{assumption}
This assumption is common in causal inference and selection bias settings (e.g., Chapter 12 of \cite{ImbensRubin(15)} and \cite{Heckman(90)})\footnote{Casting this into potential outcomes notation where $Y_i^{d}$ is the counterfactual outcome if decision $d$ were assigned, we define $Y^{0}_i = 0$ and $Y_i^{1} = Y_i^*$ (e.g., a rejected loan application cannot default). The observed outcome $Y_i$ then equals $Y^{1}_i D_i$.} and in covariate shift learning  \cite{moreno2012unifying}.
Under Assumption~\ref{assumption: selection on obs and positivity}, the regression function $\mu(x) := \e[Y_i^* \mid X_i = x]$ is identified as $\e[Y_i \mid X_i, D_i = 1]$, and may be estimated by regressing the observed outcome $Y_i$ on the features $X_i$ among observations with $D_i = 1$, yielding the outcome model $\hat{\mu}(x)$.

We can use the outcome model to estimate loss on the full population. One approach, \emph{Reject inference by extrapolation} (RIE), uses $\hat{\mu}(x)$ as pseudo-outcomes for the unknown observations \cite{crook2004does}.
We consider a second approach,
\emph{Interpolation \& extrapolation} (IE), which uses $\hat{\mu}(x)$ as pseudo-outcomes for \emph{all} applicants, replacing the $\{0,1\}$ labels for known cases with smoothed estimates of their underlying risks.
Letting $n^0$, $n^1$ be the number of observations in the training data with $D_i = 0$, $D_i = 1$ respectively, Algorithms \ref{alg:reject_inference}-\ref{alg:interpolation-extrapolation} summarize the RIE and IE methods. 
If the outcome model could perfectly recover $\mu(x)$, then the IE approach recovers an oracle setting for which the \fairs error analysis 
continues to hold (Theorem \ref{theorem:selective labels extremes} below).

\begin{algorithm}[htbp!]
\KwIn{
    $\{(X_i, Y_i, D_i =1, A_i)\}_{i=1}^{n^1}$,  $\{(X_i, D_i = 0, A_i)\}_{i=1}^{n^0}$ 
}
Estimate $\hat \mu(x)$ by regressing $Y_i \sim X_i \mid D_i = 1$.\\
$\hat{Y}(X_i) \leftarrow (1-D_i) \hat \mu(X_i) + D_i Y_i$ \\
\KwOut{
    $\{(X_i, \hat{Y}_i(X_i), D_i, A_i)\}_{i=1}^{n^1}$,  $\{(X_i, \hat{Y}_i(X_i), D_i, A_i)\}_{i=1}^{n^0}$ 
}
\caption{Reject inference by extrapolation (RIE) for the selective labels setting}
\label{alg:reject_inference}
\end{algorithm}

\begin{algorithm}[htbp!]
\KwIn{
    $\{(X_i, Y_i, D_i =1, A_i)\}_{i=1}^{n^1}$,  $\{(X_i, D_i = 0, A_i)\}_{i=1}^{n^0}$ \\
}
Estimate $\hat \mu(x)$ by regressing $Y_i \sim X_i \mid D_i = 1$.\\
$\hat{Y}(X_i) \leftarrow \hat{\mu}(X_i)$ \\
\KwOut{
    $\{(X_i, \hat{Y}_i(X_i), D_i, A_i)\}_{i=1}^{n^1}$,  $\{(X_i, \hat{Y}_i(X_i), D_i, A_i)\}_{i=1}^{n^0}$
}
\caption{Interpolation and extrapolation (IE) method for the selective labels setting}
\label{alg:interpolation-extrapolation}
\end{algorithm}

Estimating predictive disparity measures  on the full population requires a more general definition of predictive disparity than previously given in Eq.~\ref{equation: definition of disparity}. Define the modified predictive disparity measure over threshold classifiers as
\begin{align}\label{equation: definition of disparity for selective labels}
    \begin{split}
    \disp(h_f) = &\beta_0 \frac{\mathbb{E}\left[ g(X_i,Y_i) h_f(X_i, Z_\alpha) \mid \cE_{i,0} \right]}{\e[g(X_i,Y_i) \mid \cE_{i,0} ]} + \\
&\beta_1 \frac{\mathbb{E}\left[ g(X_i,Y_i)h_f(X_i, Z_\alpha) | \cE_{i,1} \right]}{\e[g(X_i,Y_i) \mid \cE_{i,1} ]},
    \end{split}
\end{align} 
where the nuisance function $g(X_i,Y_i)$ is constructed to identify the measure of interest.\footnote{Note that we state this general form of $g$ to allow $g$ to use $Y_i$ for e.g. doubly-robust style estimates.}
To illustrate, the qualified affirmative action fairness-promoting intervention (Def. \ref{definition: affirmative action and qualified affirmative action}) is identified as $ \mathbb{E}[f(X_i) | Y_i^* = 1, A_i = 1] = \frac{\mathbb{E}[f(X_i) \mu(X_i) | A_i = 1]}{\mathbb{E}[ \mu(X_i) | A_i = 1]}$ under Assumption~\ref{assumption: selection on obs and positivity} (See proof of Lemma~\ref{lemma:disparity concentration selective labels} in the Supplement). This may be estimated by plugging in the outcome model estimate $\hat \mu(x)$. 
Therefore, Eq.~\ref{equation: definition of disparity for selective labels} specifies the qualified affirmative action fairness-promoting intervention by setting $\beta_0 = 0$, $\beta_1 = 1$,  $\cE_{i,1} = 1\left\{A_i =1\right\}$, and $g(X_i, Y_i) = \hat \mu(X_i)$. This more general definition (Eq.~\ref{equation: definition of disparity for selective labels}) is only required for predictive disparity measures that condition on events $\cE$ depending on both $Y^*$ and $A$; It is straightforward 
to compute disparities based on events $\cE$ that only depend on $A$ over the full population. To compute disparities based on events $\cE$ that also depend on $Y^*$, we find the saddle point of the following Lagrangian: ${L(h_f, \lambda) = \hat{\e}\left[ \e_{Z_\alpha}\left[ c_\lambda(\underline{\hat{\mu}}_i, A_i, Z_\alpha) h_f(X_i, Z_\alpha) \right] \right] - \lambda \hat{\epsilon}}$, where we now use case weights ${c_\lambda(\underline{\hat{\mu}}_i, A_i, Z_\alpha) :=}$ ${\frac{\beta_0}{\hat{p}}g(X_i,Y_i) 1\left\{ \cE_{i,0} \right\} + \frac{\beta_1}{\hat{p}}g(X_i,Y_i) 1\left\{ \cE_{i,1} \right\} + \lambda c(\underline{\hat{\mu}}_i, Z_\alpha)}$ and $\hat{p} =\hat{\e}[g(X_i, Y_i)]$. Finally, as before, we find the saddle point using the exponentiated gradient algorithm.

\subsection{Error Analysis under Selective Labels} \label{section:error analysis under selective labels}
Define $\loss_\mu(f) := \e[l(\mu(X_i), f(X_i))]$ for $f \in \cF$ with $\loss_\mu(Q)$ defined analogously for $Q \in \Delta(\cF)$. The error analysis of the exponentiated gradient algorithm continues to hold in the presence of selective labels under oracle access to the true outcome regression function $\mu$.

\begin{theorem}[Selective Labels]\label{theorem:selective labels extremes}
Suppose Assumption \ref{assumption: selection on obs and positivity} holds and the exponentiated gradient algorithm is given as input the modified training data $\{(X_i, A_i, \mu(X_i)\}_{i=1}^{n}$.

Under the same conditions as Theorem \ref{theorem: error analysis for exp grad extremes}, the exponentiated gradient algorithm terminates in $O(n^{4\phi})$ iterations and returns $\hat{Q}_h$, which when viewed as a distribution over $\cF$, satisfies with probability at least $1-\delta$ either one of the following: 1) $\Qhat_h \neq null$, $\loss_\mu(\hat{Q}_h) \leq \epsilon + \tilde{O}(n^{-\phi})$ and $\disp(\hat{Q}_h) \leq \disp(\tilde{Q}) + \tilde{O}(n_0^{-\phi}) + \tilde{O}(n_1^{-\phi})$ for any $\tilde{Q}$ that is feasible in (\ref{equation: disp range, randomizaton predictor problem}); or 2) $\Qhat_h = null$ and (\ref{equation: disp range, randomizaton predictor problem}) is infeasible.
\end{theorem}
In practice, estimation error in $\hat{\mu}$ will affect the bounds in Theorem~\ref{theorem:selective labels extremes}. The empirical analysis in the next section finds that our method nonetheless performs well when using $\hat{\mu}$.

\section{Application: Consumer Lending }\label{section: consumer lending application}
Suppose a financial institution wishes to replace an existing credit scoring model with one that has better fairness properties and comparable performance, if such a model exists. 
The following empirical analysis demonstrates how to use \fairs for this task. We use \fairs to find the absolute predictive disparity-minimizing model over the set of good models on a real world consumer lending dataset with selectively labeled outcomes.

We use data from Commonwealth Bank of Australia, a large financial institution in Australia (henceforth, "CommBank"), on a sample of 7,414 personal loan applications submitted from July 2017 to July 2019 by customers that did not have a prior financial relationship with CommBank. A personal loan is a credit product that is paid back with monthly installments and used for a variety of purposes such as purchasing a used car or refinancing existing debt. In our sample, the median personal loan size is AU\$10,000 and the median interest rate is 13.9\% per annum. For each loan application, we observe application-level information such as the applicant's credit score and reported income, whether the application was approved by CommBank, the offered terms of the loan, and whether the applicant defaulted on the loan. There is a selective labels problem as we only observe whether an applicant defaulted on the loan within 5 months ($Y_i$) if the application was funded.\footnote{An application is funded when it is both approved by CommBank and the offered terms were accepted by the applicant}  In our sample, 44.9\% of applications were funded and 2.0\% of funded loans defaulted within 5 months.

Motivated by a decision maker that wishes to reduce credit access disparities across geographic regions, we focus on the task of predicting the likelihood of default $Y^*_i = 1$ based on information in the loan application $X_i$ while limiting predictive disparities across SA4 geographic regions within Australia. SA4 regions are statistical geographic areas defined by the Australian Bureau of Statistics (ABS) and are analogous to counties in the United States. An SA4 region is classified as socioeconomically disadvantaged ($A_i = 1$) if it falls in the top quartile of SA4 regions based on the ABS' Index of Relative Socioeconomic Disadvantage (IRSD), which is an index that aggregates census data related to socioeconomic disadvantage.\footnote{Complete details on the IRSD may be found in \cite{SEIFA(16)} and additional details on the definition of socioeconomic disadvantage are given in \textsection~\ref{section: additional details on the IRSD} of the Supplement.} Applicants from disadvantaged SA4 regions are under-represented among funded applications, comprising 21.7\% of all loan applications, but only 19.7\% of all funded loan applications.


Our experiment investigates the performance of \fairs under our two proposed extrapolation-based solutions to selective labels, RIE and IE (See Algorithms~\ref{alg:reject_inference}), as well as the Known-Good Bad (KGB) approach that uses only the selectively labelled population. 
Because we do not observe default outcomes for all applications, we conduct a semi-synthetic simulation experiment by generating synthetic funding decisions and default outcomes. 
On a $20\%$ sample of applicants, we learn $\pi(x) := \hat{P}(D_i = 1 | X_i = x)$ and $\mu(x) := \hat{P}(Y_i = 1 | X_i = x, D_i =1)$ using random forests. 
We generate synthetic funding decisions $\widetilde D_i$ according to $\widetilde D_i \mid X_i \sim Bernoulli(\pi(X_i))$ and synthetic default outcomes $\widetilde Y^*_i$ according to $\widetilde Y^*_i \mid X_i \sim Bernoulli(\mu(X_i))$. We train all models as if we only knew the synthetic outcome for the synthetically funded applications. We estimate $\hat \mu(x) := \hat{P}(\tilde Y_i = 1 | X_i = x, \tilde D_i =1)$ using random forests and use $\hat \mu(X_i)$ to generate the pseudo-outcomes $\hat Y(X_i)$ for RIE and IE as described in Algorithms~\ref{alg:reject_inference} and ~\ref{alg:interpolation-extrapolation}.
As benchmark models, we use the loss-minimizing linear models learned using KGB, RIE, and IE approaches, whose respective training losses are used to select the corresponding loss tolerances $\epsilon$. 

We compare against the fair reductions approach to classification \href{https://fairlearn.github.io/v0.5.0/api_reference/fairlearn.reductions.html}{(fairlearn)} and the Target-Fair Covariate Shift (TFCS) method. 
TFCS iteratively reweighs the training data via gradient descent on a loss term comprised of the covariate shift-reweighed classification loss and a fairness loss \cite{coston2019fair}.
Fairlearn searches for the loss-minimizing model subject to a fairness parity constraint \cite{agarwal2018reductions}. The fairlearn model is effectively a KGB model since the fairlearn package does not offer modifications for selective labels.\footnote{To accommodate reject inference, a method must support real-valued outcomes. The fairlearn package does not, but the related \emph{fair regressions} method does \cite{AgarwalEtAl(19)-FairRegression}. This is sufficient for statistical parity (Def.~\ref{definition: Statistical parity}), but other parities such as BFPC and BFNC (Def.~\ref{definition: balance for positive and negative class}) require further modifications as discussed in \textsection~\ref{section: selective labels}} We use logistic regression as the base model for both fairlearn and TFCS.
Results are reported on all applicants in a held out test set, and performance metrics are constructed with respect to the synthetic outcome $\tilde Y_i^*$. 

Figure~\ref{fig:selective_disparities} shows the AUC (y-axis) against disparity (x-axis) for the KGB, RIE, IE benchmarks and their \fairs variants as well as the
TFCS models and fairlearn models.
Colors denote the adjustment strategy for selective labels, and the shape specifies the optimization method.
The first row evaluates the models on all applicants in the test set (i.e., the target population).
On the target population, \fairs with reject extrapolation (RIE and IE) reduces disparities while achieving performance comparable to the benchmarks and to the reweighing approach (TFCS). 
It also achieves lower disparities than TFCS, likely because TFCS optimizes a non-convex objective function and may therefore converge to a local minimum. Reject extrapolation achieves better AUC than all KGB models, and only one KGB model (fairlearn) achieves a lower disparity.
The second row evaluates the models on only the \emph{funded} applicants. Evaluation on the funded cases underestimates disparities across the methods and overestimates AUC for the TFCS and KGB models. This underscores the importance of accounting for the selective labels problem in both model construction and evaluation.

\begin{figure}[htbp!]
    \centering
    \includegraphics[scale=0.45]{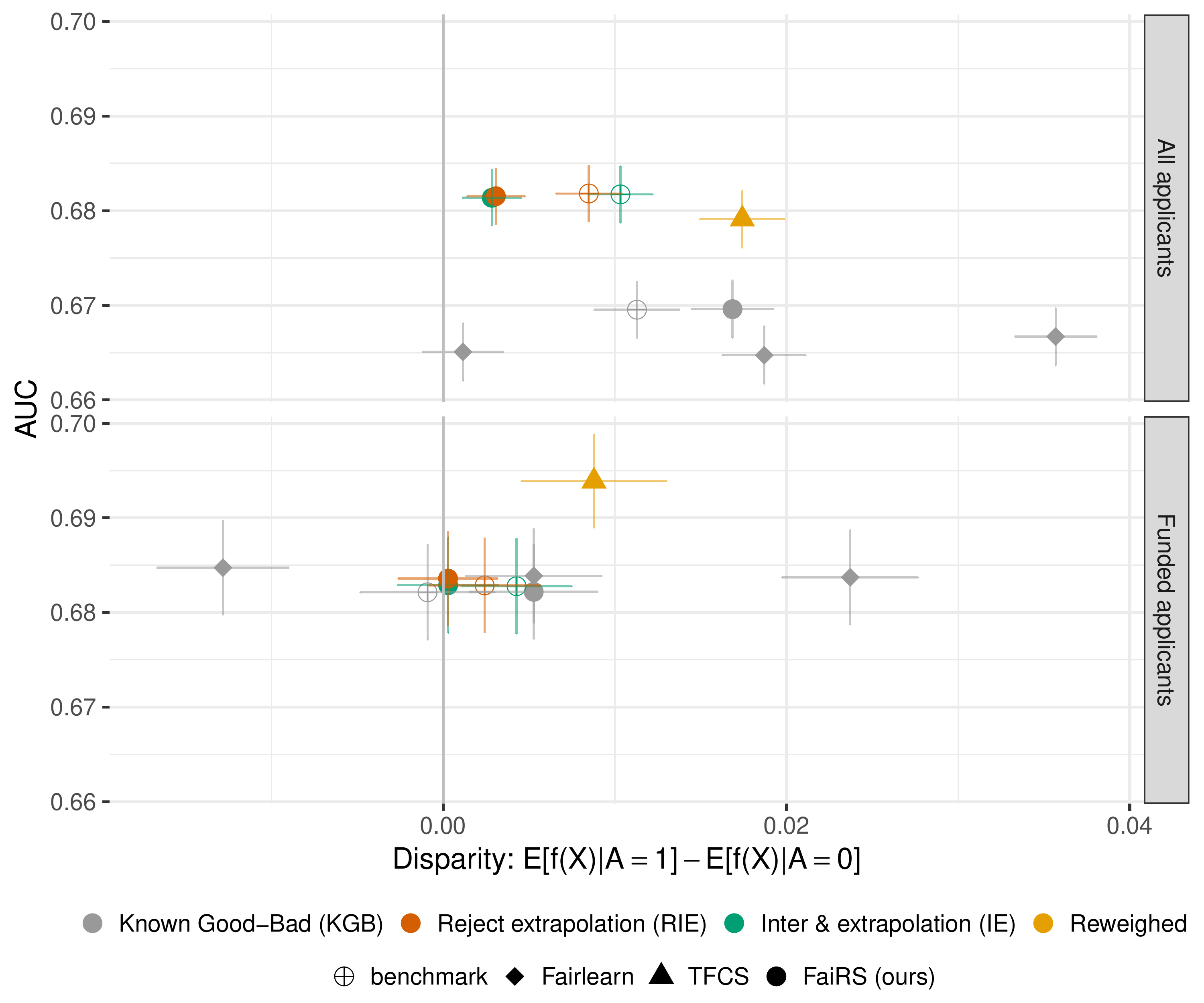}
    \caption{Area under the ROC curve (AUC) with respect to the synthetic outcome against disparity in the average risk prediction for the disadvantaged ($A_i=1$) vs advantaged ($A_i=0$) groups.
    \fairs reduces disparities for the RIE and IE approaches while maintaining AUCs comparable to the benchmark models (first row). 
    Evaluation on only funded applicants (second row) overestimates the performance of TFCS and KGB models and underestimates disparities for all models.
    Error bars show the $95\%$ confidence intervals. See \textsection~\ref{section: consumer lending application} for details.}
    \label{fig:selective_disparities}
\end{figure}

\section{Application: Recidivism Risk Prediction}\label{section: recidivism application}
We next explore the range of disparities over the set of good models in a recidivism risk prediction task applied to ProPublica's COMPAS recidivism data \cite{AngwinEtAl(16)}. Our goal is to illustrate how an auditor may use \fairs to characterize the range of predictive disparities over the set of good models, and examine whether the COMPAS risk assessment generates larger disparities than other competing good models. Such an analysis is a crucial step to assessing legal claims of disparate impact.

COMPAS is a proprietary risk assessment developed by Northpointe (now Equivant) using up to 137 features \cite{Rudin2020Age}. As this data is not publicly available, our audit makes use of ProPublica's COMPAS dataset which contains demographic information and prior criminal history for criminal defendants in Broward County, Florida. Lacking access to the data used to train COMPAS, our set of good models may not include COMPAS itself. Nonetheless, prior work has shown that simple models using age and criminal history perform on par with COMPAS \cite{angelino2018learning}. 
These features will therefore suffice to perform our audit.
A notable limitation of the ProPublica COMPAS dataset is that it does not contain information for defendants who remained incarcerated. 
Lacking both features and outcomes for this group, we proceed without addressing this source of selection bias. We also make no distinction between criminal defendants who had varying lengths of incarceration before release, effectively assuming a null treatment effect of incarceration on recidivism. This assumption is based on findings that a counterfactual audit of COMPAS yields equivalent conclusions \cite{mishler2019modeling}. 

We analyze the range of predictive disparities with respect to race for three common notions of fairness (Definitions \ref{definition: Statistical parity}-\ref{definition: balance for positive and negative class}) among logistic regression models on a quadratic polynomial of the defendant's age and number of prior offenses whose training loss is near-comparable to COMPAS (loss tolerance $\epsilon = 1\%$ of COMPAS training loss).\footnote{We use a quadratic form following the analysis in \cite{Rudin2020Age}.} We split the data 50\%-50\% into a train and test set. Table \ref{table: Compas Race - ref model compas Test} summarizes the range of predictive disparities on the test set. The disparity minimizing and disparity maximizing models over the set of good of models achieve a test loss that is comparable to COMPAS (see \textsection~\ref{section: recidivism additional results} of the Supplement).

\begin{table}[t]
\caption{COMPAS fails an audit of the ``business necessity" defense for disparate impact by race. The set of good models (performing within $1\%$ of COMPAS's training loss) includes models that achieve significantly lower disparities than COMPAS.  The first panel (SP) displays the disparity in average predictions for black versus white defendants (Def. \ref{definition: Statistical parity}). The second panel (BFPC) analyzes the disparity in average predictions for black versus white defendants in the positive class, and the third panel examines the disparity in average predictions for black versus white defendants in the negative class (Def. \ref{definition: balance for positive and negative class}). Standard errors are reported in parentheses. See \textsection~\ref{section: recidivism application} for details.}
\label{table: Compas Race - ref model compas Test}
\vskip 0.05in
\begin{center}
\begin{small}
\begin{sc}
\begin{tabular}{ccccr}
\toprule
& Min. Disp. & Max. Disp.  & COMPAS \\
\midrule
SP & $-$0.060 & 0.120 & 0.194 \\
& (0.004) & (0.007) & (0.013) \\
\midrule 
BFPC & 0.049 & 0.125 & 0.156 \\
& (0.005) & (0.012) & (0.016) \\
\midrule
BFNC & 0.044 & 0.117 & 0.174 \\
& (0.005) & (0.009) & (0.016) \\
\bottomrule
\end{tabular}
\end{sc}
\end{small}
\end{center}
\vskip -0.1in
\end{table}

For each predictive disparity measure, the set of good models includes models that achieve significantly lower disparities than COMPAS. In this sense, COMPAS generates ``unjustified'' disparate impact across groups as there exists competing models in the set of good models that would reduce disparities without compromising performance. Notably, COMPAS' disparities are also larger than the maximum disparity over the set of good models. For example, the difference in COMPAS' average predictions for black relative to white defendants is strictly larger than that of any model in the set of good models (Table \ref{table: Compas Race - ref model compas Test}, SP).  
Interestingly, the minimal balance for the positive class and balance for the negative class disparities between black and white defendants over the set of good models are strictly positive (Table \ref{table: Compas Race - ref model compas Test}, BFPC and BFNC). For example any model whose performance lies in a neighborhood of COMPAS' loss has a higher false positive rate for black defendants than white defendants. This suggests while we can reduce predictive disparities between black and white defendants relative to COMPAS on all measures, we may be unable to eliminate balance for the positive class and balance for the negative class disparities without harming predictive performance. 

In Supplement ~\textsection \ref{section: regression experiments, communities & crime dataset}, we conducted regression experiments on the Communities \& Crime dataset, comparing \fairs against a loss minimizing least squares regression and finding that \fairs improves on statistical parity without compromising performance relative to the benchmark model.

\section{Conclusion}
We develop a framework, Fairness in the Rashomon Set (\fairs), to characterize the range of predictive disparities and find the absolute disparity minimizing model over the set of good models. \fairs is generic, applying to both a large class of prediction functions and a large class of predictive disparities. \fairs is suitable for a variety of applications including settings with selectively labelled outcomes where the selection decision and outcome are unconfounded given the observed features. \fairs can facilitate audits for disparate impact and efficiently search for a more equitable model with performance comparable to a benchmark model. In many settings the set of good models is a rich class, in which models differ substantially in terms of their fairness properties. By leveraging this phenomenon, our approach can reduce predictive disparities without compromising overall performance. 

\clearpage
\singlespacing
\bibliographystyle{unsrt}  
\bibliography{ref.bib} 

\begin{thebibliography}{10}

\bibitem{CaruanaEtAl(15)}
Rich Caruana, Yin Lou, Johannes Gehrke, Paul Koch, Marc Sturm, and Noemie
  Elhadad.
\newblock Intelligible models for healthcare: Predicting pneumonia risk and
  hospital 30-day readmission.
\newblock In {\em Proceedings of the 21th ACM SIGKDD International Conference
  on Knowledge Discovery and Data Mining}, page 1721–1730, 2015.

\bibitem{ChouldechovaEtAl(18)}
Alexandra Chouldechova, Diana Benavides-Prado, Oleksandr Fialko, and Rhema
  Vaithianathan.
\newblock A case study of algorithm-assisted decision making in child
  maltreatment hotline screening decisions.
\newblock volume~81 of {\em Proceedings of Machine Learning Research}, pages
  134--148, 2018.

\bibitem{kleinberg2018human}
Jon Kleinberg, Himabindu Lakkaraju, Jure Leskovec, Jens Ludwig, and Sendhil
  Mullainathan.
\newblock Human decisions and machine predictions.
\newblock {\em The Quarterly Journal of Economics}, 133(1):237--293, 2018.

\bibitem{FusterEtAl(20)}
Andreas Fuster, Paul Goldsmith-Pinkham, Tarun Ramadorai, and Ansgar Walther.
\newblock Predictably unequal? the effects of machine learning on credit
  markets.
\newblock Technical report, 2020.

\bibitem{RaghavanEtAl(20)}
Manish Raghavan, Solon Barocas, Jon Kleinberg, and Karen Levy.
\newblock Mitigating bias in algorithmic hiring: Evaluating claims and
  practices.
\newblock In {\em Proceedings of the 2020 Conference on Fairness,
  Accountability, and Transparency}, page 469–481, 2020.

\bibitem{BarocasSelbst2016}
Solon Barocas and Andrew Selbst.
\newblock Big data's disparate impact.
\newblock {\em California Law Review}, 104:671--732, 2016.

\bibitem{dastin_2018}
Jeffrey Dastin.
\newblock Amazon scraps secret ai recruiting tool that showed bias against
  women, Oct 2018.

\bibitem{vigdor_2019}
Neil Vigdor.
\newblock Apple card investigated after gender discrimination complaints, Nov
  2019.

\bibitem{ZemelEtAl(13)}
Rich Zemel, Yu~Wu, Kevin Swersky, Toni Pitassi, and Cynthia Dwork.
\newblock Learning fair representations.
\newblock pages 325--333, 2013.

\bibitem{agarwal2018reductions}
Alekh Agarwal, Alina Beygelzimer, Miroslav Dudik, John Langford, and Hanna
  Wallach.
\newblock A reductions approach to fair classification.
\newblock In {\em Proceedings of the 35th International Conference on Machine
  Learning}, pages 60--69, 2018.

\bibitem{AgarwalEtAl(19)-FairRegression}
Alekh Agarwal, Miroslav Dud{\'{\i}}k, and Zhiwei~Steven Wu.
\newblock Fair regression: Quantitative definitions and reduction-based
  algorithms.
\newblock In {\em Proceedings of the 36th International Conference on Machine
  Learning, {ICML} 2019, 9-15 June 2019, Long Beach, California, {USA}}, 2019.

\bibitem{Breiman(01)}
Leo Breiman.
\newblock Statistical modeling: The two cultures.
\newblock {\em Statistical Science}, 16(3):199--215, 2001.

\bibitem{ChouldechovaGsell(17)}
Alexandra Chouldechova and Max G'Sell.
\newblock Fairer and more accurate, but for whom?
\newblock Technical report, In Proceedings of the 2017 FAT/ML Workshop, 2017.

\bibitem{DongRudin(20)}
Jiayun Dong and Cynthia Rudin.
\newblock Variable importance clouds: A way to explore variable importance for
  the set of good models.
\newblock Technical report, arXiv preprint arXiv:1901.03209, 2020.

\bibitem{lakkaraju2017selective}
Himabindu Lakkaraju, Jon Kleinberg, Jure Leskovec, Jens Ludwig, and Sendhil
  Mullainathan.
\newblock The selective labels problem: Evaluating algorithmic predictions in
  the presence of unobservables.
\newblock In {\em Proceedings of the 23rd ACM SIGKDD International Conference
  on Knowledge Discovery and Data Mining}, pages 275--284, 2017.

\bibitem{civil-rights-act}
Civil rights act, 1964.
\newblock 42 U.S.C. § 2000e.

\bibitem{ECOA}
Equal credit opportunity act, 1974.
\newblock 15 U.S.C. § 1691.

\bibitem{gillis}
Talia Gillis.
\newblock False dreams of algorithmic fairness: The case of credit pricing.
\newblock 2019.

\bibitem{SelbstEtAl(19)}
Andrew~D. Selbst, Danah Boyd, Sorelle~A. Friedler, Suresh Venkatasubramanian,
  and Janet Vertesi.
\newblock Fairness and abstraction in sociotechnical systems.
\newblock In {\em Proceedings of the Conference on Fairness, Accountability,
  and Transparency}, FAT* '19, page 59–68, 2019.

\bibitem{HolsteinEtAl(19)}
Kenneth Holstein, Jennifer Wortman~Vaughan, Hal Daum\'{e}, Miro Dudik, and
  Hanna Wallach.
\newblock Improving fairness in machine learning systems: What do industry
  practitioners need?
\newblock In {\em Proceedings of the 2019 CHI Conference on Human Factors in
  Computing Systems}, CHI '19, page 1–16, 2019.

\bibitem{Rudin(19)}
Cynthia Rudin.
\newblock Stop explaining black box machine learning models for high stakes
  decisions and use interpretable models instead.
\newblock {\em Nature Machine Intelligence}, 1:206–215, 2019.

\bibitem{FisherEtAl(19)}
Aaron Fisher, Cynthia Rudin, and Francesca Dominici.
\newblock All models are wrong, but many are useful: Learning a variable's
  importance by studying an entire class of prediction models simultaneously.
\newblock Technical report, arXiv preprint arXiv:1801.01489, 2019.

\bibitem{MarxEtAl(19)}
Charles~T. Marx, Flavio du~Pin~Calmon, and Berk Ustun.
\newblock Predictive multiplicity in classification.
\newblock Technical report, arXiv preprint arXiv:1909.06677, 2019.

\bibitem{SemenovaEtAl(20)}
Lesia Semenova, Cynthia Rudin, and Ronald Parr.
\newblock A study in rashomon curves and volumes: A new perspective on
  generalization and model simplicity in machine learning.
\newblock Technical report, arXiv preprint arXiv:1908.01755, 2020.

\bibitem{Dwork2012}
Cynthia Dwork, Toniann~Pitassi Moritz~Hardt, Omer Reingold, and Richard Zemel.
\newblock Fairness through awareness.
\newblock In {\em ITCS '12: Proceedings of the 3rd Innovations in Theoretical
  Computer Science Conference}, pages 214--226, 2012.

\bibitem{hardt2016equality}
Moritz Hardt, Eric Price, and Nati Srebro.
\newblock Equality of opportunity in supervised learning.
\newblock In {\em NIPS'16: Proceedings of the 30th International Conference on
  Neural Information Processing Systems}, page 3323–3331, 2016.

\bibitem{MenonWilliamson(18)}
Aditya~Krishna Menon and Robert~C Williamson.
\newblock The cost of fairness in binary classification.
\newblock In {\em Proceedings of the 1st Conference on Fairness, Accountability
  and Transparency}, pages 107--118. 2018.

\bibitem{DoniniEtAl(18)}
Michele Donini, Luca Oneto, Shai Ben-David, John~R Shawe-Taylor, and
  Massimiliano~A. Pontil.
\newblock Empirical risk minimization under fairness constraints.
\newblock In {\em NIPS'18: Proceedings of the 32nd International Conference on
  Neural Information Processing Systems}, page 2796–2806, 2018.

\bibitem{ZafarEtAl(19)}
Muhammad~Bilal Zafar, Isabel Valera, Manuel Gomez-Rodriguez, and Krishna~P.
  Gummadi.
\newblock Fairness constraints: A flexible approach for fair classification.
\newblock {\em Journal of Machine Learning Research}, 20(75):1--42, 2019.

\bibitem{PleissEtAl(17)}
Geoff Pleiss, Manish Raghavan, Felix Wu, Jon Kleinberg, and Kilian~Q
  Weinberger.
\newblock On fairness and calibration.
\newblock In {\em Advances in Neural Information Processing Systems 30 (NIPS
  2017)}, pages 5680--5689. 2017.

\bibitem{KimEtAl(19)}
Michael~P. Kim, Amirata Ghorbani, and James Zou.
\newblock Multiaccuracy: Black-box post-processing for fairness in
  classification.
\newblock In {\em Proceedings of the 2019 AAAI/ACM Conference on AI, Ethics,
  and Society}, page 247–254, 2019.

\bibitem{little2019statistical}
Roderick~JA Little and Donald~B Rubin.
\newblock {\em Statistical analysis with missing data}, volume 793.
\newblock John Wiley \& Sons, 2019.

\bibitem{zeng2014rule}
Guoping Zeng and Qi~Zhao.
\newblock A rule of thumb for reject inference in credit scoring.
\newblock {\em Math. Finance Lett.}, 2014:Article--ID, 2014.

\bibitem{nguyen2016reject}
Ha-Thu Nguyen et~al.
\newblock Reject inference in application scorecards: evidence from france.
\newblock Technical report, University of Paris Nanterre, EconomiX, 2016.

\bibitem{kallus2018residual}
Nathan Kallus and Angela Zhou.
\newblock Residual unfairness in fair machine learning from prejudiced data.
\newblock {\em arXiv preprint arXiv:1806.02887}, 2018.

\bibitem{CostonEtAl(20)}
Amanda Coston, Alan Mishler, Edward~H. Kennedy, and Alexandra Chouldechova.
\newblock Counterfactual risk assessments, evaluation and fairness.
\newblock In {\em FAT* '20: Proceedings of the 2020 Conference on Fairness,
  Accountability, and Transparency}, pages 582--593, 2020.

\bibitem{LiEtAl(20)}
Zhiyong Li, Xinyi Hu, Ke~Li, Fanyin Zhou, and Feng Shen.
\newblock Inferring the outcomes of rejected loans: An application of
  semisupervised clustering.
\newblock {\em Journal of the Royal Statistical Society: Series A},
  183(2):631--654, 2020.

\bibitem{mancisidor2020deep}
Rogelio~A Mancisidor, Michael Kampffmeyer, Kjersti Aas, and Robert Jenssen.
\newblock Deep generative models for reject inference in credit scoring.
\newblock {\em Knowledge-Based Systems}, page 105758, 2020.

\bibitem{coston2019fair}
Amanda Coston, Karthikeyan~Natesan Ramamurthy, Dennis Wei, Kush~R Varshney,
  Skyler Speakman, Zairah Mustahsan, and Supriyo Chakraborty.
\newblock Fair transfer learning with missing protected attributes.
\newblock In {\em Proceedings of the 2019 AAAI/ACM Conference on AI, Ethics,
  and Society}, pages 91--98, 2019.

\bibitem{SinghViolations}
Harvineet Singh, Rina Singh, Vishwali Mhasawade, and Rumi Chunara.
\newblock Fairness violations and mitigation under covariate shift.
\newblock Technical report, arXiv preprint arXiv:1911.00677, 2021.

\bibitem{FeldmanEtAl(15)}
Michael Feldman, Sorelle~A. Friedler, John Moeller, Carlos Scheidegger, and
  Suresh Venkatasubramanian.
\newblock Certifying and removing disparate impact.
\newblock In {\em KDD '15: Proceedings of the 21th ACM SIGKDD International
  Conference on Knowledge Discovery and Data Mining}, pages 259--268, 2015.

\bibitem{barocas-hardt-narayanan}
Solon Barocas, Moritz Hardt, and Arvind Narayanan.
\newblock {\em Fairness and Machine Learning}.
\newblock fairmlbook.org, 2019.
\newblock \url{http://www.fairmlbook.org}.

\bibitem{CorbettDaviesEtAl2017}
Sam Corbett-Davies, Emma Pierson, Avi Feller, Sharad Goel, and Aziz Huq.
\newblock Algorithmic decision making and the cost of fairness.
\newblock pages 797--806, 2017.

\bibitem{MitchellEtAl(19)}
Shira Mitchell, Eric Potash, Solon Barocas, Alexander D'Amour, and Kristian
  Lum.
\newblock Prediction-based decisions and fairness: A catalogue of choices,
  assumptions, and definitions.
\newblock Technical report, arXiv Working Paper, arXiv:1811.07867, 2019.

\bibitem{CotterEtAl(19)-ALT}
Andrew Cotter, Heinrich Jiang, and Karthik Sridharan.
\newblock Two-player games for efficient non-convex constrained optimization.
\newblock In {\em Proceedings of the 30th International Conference on
  Algorithmic Learning Theory}, pages 300--332, 2019.

\bibitem{ImbensRubin(15)}
Guido~W Imbens and Donald~B Rubin.
\newblock {\em Causal Inference for Statistics, Social and Biomedical Sciences:
  An Introduction}.
\newblock Cambridge University Press, Cambridge, United Kingdom, 2015.

\bibitem{Heckman(90)}
James~J. Heckman.
\newblock Varieties of selection bias.
\newblock {\em The American Economic Review}, 80(2):313--318, 1990.

\bibitem{moreno2012unifying}
Jose~G Moreno-Torres, Troy Raeder, Roc{\'\i}o Alaiz-Rodr{\'\i}guez, Nitesh~V
  Chawla, and Francisco Herrera.
\newblock A unifying view on dataset shift in classification.
\newblock {\em Pattern Recognition}, 45(1):521--530, 2012.

\bibitem{crook2004does}
Jonathan Crook and John Banasik.
\newblock Does reject inference really improve the performance of application
  scoring models?
\newblock {\em Journal of Banking \& Finance}, 28(4):857--874, 2004.

\bibitem{SEIFA(16)}
{Australian Bureau of Statistics}.
\newblock Socio-economic indexes for areas (seifa) technical paper.
\newblock Technical report, 2016.

\bibitem{AngwinEtAl(16)}
Julia Angwin, Jeff Larson, Surya Mattu, and Lauren Kirchner.
\newblock Machine bias. there’s software used across the country to predict
  future criminals. and it’s biased against blacks.
\newblock {\em ProPublica}, 2016.

\bibitem{Rudin2020Age}
Cynthia Rudin, Caroline Wang, and Beau Coker.
\newblock The age of secrecy and unfairness in recidivism prediction.
\newblock {\em Harvard Data Science Review}, 2(1), 2020.

\bibitem{angelino2018learning}
Elaine Angelino, Nicholas Larus-Stone, Daniel Alabi, Margo Seltzer, and Cynthia
  Rudin.
\newblock Learning certifiably optimal rule lists for categorical data.
\newblock {\em Journal of Machine Learning Research}, 18:1--78, 2018.

\bibitem{mishler2019modeling}
Alan Mishler.
\newblock Modeling risk and achieving algorithmic fairness using potential
  outcomes.
\newblock In {\em Proceedings of the 2019 AAAI/ACM Conference on AI, Ethics,
  and Society}, pages 555--556, 2019.

\bibitem{Dua:2019}
Dheeru Dua and Casey Graff.
\newblock {UCI} machine learning repository, 2017.

\end{thebibliography}

\clearpage
\newpage
\appendix

\section{Additional Theoretical Results}\label{section: additional theoretical results}

\subsection{Implementation of the Exponentiated Gradient Algorithm in \textsection~\ref{section: range of predictive disparities}}\label{sec: exp grad for range}
We now provide the details of the exponentiated gradient algorithm discussed in \textsection~\ref{section: range of predictive disparities} for finding the predictive disparity minimizing model within the set of good models. Algorithm \ref{alg: exp grad for fairness frontier, extremes} implements the exponentiated gradient algorithm, except for the best-response functions of the $\lambda$-player and the $Q_h$-player. The best-response function of the $\lambda$-player is
    \begin{equation}
        \best_{\lambda}(Q_h) := \begin{cases}
            0 \mbox{ if } \widehat{\cost}(Q_h) - \hat{\epsilon} \leq 0, \\
            B_\lambda \mbox{ otherwise}.
        \end{cases}
    \end{equation}
The best-response function of the $Q_h$-player may be constructed through a further reduction to cost-sensitive classification. The Lagrangian may be re-written as 
    \begin{equation}
        L(h_f, \lambda) = \hat{\e}\left[ \e_{Z_\alpha}\left[ c_\lambda(\underline{Y}^*_i, A_i, Z_\alpha) h_f(X_i, Z_\alpha) \right] \right] - \lambda \hat{\epsilon},
    \end{equation}
where 
\begin{eqnarray} \label{equation: lambda cost for base case}
 c_\lambda(\underline{Y}^*_i, A_i, Z_\alpha) := \frac{\beta_0}{\hat{p}_0} 1\left\{ \cE_{i,0} \right\} + \frac{\beta_1}{\hat{p}_1} 1\left\{ \cE_{i,1} \right\} + \lambda c(\underline{Y}^*_i, Z_\alpha)
\end{eqnarray}
and $\hat{p}_a := \hat{\e}\left[ \cE_{i,a} \right]$ for $a \in \{0, 1\}$. This is solved by calling cost-sensitive classification oracle on an augmented dataset of size $n \times N$ with observations $\{(X_{i, z_\alpha}, C_{i, z_\alpha}\}_{i \in [n], z_\alpha \in \cZ_\alpha}$ with $X_{i, z_\alpha} = (X_i, z_\alpha)$ and $C_{i, z_\alpha} = c_\lambda(\underline{Y}^*_i, A_i, z_\alpha)$. In our empirical implementation, we use the heuristic least-squares reduction described in \cite{AgarwalEtAl(19)-FairRegression}, which eases the computational burden of the algorithm. The heuristic reduction generally performed well in our empirical work, but performance losses depended on the dataset and the choice of predictive disparity.

\begin{algorithm}[hb!]
\SetAlgoLined
\caption{Algorithm for finding the predictive disparity minimizing model}
\label{alg: exp grad for fairness frontier, extremes}
\KwIn{
    Training data $\{(X_i, Y_i, A_i)\}_{i=1}^{n}$, parameters $\beta_0, \beta_1$, events $\cE_{i,0}, \cE_{i,1}$, empirical loss tolerance $\hat \epsilon$, bound $B_\lambda$, accuracy $\nu$ and learning rate $\eta$.
}
\KwResult{$\nu$-approximate saddle point $(\hat{Q}_{h}, \hat{\lambda})$}
Set $\theta_1 = 0 \in \reals$ \;
\For{$t = 1,2, \hdots$}{ 
    Set $\lambda_{t} = B_\lambda \frac{\exp(\theta_{t})}{1 + \exp(\theta_t)}$\;
    $h_t \leftarrow \best_h(\lambda_t)$\;
    $\hat Q_{h, t} \leftarrow \frac{1}{t} \sum_{s=1}^{t} h_{s}$, $\quad$ $\bar{L} \leftarrow L(\hat{Q}_{h, t}, \best_{\lambda}(\hat{Q}_{h, t})$\;
    $\hat \lambda_t \leftarrow \frac{1}{t} \sum_{s=1}^{t} \lambda_s$, $\quad$ $\underline{L} \leftarrow L(\best_h(\hat{\lambda_t}), \hat{\lambda_t})$\;
    $\nu_t \leftarrow \max\left\{ L(\hat{Q}_{h,t}, \hat{\lambda}_t) - \underline{L}, \bar{L} - L(\hat{Q}_{h, t}, \hat{\lambda}_t) \right\}$\;
    \If{$\nu_t \leq \nu$}{
        \eIf{ $\widehat{\cost}(\Qhat_{h, t}) \leq \hat{\epsilon} + \frac{ |\beta_0| + |\beta_1| + 2\nu }{ B_\lambda }$ }{
            \Return{$(\hat{Q}_{h, t}, \hat{\lambda}_t)$}\;
        }{ 
            \Return{ null }
        }
    }
    Set $\theta_{t+1} = \theta_{t} + \eta \left( \widehat{\cost}(h_t) - \hat \epsilon \right)$\;
} 
\end{algorithm}

\subsection{Shrinking the Support of the Stochastic Risk Score}\label{section: linear program reduction}

As discussed in \textsection~\ref{section: range of predictive disparities}, a key challenge to the practical use of Algorithm \ref{alg: exp grad for fairness frontier, extremes} is it returns a stochastic prediction function $\hat{Q}_h$ with possibly large support. The number of prediction functions in the support of $\hat{Q}_h$ is equal to the total number of iterations taken by the respective algorithm. As a result, $\hat{Q}_h$ may be complex to describe, time-intensive to evaluate, and memory-intensive to store.

The support of the returned stochastic prediction may be shrunk while maintaining the same guarantees on its performance by solving a simple linear program. To do so, we take the set of prediction functions in the support of $\hat{Q}_h$ and solve the following linear program
    \begin{equation}\label{equation: lp shrinkage}
        \min_{p \in \Delta^T} \sum_{t=1}^{T} p_t \widehat{\disp}(h_{t}) \mbox{ s.t. } \sum_{t=1}^{T} p_t \widehat{\cost}(h_t) \leq \hat{\epsilon} + 2\nu, 
    \end{equation}
where $T$ is the number of iterations of Algorithm \ref{alg: exp grad for fairness frontier, extremes}, $\Delta^{T}$ is the $T$-dimensional unit simplex and $h_t$ is the $t$-th prediction function in the support of $\hat{Q}_h$ (i.e., the prediction function constructed at the $t$-th iteration of Algorithm \ref{alg: exp grad for fairness frontier, extremes}). We then use the randomized prediction function that assigns probability $p_t$ to each prediction function in the support of $\hat{Q}_h$. In practice, we calibrate the constraint in (\ref{equation: lp shrinkage}) by choosing the smallest $\nu \geq 0$ such that the linear program has a feasible solution, following the practical recommendations in \cite{CotterEtAl(19)-ALT}. 

Lemma 7 of \cite{CotterEtAl(19)-ALT} shows that the solution to (\ref{equation: lp shrinkage}) has at most $2$ support points and the same performance guarantees as the original solution $\hat{Q}_h$.

\subsection{Computing the Absolute Predictive Disparity Minimizing Model}\label{section: computing abs disp min model}

In this section, we extend the reductions approach to compute the prediction function that minimizes the absolute predictive disparity over the set of good models (\ref{equation: absolute disparity min problem}). We solve
    \begin{equation}\label{equation: absolute disparity min problem, randomized predictor}
        \min_{Q \in \Delta(\cF)} \left| \disp(Q) \right| \mbox{ s.t. } \loss(Q) \leq \epsilon.
    \end{equation}
Through the same discretization argument, this problem may be reduced to a constrained classification problem over the set of threshold classifiers
    \begin{equation}
        \min_{Q_h \in \Delta(\cH)} \left| \disp(Q_h) \right| \mbox{ s.t. } \cost(Q_h) \leq \epsilon - c_0.
    \end{equation}
    
To further deal with the absolute value operator in the objective function, we introduce a slack variable $\xi$ and define the equivalent problem over both $Q_h \in \Delta(\cH)$, $\xi \in \reals$
    \begin{align*}
        \min_{\xi, Q_h \in \Delta(\cH)} \xi \numberthis \label{equation: disparity minimizing, randomized, slack variable} \\
        \mbox{ s.t. } & \disp(Q_h) - \xi \leq 0, \\
        & -\disp(Q_h) - \xi \leq 0, \\
        & \cost(Q_h) \leq \epsilon - c_0.
    \end{align*}
We construct solutions to the empirical analogue of (\ref{equation: disparity minimizing, randomized, slack variable}).
    
Solving the empirical analogue of (\ref{equation: disparity minimizing, randomized, slack variable}) is equivalent to finding the saddle point $\min_{Q_h \in \Delta(\cH), \xi \in [0, B_\xi]} \max_{\|\lambda\| \leq B_\lambda} L(\xi, Q_h, \lambda)$ with Lagrangian $L(\xi, Q_h, \lambda) = \xi + \lambda_{+} \left( \hat \disp(Q_h) - \xi \right) + \lambda_{-} \left(-\hat\disp(Q_h) - \xi \right) + \lambda_{\cost} \left( \costhat(Q_h) - \hat{\epsilon} \right)$, $\lambda = \left(\lambda_{+}, \lambda_{-}, \lambda_{\cost}\right)$ and $B_\xi$ is a bound on the slack variable. Since the absolute predictive disparity is bounded by one, we define $B_\xi = 1$ in practice. We search for the saddle point by treating it as the equilibrium of a two-player zero-sum game in which one player chooses $(\xi, Q_h)$ and the other chooses $\lambda$.

Algorithm \ref{alg: exp grad for fairness frontier, disparity minimizing} computes a $\nu$-approximate saddle point of $L(\xi, Q_h, \lambda)$. The best-response of the $\lambda$-player sets the Lagrange multiplier associated with the maximally violated constraint equal to $B_\lambda$. Otherwise, she sets all Lagrange multipliers to zero if all constraints are satisfied. In order to analyze the best-response of the $(\xi, Q_h)$-player, rewrite the Lagrangian as 
\begin{align*}
    L(\xi, Q_h, \lambda) &= (1-\lambda_+-\lambda_-)\xi \numberthis \\
    & + (\lambda_{+} - \lambda_{-}) \disphat(Q_h) + \lambda_{\cost} (\costhat(Q_h) - \hat{\epsilon}).
\end{align*}
For a fixed value of $\lambda$, minimizing $L(\xi, Q_h, \lambda)$ over $(\xi, Q_h)$ jointly is equivalent to separately minimizing the first term involving $\xi$ and the remaining terms involving $Q_h$. To minimize $(1-\lambda_+-\lambda_-)\xi$, the best-response is to set $\xi = B_\xi$ if $1-\lambda_+-\lambda_- < 0$, and set $\xi = 0$ otherwise. Minimizing 
\begin{equation}
    (\lambda_{+} - \lambda_{-}) \disphat(Q_h) + \lambda_{\cost} (\costhat(Q_h) - \hat{\epsilon})
\end{equation}
over $Q_h$ can be achieved through a reduction to cost-sensitive classification since minimizing the previous display is equivalent to minimizing
    \begin{equation}
        \hat{\e}\left[ \e_{Z_\alpha}\left[ c_\lambda(\underline{Y}^*_i, A_i, Z_\alpha) h_f(X_i, Z_\alpha) \right] \right],
    \end{equation}
where now $c_\lambda(\underline{Y}^*_i, A_i, Z_\alpha) := (\lambda_+ - \lambda_-) \left( \frac{\beta_0}{\hat{p}_0} 1\left\{ \cE_{i,0} \right\} + \frac{\beta_1}{\hat{p}_1} 1\left\{ \cE_{i,1} \right\} \right) + \lambda_{\cost} c(\underline{Y}^*_i, Z_\alpha)$.  

We use an analogous linear program reduction (\textsection~\ref{section: linear program reduction}) to shrink the support of the solution returned by Algorithm \ref{alg: exp grad for fairness frontier, disparity minimizing}.

\begin{algorithm}[htbp!]
\SetAlgoLined
\KwIn{Training data $\{(X_i, Y_i, A_i)\}_{i=1}^{n}$, \\
    Parameters $\beta_0, \beta_1$, Events $\cE_{i,0}, \cE_{i,1}$, and empirical loss tolerance $\hat \epsilon$ \\
    Bounds $B_\lambda$, $B_\xi$, accuracy $\nu$ and learning rate $\eta$
}
\KwResult{$\nu$-approximate saddle point $(\hat \xi, \hat{Q}, \hat{\lambda})$}
Set $\theta_1 = 0 \in \mathbb{R}^{3}$ \;
\For{$t = 1,2, \hdots$}{ 
    Set $\lambda_{t, k} = B_\lambda \frac{\exp(\theta_{t, k})}{1 + \sum_{k'} \exp(\theta_{t,k'})}$ for all $k = \{\cost, +, -\}$\;
    $h_t \leftarrow \best_h(\lambda_t)$, $\quad$ $\xi_t \leftarrow \best_{\xi}(\lambda_t)$ \;
    $\hat Q_{h,t} \leftarrow \frac{1}{t} \sum_{s=1}^{t} h_{s}$, $\quad$ $\hat{\xi}_t \leftarrow \frac{1}{t} \sum_{s=1}^{t} \xi_t$ \; 
    $\bar{L} \leftarrow L(\hat{\xi}_t, \hat{Q}_t, \best_{\lambda}(\hat{\xi}_t. \hat{Q}_t)$\;
    $\hat \lambda_t \leftarrow \frac{1}{t} \sum_{s=1}^{t} \lambda_s$, $\quad$ $\underline{L} \leftarrow L(\best_{\xi}(\lambda_t), \best_h(\hat{\lambda_t}), \hat{\lambda_t})$\;
    $\nu_t \leftarrow \max\left\{ L(\hat{\xi}_t, \hat{Q}_t, \hat{\lambda}_t) - \underline{L}, \bar{L} - L(\hat{\xi}_t, \hat{Q}_t, \hat{\lambda}_t) \right\}$\;
    \If{$\nu_t \leq \nu$}{
        \eIf{$\widehat{\cost}(\Qhat) \leq \hat{\epsilon} + \frac{ B_\xi + 2 \nu }{ B_\lambda }$}{
            \Return{$(\hat{\xi}_t, \hat{Q}_t, \hat{\lambda}_t)$}\;
        }{
            \Return{$null$}\;
        }
    }
    Set $\theta_{t+1} = \theta_{t} + \eta \begin{pmatrix} \widehat{\disp}(h_t) - \xi_t  \\ -\widehat{\disp}(h_t) - \xi_t \\ \widehat{\cost}(h_t) - \hat \epsilon \end{pmatrix} $\;
} 
\caption{Algorithm for finding the absolute predictive disparity minimizing model among the set of good models}
\label{alg: exp grad for fairness frontier, disparity minimizing}
\end{algorithm}

\subsubsection{Error Analysis}

We analyze the suboptimality of the solution returned by Algorithm \ref{alg: exp grad for fairness frontier, disparity minimizing}.

\begin{theorem}\label{theorem: error analysis for exp grad, disp minimizing}
Suppose Assumption \ref{assumption: error analysis for exp grad extremes} holds for $C' \geq 2C + 2 + \sqrt{2 \ln(8 N/\delta)}$ and $C^{\prime \prime} \geq \sqrt{\frac{-\log(\delta / 8)}{2}}$.

Then, Algorithm \ref{alg: exp grad for fairness frontier, disparity minimizing} with $\nu \propto n^{-\phi}, B_\lambda \propto n^\phi, N \propto n^\phi$ terminates in at most $O(n^{4\phi})$ iterations. It returns $\Qhat_h$, which when viewed as a distribution over $\cF$, satisfies with probability at least $1 - \delta$ either one of the following: 1) $\hat{Q}_h \neq null$, $\loss(\Qhat_h) \leq \epsilon + \tilde{O}(n^{-\phi})$ and $\left| \disp(\Qhat_h) \right| \leq \left| \disp(\tilde Q) \right| + \tilde{O}(n_0^{-\phi}) + \tilde{O}(n_1^{-\phi})$ for any $\tilde Q$ that is feasible in (\ref{equation: absolute disparity min problem, randomized predictor}); or 2) $\Qhat_h = null$ and (\ref{equation: absolute disparity min problem, randomized predictor}) is infeasible.
\end{theorem}

We next provide an oracle result for the absolute disparity minimizing algorithm under selective labels.

\begin{theorem}[Selective Labels for Algorithm \ref{alg: exp grad for fairness frontier, disparity minimizing}]\label{theorem: selective labels error analysis for exp grad, disp minimizing}
Suppose Assumption \ref{assumption: selection on obs and positivity} holds and Algorithm \ref{alg: exp grad for fairness frontier, disparity minimizing} is given as input the modified training data $\{(X_i, A_i, \mu(X_i)\}_{i=1}^{n}$.

Under the same conditions as Theorem \ref{theorem: error analysis for exp grad, disp minimizing}, Algorithm \ref{alg: exp grad for fairness frontier, disparity minimizing} terminates in at most $O(n^{4\phi})$ iterations. It returns $\Qhat_h$, which when viewed as a distribution over $\cF$, satisfies with probability at least $1 - \delta$ either one of the following: 1) $\hat{Q}_h \neq null$, $\loss_\mu(\Qhat_h) \leq \epsilon + \tilde{O}(n^{-\phi})$ and $\left| \disp(\Qhat_h) \right| \leq \left| \disp(\tilde Q) \right| + \tilde{O}(n_0^{-\phi}) + \tilde{O}(n_1^{-\phi})$ for any $\tilde Q$ that is feasible in (\ref{equation: absolute disparity min problem, randomized predictor}); or 2) $\Qhat_h = null$ and (\ref{equation: absolute disparity min problem, randomized predictor}) is infeasible.
\end{theorem}

We omit the proof of Theorem \ref{theorem: selective labels error analysis for exp grad, disp minimizing}  since the analogous steps are given in proofs of Theorems~\ref{theorem:selective labels extremes}-\ref{theorem: error analysis for exp grad, disp minimizing} below.

\subsection{Bounded Group Loss Disparity}\label{section: bounded group loss special case}

Bounded group loss is a common notion of predictive fairness that examines the variation in average loss across values of the protected or sensitive attribute. It is commonly used to ensure that the prediction function achieves some minimal threshold of predictive performance across all values of the attribute \cite{AgarwalEtAl(19)-FairRegression}. We define a bounded group loss disparity to be the difference in average loss across values of the attribute, $\disp(f) = \e\left[ l(Y^*_i, f(X_i)) \mid A_i = 1 \right] - \e\left[ l(Y^*_i, f(X_i)) \mid A_i = 0 \right]$. This choice of predictive disparity measure is convenient as it allows us to drastically simplify our algorithm by skipping the discretization step entirely and reducing the problem to an instance of weighted loss minimization. \cite{AgarwalEtAl(19)-FairRegression} apply the same idea in their analysis of fair regression under bounded group loss. 

Take, for example, the problem of finding the range of bounded group loss disparities that are possible over the set of good models. Letting $\loss(f \mid A_i = a) := \e\left[ l(Y^*_i, f(X_i) \mid A_i = a \right]$ and $\loss(Q \mid A_i = a) := \sum_{f \in \cF} Q(f) \loss(f \mid A_i = a)$, we solve
    \begin{align*}
        \min_{Q \in \Delta(\cF)} & \loss(Q \mid A_i = 1) - \loss(f \mid A_i = 0) \\
        & \mbox{ s.t. } \loss(Q) \leq \epsilon.
    \end{align*}
The sample version of this problem is to minimize $\widehat{\loss}(Q \mid A_i = 1) - \widehat{\loss}(f \mid A_i = 0)$ subject to $\widehat{\loss}(Q) \leq \epsilon$. We solve the sample problem by finding a saddle point of the associated Lagrangian $L(Q, \lambda) = \widehat{\loss}(Q \mid A_i = 1) - \widehat{\loss}(f \mid A_i = 0) + \lambda (\widehat{\loss}(Q) - \epsilon)$. We compute a $\nu$-approximate saddle point by treating it as a zero-sum game between a $Q$-player and a $\lambda$-player. The best response of the $\lambda$-player is the same as before: if the constraint $\hat{\loss}(Q) - \epsilon$ is violated, she sets $\lambda = B_\lambda$, and otherwise she sets $\lambda = 0$. The best-response of the $Q$-player may reduced to an instance of weighted loss minimization since
    \begin{align*}
        & \widehat{\loss}(f | \cE_{i,0}) - \widehat{\loss}(f | \cE_{i,1}) + \lambda (\widehat{\loss}(f) - \epsilon) \\
        &= \hat{\mathbb{E}}\left[ \left( \frac{1}{\hat p_0} 1\{\cE_0\} - \frac{1}{\hat p_1} 1\{\cE_1\} + \lambda \right) l(Y_i, f(X_i)) \right]
    \end{align*}
Therefore, defining the weights $W_i = \frac{1}{\hat p_0} 1\{\cE_{i,0}\} - \frac{1}{\hat p_1} 1\{\cE_{i,1}\} + \lambda$, we see that minimizing $L(h, \lambda)$ is equivalent to solving an instance of weighted loss minimization. Algorithm \ref{alg: exp grad for fairness frontier, extremes BGL} formally states the procedure for finding the range of bounded group loss disparities. We may analogously extend Algorithm \ref{alg: exp grad for fairness frontier, disparity minimizing} to find the absolute bounded group loss minimizing model among the set of good models. 

\begin{algorithm}[htbp!]
\SetAlgoLined
\KwIn{Training data $\{(X_i, Y_i, A_i)\}_{i=1}^{n}$, \\
    Parameters $\beta_0, \beta_1$, Events $\cE_{i,0}, \cE_{i,1}$, and loss tolerance $\hat \epsilon$ \\
    Bound $B_\lambda$, accuracy $\nu$ and learning rate $\eta$
}
\KwResult{$\nu$-approximate saddle point $(\hat{Q}_{h}, \hat{\lambda})$}
Set $\theta_1 = 0 \in \reals$ \;
\For{$t = 1,2, \hdots$}{ 
    Set $\lambda_{t} = B_\lambda \frac{\exp(\theta_{t})}{1 + \exp(\theta_t)}$\;
    $f_t \leftarrow \best_f(\lambda_t)$\;
    $\hat Q_{t} \leftarrow \frac{1}{t} \sum_{s=1}^{t} f_{s}$, $\quad$ $\bar{L} \leftarrow L(\hat{Q}_{t}, \best_{\lambda}(\hat{Q}_{t})$\;
    $\hat \lambda_t \leftarrow \frac{1}{t} \sum_{s=1}^{t} \lambda_s$, $\quad$ $\underline{L} \leftarrow L(\best_f(\hat{\lambda_t}), \hat{\lambda_t})$\;
    $\nu_t \leftarrow \max\left\{ L(\hat{Q}_t, \hat{\lambda}_t) - \underline{L}, \bar{L} - L(\hat{Q}_t, \hat{\lambda}_t) \right\}$\;
    \If{$\nu_t \leq \nu$}{
        \eIf{ $\widehat{\loss}(\Qhat_{t}) \leq \hat{\epsilon} + \frac{ |\beta_0| + |\beta_1| + 2\nu }{ B_\lambda }$ }{
            \Return{$(\hat{Q}_{t}, \hat{\lambda}_t)$}\;
        }{ 
            \Return{ null }
        }
    }
    Set $\theta_{t+1} = \theta_{t} + \eta \left( \widehat{\loss}(f_t) - \hat \epsilon \right)$\;
} 
\caption{Algorithm for finding the bounded group loss disparity minimizing model over the set of good models}
\label{alg: exp grad for fairness frontier, extremes BGL}
\end{algorithm}

\section{Proofs of Main Results}\label{section: proofs of main results}

\subsection*{Proof of Lemma \ref{lemma: disparity classifier relates to disparity prediction function}}

Fix $f \in \cF$. For $x \in \cX$ and $z_{\alpha} \in \cZ_{\alpha}$
    \[
    h_f(x, z_{\alpha})  = 1\{f(x) \geq z_\alpha\} = 1\{ \underline{f}(x) \geq z_\alpha\},
    \]
Therefore,
    \[
    \mathbb{E}_{Z_\alpha}\left[ h_f(x, Z_\alpha) \right] = \mathbb{E}_{Z_\alpha}\left[ 1\{\underline{f}(x) \geq Z_\alpha\} \right] = \underline{f}(x),
    \]
and for any $a \in \{0, 1\}$,
    \begin{align*}
    & | \mathbb{E}\left[ h_f(X, Z_\alpha) | \cE_{i,a} \right] - \mathbb{E}\left[ f(X) | \cE_{i,a} \right] | \\
    &= | \mathbb{E}\left[ \mathbb{E}_{Z_\alpha}\left[ h_f(X, Z_\alpha) \right] - f(X) | \cE_{i,a} \right] | \\
    &= | \mathbb{E}\left[ \underline{f}(X) - f(X) | \cE_{i,a} \right] | \leq \alpha
    \end{align*}
where the first equality uses iterated expectations plus the fact that $Z_\alpha$ is independent of $(X, A, Y^*)$ and the final equality follows by the definition of $\underline{f}(X)$. The claim is immediate after noticing $\disp(h_f) - \disp(f)$ equals $\beta_0 \left( \mathbb{E}\left[ h_f(X, Z_\alpha) - f(X) | \cE_{i,0} \right] \right) + \beta_1 \left( \mathbb{E}\left[ h_f(X, Z_\alpha) - f(X) | \cE_{i,1} \right] \right)$ and applying the triangle inequality. $\Box$

\subsection*{Proof of Theorem \ref{theorem: error analysis for exp grad extremes}} 

The claim about the iteration complexity of Algorithm \ref{alg: exp grad for fairness frontier, extremes} follows immediately from Lemma \ref{lemma: iteration complexity for extremes}, substituting in the stated choices of $\nu$ and $B$. 

The proof strategy for the remaining claims follows the proof of Theorems 2-3 in \cite{AgarwalEtAl(19)-FairRegression}. We consider two cases.
    
\paragraph{Case 1: There is a feasible solution $Q^*$ to the population problem (\ref{equation: disp range, randomizaton predictor problem})} 
    
Using Lemmas \ref{lemma: lemma 2 of fair classification, extremes}-\ref{lemma: lemma 3 of fair classification, extremes}, the $\nu$-approximate saddle point $\Qhat_h$ satisfies
    \begin{align}
        & \widehat{\disp}(\Qhat_h) \leq \widehat{\disp}(Q_h) + 2\nu \\
        & \costhat(\Qhat_h) \leq \hat{\epsilon} + \frac{ |\beta_0|+ |\beta_1| + 2\nu }{B}
    \end{align}
for any distribution $Q_h$ that is feasible in the empirical problem. This implies that Algorithm \ref{alg: exp grad for fairness frontier, extremes} returns $\hat{Q} \neq null$. We now show that the returned $\hat{Q}_h$ provides an approximate solution to the discretized population problem.
    
First, define $\widehat{\cost}_{z}(h) := \hat{\mathbb{E}}\left[ c(\underline{Y}^*_i, z) h(X_i, z) \right]$ and $\cost_z(h) := \mathbb{E}\left[ c(\underline{Y}^*_i, z) h(X_i, z) \right]$. Since $c(\underline{Y}^*_i, z) \in [-1, 1]$, we invoke Lemma \ref{lemma: lemma 2 in fair reg paper} with $S_i = c(\underline{Y}^*_i, z_i)$, $U_i = (X_i, z)$, $\cG = \cH$ and $\psi(s, t) = st$ to obtain that with probability at least $1 - \frac{\delta}{4}$ for all $z \in \cZ_\alpha$ and $h \in \cH$
        \[
        \left| \widehat{\cost}_z(h) - \cost_z(h) \right| \leq 
        \]
        \[
        2 R_n(\cH) + \frac{2}{\sqrt{n}} + \sqrt{ \frac{2 \ln(8 N/\delta)}{ n } } = \tilde{O}(n^{-\phi}),
        \]
where the last equality follows by the bound on $R_n(\cH)$ in Assumption \ref{assumption: error analysis for exp grad extremes} and setting $N \, \propto \, n^\phi$. Averaging over $z \in \cZ_\alpha$ and taking a convex combination of according to $Q_h \in \Delta(\cH)$ then delivers via Jensen's Inequality that with probability at least $1 - \delta/4$ for all $Q \in \Delta(\cH)$
        \begin{equation}\label{equation: cost bound i, extremes}
            \left| \widehat{\cost}(Q_h) - \cost(Q_h) \right| \leq \tilde{O}(n^{-\phi}).
        \end{equation}
        
Next, define $\widehat{\disp}_z(h) := \beta_0 \hat{\mathbb{E}}\left[ h(X_i, z) | \cE_{i,0} \right] + \beta_1 \hat{\mathbb{E}}\left[ h(X_i, z) | \cE_{i,1} \right]$ and $\disp_z(h) := \beta_0 \mathbb{E}\left[ h(X_i, z) | \cE_{i,0} \right] + \beta_1 \mathbb{E}\left[ h(X_i, z) | \cE_{i,1} \right]$, where the difference can be expressed as 
    \begin{align*}
        & \widehat{\disp}_z(h) - \disp_z(h) = \\
        & \beta_0 \left( \hat{\mathbb{E}}\left[ h(X_i, z) | \cE_{i,0} \right] - \mathbb{E}\left[ h(X_i, z) | \cE_{i,0} \right] \right) + \\
        & \beta_1 \left(  \hat{\mathbb{E}}\left[ h(X_i, z) | \cE_{i,1} \right] - \mathbb{E}\left[ h(X_i, z) | \cE_{i,1} \right] \right).
    \end{align*}
Therefore, by the triangle inequality, 
    \begin{align*}
        & \left| \widehat{\disp}_z(h) - \disp_z(h) \right| \leq \\
        &|\beta_0| \left| \hat{\mathbb{E}}\left[ h(X_i, z) | \cE_{i,0} \right] - \mathbb{E}\left[ h(X_i, z) | \cE_{i,0} \right] \right| + \\ 
        & |\beta_1| \left| \hat{\mathbb{E}}\left[ h(X_i, z) | \cE_{i,1} \right] - \mathbb{E}\left[ h(X_i, z) | \cE_{i,1} \right] \right|.
    \end{align*}
For each term on the right-hand side of the previous display, we invoke Lemma \ref{lemma: lemma 2 in fair reg paper} applied to the data distribution conditional on $\cE_0$ and $\cE_1$. We set $S = 1$, $U = (X_i, z)$, $\cG = \cH$ and $\psi(s, t) = st$. With probability at least $1 - \frac{\delta}{4}$ for all $z \in \cZ_\alpha$,
\[
\left| \hat{\mathbb{E}}\left[ h(X_i, z) | \cE_{i,0} \right] - \mathbb{E}\left[ h(X_i, z) | \cE_{i,0} \right] \right| \leq 
\]
\[
R_{n_0}(\cH) + \frac{2}{\sqrt{n_0}} + \sqrt{ \frac{ 2 \ln(8 N /\delta)}{n_0} },
\]
\[
\left| \hat{\mathbb{E}}\left[ h(X_i, z) | \cE_{i,1} \right] - \mathbb{E}\left[ h(X_i, z) | \cE_{i,1} \right] \right| \leq
\]
\[
R_{n_1}(\cH) + \frac{2}{\sqrt{n_1}} + \sqrt{ \frac{ 2 \ln(8 N /\delta)}{n_1} }.
\]
Then, averaging over $z \in \cZ_\alpha$ and taking a convex combination according to $Q_h \in \Delta(\cH)$ delivers via Jensen's Inequality that with probability at least $1 - \delta/4$ for all $Q \in \Delta(\cH)$
        \begin{align}
            & \left| \hat{\mathbb{E}}\left[ Q_h | \cE_{i,0} \right] - \mathbb{E}\left[ Q_h | \cE_{i,0} \right] \right| \leq R_{n_0}(\cH) + \frac{2}{\sqrt{n_0}} + \sqrt{ \frac{ 2 \ln(8 N /\delta)}{n_0} } \label{eqn: fairness bound (i), extremes} \\
            & \left| \hat{\mathbb{E}}\left[ Q_h | \cE_{i,1} \right] - \mathbb{E}\left[ Q_h | \cE_{i,1} \right] \right| \leq R_{n_1}(\cH) + \frac{2}{\sqrt{n_1}} + \sqrt{ \frac{ 2 \ln(8 N /\delta)}{n_1} } \label{eqn: fairness bound (ii), extremes}
        \end{align}
By the union bound, both inequalities hold with probability at least $1 - \delta/2$.
    
Finally, Hoeffding's Inequality implies that with probability at least $1 - \delta/4$,
        \begin{equation}\label{equation: Hoeffding for c0}
            \left| \hat{c}_0 - c_0 \right| \leq \sqrt{ \frac{ -\log(\delta / 8) }{2n} }.
        \end{equation}

From Lemma \ref{lemma: solution quality of exp grad for extremes}, we have that Algorithm \ref{alg: exp grad for fairness frontier, extremes} terminates and delivers a distribution $\Qhat_h$ that compares favorably against any feasible $Q$ in the discretized sample problem. That is, for any such $Q_h$, 
        \begin{equation}
            \disphat(\Qhat_h) \leq \disphat(Q_h) + O(n^{-\phi})
        \end{equation}
        \begin{equation}
            \costhat(\Qhat_h) \leq \hat{\epsilon} + O(n^{-\phi}) \label{eqn: costhat bound ii}
        \end{equation}
where we used the fact that $\nu \propto n^{-\phi}$ and $B \propto n^\phi$ by assumption. First, (\ref{equation: cost bound i, extremes}),  (\ref{equation: Hoeffding for c0}), (\ref{eqn: costhat bound ii}) imply
        \begin{equation}\label{equation: Qhat approximate feasibility}
            \cost(\Qhat_h) \leq \hat{\epsilon} + \tilde{O}(n^{-\phi}) \leq \epsilon - c_0 + \tilde{O}(n^{-\phi}),
        \end{equation}
where we used that $\hat{\epsilon} = \epsilon - \hat{\mathbb{E}}[l(\underline{Y}^*_i, \frac{\alpha}{2})] + C'n^{-\phi} - C'' n^{-1/2}$. by assumption. Second, the bounds in (\ref{eqn: fairness bound (i), extremes}), (\ref{eqn: fairness bound (ii), extremes}) imply 
        \begin{equation}\label{equation: Qhat approximate optimality}
            \disp(\hat{Q}_h) \leq \disp(Q_h) + \tilde{O}(n_0^{-\beta}) + \tilde{O}(n_1^{-\phi}). 
        \end{equation}
    
We assumed that $Q_h$ was a feasible point in the discretized sample problem. Assuming that (\ref{equation: cost bound i, extremes}) holds implies that any feasible solution of the population problem is also feasible in the empirical problem due to how we have set $C^\prime$ and $C^{\prime \prime}$. Therefore, we have just shown in (\ref{equation: Qhat approximate feasibility}), (\ref{equation: Qhat approximate optimality}) that $\hat{Q}_h$  is approximately feasible and approximately optimal in the discretized population problem (\ref{equation: discretized problem for extremes}). Our last step is to relate $\hat{Q}_h$ to the original problem over $f \in \cF$ (\ref{equation: disparity range problem}).
    
From Lemma 1 in \cite{AgarwalEtAl(19)-FairRegression} and (\ref{equation: Qhat approximate feasibility}), we observe that 
        \begin{align*}
            \loss_{\alpha}(\Qhat_h) &\overset{(1)}{\leq} \epsilon + \tilde{O}(n^{-\phi}), \\
            \loss(\Qhat_h) &\overset{(2)}{\leq} \epsilon + \tilde{O}(n^{-\phi}),
        \end{align*}
where (1) used Lemma 1 in \cite{AgarwalEtAl(19)-FairRegression} and we now view $\hat{Q}_h$ as a distribution of risk scores $f \in \cF$, (2) used that $\loss(Q) \leq \loss_\alpha(Q) + \alpha$. Next, from Lemma \ref{lemma: disparity classifier relates to disparity prediction function} and (\ref{equation: Qhat approximate optimality}), we observe that
        \begin{equation*}
            \disp(\hat{Q}_h) \leq \disp(\tilde Q) + \left( |\beta_0| + |\beta_1| \right) \alpha + \tilde{O}(n_0^{-\phi}) + \tilde{O}(n_1^{-\phi}). 
        \end{equation*}
where $\hat{Q}_h$ is viewed as a distribution over risk scores $f \in \cF$ and $\tilde Q$ is now any distribution over risk scores $f \in \cF$ that is feasible in the fairness frontier problem. This proves the result for Case I.
    
\paragraph{Case II: There is no feasible solution to the population problem (\ref{equation: disp range, randomizaton predictor problem})} This follows the proof of Case II in Theorem 3 of \cite{AgarwalEtAl(19)-FairRegression}. If the algorithm returns a $\nu$-approximate saddle point $\hat{Q}_h$, then the theorem holds vacuously since there is no feasible $\tilde Q$. Similarly, if the algorithm returns $null$, then the theorem also holds. $\Box$

\subsection*{Proof of Theorem~\ref{theorem:selective labels extremes}}
Under oracle access to $\mu(x)$, the iteration complexity and bound on cost hold immediately from Theorem~\ref{theorem: error analysis for exp grad extremes}. 
The bound on disparity holds immediately for choices $\cE_{i,0}, \cE_{i,1}$ that depend on only $A$.
For choices of $\cE_{i,0}, \cE_{i,1}$ that depends on $Y_i$, such as the qualified affirmative action fairness-enhancing intervention, we rely on Lemma~\ref{lemma:disparity concentration selective labels}.
We first observe that under oracle access to $\mu(x)$, we can identify any disparity as 
\begin{equation}
\frac{\beta_1 \e [f(X) g(\mu(X)) \mid A = 1]}{\e [ g(\mu(X)) \mid A = 1]} - \frac{\beta_0 \e [f(X) g(\mu(X)) \mid A = 0]}{\e [ g(\mu(X)) \mid A = 0]},
\end{equation}
where $g(x) = x$ for the balance for the positive class and qualified affirmative action criteria; $g(x) = (1-x)$ for balance for the negative class; and $g(x) = 1$ for the statistical parity and the affirmative action criteria (see proof of Lemma~\ref{lemma:disparity concentration selective labels} below proof for an example).
We define the shorthand
\begin{eqnarray*}
& \numa := \e [f(X) g(\mu(X)) \mid A = 1]\\
&\dena := \e [ g(\mu(X)) \mid A = 1] \\
&\numo := \e [f(X) g(\mu(X)) \mid A = 0]\\ 
&\deno := \e [ g(\mu(X)) \mid A = 0] 
\end{eqnarray*}
and we use $\hnuma$, $\hdena$, $\hnumo$, and $\hdeno$ to denote their empirical estimates.
Lemma \ref{lemma:disparity concentration selective labels} gives the following bound on the empirical estimate of the disparity:
\begin{align*}
& \p\Bigg[  \abs{\frac{\beta_1 \hnuma }{\hdena} - \frac{\beta_0 \hnumo}{\hdeno} -  \Big(\frac{\beta_1 \numa}{\dena} - \frac{\beta_0 \numo}{\deno} \Big)} \geq \epsilon \Bigg] \\
 &\leq 4\exp\Bigg[-\frac{n}{2}\Bigg(\frac{\epsilon \bar{\omega}_{\wedge}}{8\beta}  - 4 R_n(\cG) - \frac{2}{\sqrt{n}}\Bigg)^2 \Bigg] + 2\exp \Big[\frac{-n \epsilon^2 \bar{\omega}_{\wedge}^4}{64\beta^2\omega_{\vee}^2} \Big] \\
 &+2\exp\Big[\frac{-n\bar{\omega}_{\wedge}^2}{4} \Big]
\end{align*} 
where $\omega_{\vee} = \max(\numa, \numo)$, $\bar{\omega}_{\wedge} = \min(\dena, \deno)$ and ${\beta = \max(\left| \beta_1 \right|, \left |\beta_0 \right|)}$.

We now proceed to relax and simplify the bound.
For $\epsilon \leq 4 \frac{\beta \omega_{\vee} }{\bar{\omega}_{\wedge}}$, we have
\begin{equation*}
     2\exp \Big[\frac{-n \epsilon^2 \bar{\omega}_{\wedge}^4}{64\beta^2\omega_{\vee}^2} \Big] \geq 2\exp\Big[\frac{-n\bar{\omega}_{\wedge}^2}{4} \Big]
\end{equation*}

\textbf{Case 1: }
We first consider the likely case that $\bar{\omega}_{\wedge} \geq \omega_{\vee}$. Then we have
\begin{equation*}
     2\exp \Big[\frac{-n \epsilon^2 \bar{\omega}_{\wedge}^4}{64\beta^2\omega_{\vee}^2} \Big] \leq 2\exp \Big[\frac{-n \epsilon^2 \bar{\omega}_{\wedge}^2}{64\beta^2} \Big] 
\end{equation*}
\emph{1a)} If 
\begin{equation}
    \frac{\epsilon \bar{\omega}_{\wedge}}{8\beta} \geq 4R_n(\cG) + \frac{2}{\sqrt{n}}
\end{equation} then 
\begin{equation*}
     \exp \Big[\frac{-n \epsilon^2 \bar{\omega}_{\wedge}^2}{64\beta^2} \Big] \leq \exp\Bigg[-\frac{n}{2}\Bigg(\frac{\epsilon \bar{\omega}_{\wedge}}{8\beta}  - 4 R_n(\cG) - \frac{2}{\sqrt{n}}\Bigg)^2 \Bigg] 
\end{equation*}
Then we have
\begin{eqnarray}
&\p\Bigg[  \abs{\frac{\beta_1 \hnuma }{\hdena} - \frac{\beta_0 \hnumo}{\hdeno} -  \Big(\frac{\beta_1 \numa}{\dena} - \frac{\beta_0 \numo}{\deno} \Big)} \geq \epsilon \Bigg] \\
& \leq 8\exp\Bigg[-\frac{n}{2}\Bigg(\frac{\epsilon \bar{\omega}_{\wedge}}{8\beta}  - 4 R_n(\cG) - \frac{2}{\sqrt{n}}\Bigg)^2 \Bigg] 
\end{eqnarray}
Inverting this bound yields the following: with probability at least $1- \delta$, 
\[
\abs{\frac{\beta_1 \hnuma }{\hdena} - \frac{\beta_0 \hnumo}{\hdeno} -  \Big(\frac{\beta_1 \numa}{\dena} - \frac{\beta_0 \numo}{\deno} \Big)} \leq
\]
\[
\frac{8\beta}{\bar{\omega}_{\wedge}}\Bigg( 4 R_n(\cG) + \frac{2}{\sqrt{n}} +\sqrt{\frac{2}{n}\log\Big(\frac{8}{\delta}\Big)}  \Bigg)
\]

\emph{1b)}
\begin{equation}
    \frac{\epsilon \bar{\omega}_{\wedge}}{8\beta} < 4R_n(\cG) + \frac{2}{\sqrt{n}}
\end{equation} implies that 
\[
\abs{\frac{\beta_1 \hnuma }{\hdena} - \frac{\beta_0 \hnumo}{\hdeno} -  \Big(\frac{\beta_1 \numa}{\dena} - \frac{\beta_0 \numo}{\deno} \Big)} \leq
\]
\[
\frac{8\beta}{\bar{\omega}_{\wedge}}\Bigg( 4 R_n(\cG) + \frac{2}{\sqrt{n}} \Bigg).
\]

\textbf{Case 2: } We now consider the unlikely but plausible case that $\bar{\omega}_{\wedge} < \omega_{\vee}$.
Then we have 
\[
\exp\Bigg[-\frac{n}{2}\Bigg(\frac{\epsilon \bar{\omega}_{\wedge}}{8\beta}  - 4 R_n(\cG) - \frac{2}{\sqrt{n}}\Bigg)^2 \Bigg]  \leq
\]
\[
 \exp\Bigg[-\frac{n}{2}\Bigg(\frac{\epsilon \omega_{\vee}}{8\beta}  - 4 R_n(\cG) - \frac{2}{\sqrt{n}}\Bigg)^2 \Bigg] 
\]
and 
\begin{equation*}
\exp \Big[\frac{-n \epsilon^2 \bar{\omega}_{\wedge}^4}{64\beta^2\omega_{\vee}^2} \Big] \leq \exp \Big[\frac{-n \epsilon^2 \omega_{\vee}^2}{64\beta^2} \Big]
\end{equation*}
We proceed with the same steps as in Case 1 to conclude that 
with probability at least $1- \delta$,
\[
\abs{\frac{\beta_1 \hnuma }{\hdena} - \frac{\beta_0 \hnumo}{\hdeno} -  \Big(\frac{\beta_1 \numa}{\dena} - \frac{\beta_0 \numo}{\deno} \Big)} \leq
\]
\[
\frac{8\beta}{\bar{\omega}_{\wedge}}\Bigg( 4 R_n(\cG) + \frac{2}{\sqrt{n}} +\sqrt{\frac{2}{n}\log\Big(\frac{8}{\delta}}\Big)  \Bigg)
\]
Applying our assumption that 
    \[
    R_n(\cH) \leq C n^{-\phi} \mbox{ and } \hat{\epsilon} = \epsilon - \hat{c}_0 + C^\prime n^{-\phi} - C^{\prime \prime} n^{-1/2}.
    \]
for  $\phi \leq 1/2$  and $C' \geq 2C + 2 + \sqrt{2 \ln(8 N/\delta)}$ and $C^{\prime \prime} \geq \sqrt{\frac{-\log(\delta / 8)}{2}}$, then
\begin{equation}
    \disp(\hat{Q}_h) \leq \disp(\tilde{Q}) + \tilde{O}(n^{-\phi}),
\end{equation} which implies 
\begin{equation}
    \disp(\hat{Q}_h) \leq \disp(\tilde{Q}) + \tilde{O}(n_0^{-\phi}) + \tilde{O}(n_1^{-\phi}).
\end{equation}
$\Box$

\subsection*{Proof of Theorem \ref{theorem: error analysis for exp grad, disp minimizing}}

The claim about the iteration complexity of Algorithm \ref{alg: exp grad for fairness frontier, disparity minimizing} follows from Lemma \ref{lemma: iteration complexity for disparity minimizing} after substituting in the stated choices of $\nu, B_\lambda$. We consider two cases.
        
\paragraph{Case 1: There is a feasible solution $\tilde Q$ to the population problem (\ref{equation: absolute disparity min problem, randomized predictor})} Using Lemmas \ref{lemma: empirical slack variable bound for disparity minimization}-\ref{lemma: empirical cost violation, disparity minimizing}, the $\nu$-approximate saddle point $(\hat \xi, \Qhat_h)$ satisfies 
            \begin{align}
                & \disphat(\Qhat_h) - \hat \xi \leq \frac{B_\xi + 2 \nu}{B_\lambda}, \\
                & -\disphat(\Qhat_h) - \hat \xi \leq \frac{B_\xi + 2 \nu}{B_\lambda} \\
                & \costhat(\Qhat_h) - \hat{\epsilon}_{\cost} \leq \frac{ B_\xi + 2 \nu }{ B_\lambda }
            \end{align}
for any $(\xi, Q)$ that is feasible in the empirical problem. This implies that Algorithm \ref{alg: exp grad for fairness frontier, disparity minimizing} returns $\Qhat \neq null$. We will now show that the $(\hat \xi, \Qhat)$ provides an approximate solution to the discretized population problem.
        
First, through the same argument as in the proof of Theorem \ref{theorem: error analysis for exp grad extremes}, we obtain that with probability at least $1 - \delta/4$ for all $Q_h \in \Delta(\cH)$
        \begin{equation}\label{equation: cost bound disp min i}
            \left| \costhat(Q_h) - \cost(Q_h) \right| \leq \tilde{O}(n^{-\phi}).
        \end{equation}
Second, with probability at least $1 - \delta/2$ for all $Q \in \Delta(\cH)$,
        \begin{align}\label{equation: disp bounds, disp min}
            & \left| \hat{\mathbb{E}}[Q_h | \cE_{i,0}] - \mathbb{E}[Q_h | \cE_{i,0}] \right| \leq \tilde{O}(n_0^{-\phi}) \\
            & \left| \hat{\mathbb{E}}[Q_h | \cE_{i,1}] - \mathbb{E}[Q_h | \cE_{i,1}] \right| \leq \tilde{O}(n_1^{-\phi}).
        \end{align} 
Finally, Hoeffding's Inequality implies that with probability at least $1 - \delta/4$,
        \begin{equation}\label{equation: Hoeffding for c0, disp min}
            |\hat{c}_0 - c_0| \leq \sqrt{ \frac{-\log(\delta/8)}{2n}}. 
        \end{equation}
        
From Lemma \ref{lemma: solution quality, disparity minimizing}, we have that Algorithm \ref{alg: exp grad for fairness frontier, disparity minimizing} terminates and delivers $(\hat \xi, \hat Q_h)$ that compares favorable with any feasible $(\xi, Q_h)$ in the discretized sample problem. That is, for any such $(\xi, Q_h)$,
            \begin{align}
                & \hat{\xi} \leq \xi + O(n^{-\phi}), \label{equation:slack xi bound} \\ 
                & \disphat(\Qhat_h) \leq \hat \xi + O(n^{-\phi}), \\
                & -\disphat(\Qhat_h) \leq \hat \xi + O(n^{-\phi})  \label{equation: bound disp min negative} \\ 
                & \costhat(\Qhat_h) \leq \hat{\epsilon}_{\cost} + O(n^{-\phi}) \label{equation: cost bound disp min ii}
            \end{align}
Notice that (\ref{equation: cost bound disp min i}), (\ref{equation: Hoeffding for c0, disp min}) and (\ref{equation: cost bound disp min ii}) imply that 
            \begin{equation}
                \cost(\Qhat_h) \leq \epsilon - c_0 + \tilde{O}(n^{-\phi}),
            \end{equation}
where we used that $\hat{\epsilon} = \epsilon - \hat{c}_0 + C' n^{-\phi} - C'' n^{-\phi}$. For any feasible $(\xi, Q_h)$, then $( \left|\disp(Q_h) \right|, Q_h)$ is also feasible. Then, combining ~(\ref{equation:slack xi bound})-(\ref{equation: bound disp min negative}) yields
\begin{equation}
    \left|\disp(\Qhat_h) \right| \leq \left|\disp(Q_h) \right| + \tilde{O}(n^{-\phi})
\end{equation}

Second, notice that this implies that 
    \begin{equation}
        \left|\disp(\Qhat_h) \right| \leq \left|\disp(Q_h) \right| + \tilde{O}(n_0^{-\phi}) + \tilde{O}(n_1^{-\phi})
    \end{equation}
We assumed that $(\xi, Q_h)$ were feasible in the discretized sample problem. Assuming that (\ref{equation: cost bound disp min i}) holds implies that any feasible solution of the population problem is also feasible in the empirical problem due to how we set $C', C''$. Therefore, we have just shown that $(\hat \xi, \Qhat_h)$ are approximately optimal in the discretized population problem. 
        
Then, following the proof of Theorem \ref{theorem: error analysis for exp grad extremes}, we observe that $\loss(\Qhat_h) \leq \epsilon + \tilde{O}(n^{-\phi})$, where we now interpret $\Qhat_h$ as a distribution over risk scores $f \in \cF$. This proves the result for Case I.
        
\paragraph{Case II: There is no feasible solution to the population problem (\ref{equation: absolute disparity min problem, randomized predictor})} This follows the proof of Case II of Theorem 3 in \cite{AgarwalEtAl(19)-FairRegression}. If the algorithm returns a $\nu$-approximate saddle point $\Qhat_h$, then the theorem holds vacuously since there is no feasible $\tilde Q$. Similarly, if the algorithm returns $null$, then the theorem also holds. $\Box$

\section{Auxiliary Lemmas for Main Results}\label{section: auxiliary lemmas}

In this section, we state and prove a series of auxiliary lemmas that are used in the proofs of our main results in the text.

\subsection{Auxiliary Lemmas for the Proof of Theorem \ref{theorem: error analysis for exp grad extremes}} 

\subsubsection{Iteration Complexity of Algorithm \ref{alg: exp grad for fairness frontier, extremes}} 

\begin{lemma}\label{lemma: iteration complexity for extremes}
    Letting $\rho := \max_{h \in \cH} | \costhat(h) - \hat{\epsilon} |$, Algorithm \ref{alg: exp grad for fairness frontier, extremes} satisfies the inequality 
        \begin{equation*}
            \nu_{t} \leq \frac{B \log(2) }{ \eta t} + \eta \rho^2 B.
        \end{equation*}
    For $\eta = \frac{\nu}{2 \rho^2 B}$, Algorithm \ref{alg: exp grad for fairness frontier, extremes} will return a $\nu$-approximate saddle point of $L$ in at most $ \frac{4 \rho^2 B^2 \log(2)}{\nu^2}$. Since in our setting, $\rho \leq 1$, the iteration complexity of Algorithm \ref{alg: exp grad for fairness frontier, extremes} is $4 B^2 \log(2)/\nu^2$. 
    
    \begin{proof}
        Follows immediately from the proof of iteration complexity in Theorem 3 of \cite{AgarwalEtAl(19)-FairRegression}. Since the cost is bounded on $[-1, 1]$ and $\widehat{\cost}(h) - \hat{\epsilon} \leq \widehat{\cost}(h) \leq 1$ for any $h \in \cH$, we see that $\rho \leq 1$.
    \end{proof}
\end{lemma}

\subsubsection{Solution Quality for Algorithm \ref{alg: exp grad for fairness frontier, extremes}} \label{section: solution quality for algorithm exp grad}

Let $\Lambda := \left\{ \lambda \in \mathbb{R}_{+} \colon \lambda \leq B \right\}$ denote the domain of $\lambda$. Throughout this section, we assume we are given a pair $(\hat{Q}_h, \hat{\lambda})$ that is a $\nu$-approximate saddle point of the Lagrangian
    \begin{align*}
        & L(\hat{Q}_h, \hat{\lambda}) \leq L(Q_h, \hat \lambda) + \nu \mbox{ for all } Q_h \in \Delta(\cH), \\
        & L(\hat{Q}_h, \hat{\lambda}) \geq L(\hat{Q}_h, \lambda) - \nu \mbox{ for all } 0 \leq \lambda \leq B.
    \end{align*}
We extend Lemma 1, Lemma 2 and Lemma 3 of \cite{agarwal2018reductions} to our setting.

\begin{lemma}\label{lemma: lemma 1 of fair classification, extremes}
    The pair $(\hat{Q}_h, \hat{\lambda})$ satisfies
        \begin{equation*}
            \hat{\lambda} \left( \widehat{\cost}(\hat{Q}_h) - \hat{\epsilon} \right) \geq B \left( \widehat{\cost}(\hat{Q}_h) - \hat{\epsilon} \right)_{+} - \nu,
        \end{equation*}
    where $\left(x\right)_{+} = \max\{x, 0\}$.
    
    \begin{proof}
        We consider a dual variable $\lambda$ that is defined as
            \begin{equation*}
                \lambda = \begin{cases}
                    0 \mbox{ if } \widehat{\cost}(\hat{Q}_h) \leq \hat{\epsilon} \\
                    B \mbox{ otherwise}.
                \end{cases}
            \end{equation*}
        From the $\nu$-approximate optimality conditions,
            \begin{align*}
                \widehat{\disp}(\hat{Q}) + \hat{\lambda} \left( \widehat{\cost}(\hat{Q}) - \hat{\epsilon} \right) &= L(\hat{Q}, \hat{\lambda}) \\
                &\geq L(\hat{Q}, \lambda) - \nu \\
                &= \widehat{\disp}(\hat{Q}) + \lambda \left( \costhat(Q) - \hat{\epsilon} \right),
            \end{align*}
        and the claim follows by our choice of $\lambda$. 
    \end{proof}
\end{lemma}

\begin{lemma}\label{lemma: lemma 2 of fair classification, extremes}
    The distribution $\hat{Q}_h$ satisfies 
    $$\widehat{\disp}(\hat{Q}_h) \leq \widehat{\disp}(Q_h) + 2 \nu$$ 
    for any $Q_h$ satisfying the empirical constraint (i.e., any $Q_h$ such that $\widehat{\cost}(Q_h) \leq \hat{\epsilon}$).
    
    \begin{proof}
        Assume $Q_h$ satisfies $\widehat{\cost}(Q_h) \leq \hat{\epsilon}$. Since $\hat{\lambda} \geq 0$, we have that 
            \[
                L(Q_h, \hat \lambda) = \widehat{\disp}(Q_h) + \hat{\lambda} \left( \widehat{\cost}(Q_h) - \hat{\epsilon} \right) \leq \widehat{\disp}(Q_h).
            \]
        Moreover, the $\nu$-approximate optimality conditions imply that $L(\hat{Q}_h, \hat{\lambda}) \leq L(Q_h, \hat{\lambda}) + \nu$. Together, these inequalities imply that 
            \[
                L(\hat{Q}_h, \hat{\lambda}) \leq \widehat{\disp}(Q_h) + \nu.
            \]
        Next, we use Lemma \ref{lemma: lemma 1 of fair classification, extremes} to construct a lower bound for $L(\hat{Q}_h, \hat{\lambda})$. We have that 
            \begin{align*}
                L(\hat{Q}_h, \hat{\lambda}) &= \disphat(\Qhat_h) + \hat{\lambda} \left( \costhat(Q_h) - \hat{\epsilon}^\prime \right) \\
                &\geq \disphat(\hat{Q}_h) + B \left( \costhat(\Qhat) - \hat{\epsilon}^\prime \right)_{+} - \nu \\
                &\geq \disphat(\hat{Q}_h) - \nu.
            \end{align*}
        By combining the inequalities $L(\hat{Q}_h, \hat{\lambda}) \geq \disphat(\hat{Q}_h) - \nu$ and $L(\hat{Q}_h, \hat{\lambda}) \leq \disphat(Q_h) + \nu$, we arrive at the claim.
    \end{proof}
\end{lemma}

\begin{lemma}\label{lemma: lemma 3 of fair classification, extremes}
    Assume the empirical constraint $\costhat(Q_h) \leq \hat{\epsilon}$ is feasible. Then, the distribution $\hat{Q}_h$ approximately satisfies the empirical cost constraint with
        \begin{equation*}
            \costhat(\hat{Q}_h) - \hat{\epsilon} \leq \frac{|\beta_0| + |\beta_1| + 2 \nu}{B}.
        \end{equation*}
     
    \begin{proof}
        Let $Q_h$ satisfy $\costhat(Q_h) \leq \hat{\epsilon}$. Recall from the proof of Lemma \ref{lemma: lemma 2 of fair classification, extremes}, we showed that
        \[
        \disphat(\hat{Q}_h) + B \left( \costhat(\Qhat_h) - \hat{\epsilon} \right)_{+} - \nu \leq L(\hat{Q}_h, \hat{\lambda}) \leq
        \]
        \[
        \disphat(Q_h) + \nu.
        \]
        Therefore, we observe that 
            \begin{equation*}
                B \left( \costhat(Q_h) - \hat{\epsilon} \right) \leq \left( \disphat(Q_h) - \disphat(\hat{Q}_h) \right) + 2\nu.
            \end{equation*}
        Since we can bound $\disphat(Q_h) - \disphat(\hat{Q}_h)$ by $|\beta_0| + |\beta_1|$, the result follows.
    \end{proof}
\end{lemma}

\begin{lemma}\label{lemma: solution quality of exp grad for extremes}
    Suppose that $Q_h$ is any feasible solution to discretized sample problem. Then, the solution $\hat{Q}_h$ returned by Algorithm \ref{alg: exp grad for fairness frontier, extremes} satisfies 
        \begin{align*}
            & \disphat(\hat{Q}_h) \leq \disphat(Q_h) + 2 \nu \\
            & \costhat(\hat{Q}_h) \leq \hat{\epsilon} + \frac{|\beta_0| + |\beta_1| + 2\nu}{B}.
        \end{align*}
        
    \begin{proof}
        This is an immediate consequence of Lemma \ref{lemma: iteration complexity for extremes}, Lemma \ref{lemma: lemma 2 of fair classification, extremes} and Lemma \ref{lemma: lemma 3 of fair classification, extremes}. If the algorithm returns $null$, then these inequalities are vacuously satisfied. 
    \end{proof}
\end{lemma}

\subsubsection{Concentration Inequality}
 We restate Lemma 2 in \cite{AgarwalEtAl(19)-FairRegression}, which provides a uniform concentration inequality on the convergence of a sample moment over a function class. 
 
 Let $\cG$ be a class of functions $g \colon \cU \rightarrow \mathbb{R}$ over some space $\cU$. The \textit{Rademacher complexity} of the function class $\cG$ is defined as 
    \[
    R_n(\cG) := \sup_{u_1, \hdots, u_n \in \cU} \mathbb{E}_{\sigma}\left[ \sup_{g \in \cG} \left| \frac{1}{n} \sum_{i=1}^{n} \sigma_i g(u_i) \right| \right],
    \]
where the expectation is defined over the i.i.d. random variables $\sigma_1, \hdots, \sigma_n$ with $P(\sigma_i = 1) = P(\sigma_i = -1) = 1/2$. 

\begin{lemma}[Lemma 2 in \cite{AgarwalEtAl(19)-FairRegression}]\label{lemma: lemma 2 in fair reg paper}
    Let $D$ be a distribution over a pair of random variables $(S, U)$ taking values in $\cS \times \cU$. Let $\cG$ be a class of functions $g \colon \cU \rightarrow [0, 1]$, and let $\psi \colon \cS \times [0, 1] \rightarrow [-1, 1]$ be a contraction in its second argument (i.e., for all $s \in \cS$ and $t, t' \in [0, 1]$, $|\psi(s, t) - \psi(s, t')| \leq |t - t'|$). Then, with probability $1 - \delta$, for all $g \in \cG$,
    \[
    \left| \hat{\mathbb{E}}\left[ \psi(S, g(U)) \right] - \mathbb{E}\left[\psi(S, g(U)) \right]   \right| \leq
    \]
    \[
    4 R_n(\cG) + \frac{2}{\sqrt{n}} + \sqrt{ \frac{2 \ln(2/\delta)}{n}  },
    \]
    where the expectation is with respect to $D$ and the empirical expectation is based on $n$ i.i.d. draws from $D$. If $\psi$ is linear in its second argument, then a tighter bound holds with $4 R_n(\cG)$ replaced by $2 R_n(\cG)$. 
\end{lemma}

\subsection{Auxiliary Lemmas for the Proof of Theorem \ref{theorem:selective labels extremes}}
\subsubsection{Concentration result for disparity under selective labels}
\begin{lemma}
\label{lemma:disparity concentration selective labels}
\begin{align*}
& \p\Bigg[  \abs{\frac{\beta_1 \hnuma }{\hdena} - \frac{\beta_0 \hnumo}{\hdeno} -  \Big(\frac{\beta_1 \numa}{\dena} - \frac{\beta_0 \numo}{\deno} \Big)} \geq \epsilon \Bigg] \\
 &\leq 4\exp\Bigg[-\frac{n}{2}\Bigg(\frac{\epsilon \bar{\omega}_{\wedge}}{8\beta}  - 4 R_n(\cG) - \frac{2}{\sqrt{n}}\Bigg)^2 \Bigg] + 2\exp \Big[\frac{-n \epsilon^2 \bar{\omega}_{\wedge}^4}{64\beta^2\omega_{\vee}^2} \Big] \\
 &+2\exp\Big[\frac{-n\bar{\omega}_{\wedge}^2}{4} \Big]
\end{align*} 
where $\omega_{\vee} = \max(\numa, \numo)$, $\bar{\omega}_{\wedge} = \min(\dena, \deno)$ and ${\beta = \max(\left| \beta_1 \right|, \left |\beta_0 \right|)}$
\end{lemma}

\begin{proof}
For exposition, we first show the steps for qualified affirmative action and then extend the result to the general disparity.
We can rewrite the qualified affirmative action criterion as
\begin{equation}
    \mathbb{E}[f(X) | Y = 1, A = 1] = \frac{\mathbb{E}[f(X) \mu(X) | A = 1]}{\mathbb{E}[ \mu(X) | A = 1]}
\end{equation} 
where  $\mu(x) := \e[Y \mid X = x]$.

$   \mathbb{E}[f(X) | Y = 1, A = 1] $
\begin{eqnarray}
  & =  \frac{\mathbb{E}[f(X) 1\{ Y = 1\} \mid A = 1]}{P(Y = 1 \mid A = 1)} \\
      & =  \frac{\mathbb{E}[f(X) \mathbb{E} [ 1\{ Y = 1\} \mid X, A = 1] \mid A = 1]}{E[P(Y = 1 \mid X, A = 1) \mid A =1]} \\
        & =  \frac{\mathbb{E}[f(X) P(Y = 1 \mid X, A = 1) \mid A = 1]}{E[\mu(X) \mid A =1]} \\
        & =  \frac{\mathbb{E}[f(X) \mu(X) \mid A = 1]}{E[\mu(X) \mid A =1]}
\end{eqnarray}

Assuming access to the oracle $\mu$ function, we can estimate this on the full training data as
\begin{equation}
     \frac{\mathbb{\hat E}[f(X) \mu(X, A = 1) | A = 1]}{\mathbb{\hat E}[ \mu(X, A = 1) | A = 1]}
\end{equation}

Next we will make use of Lemma 2 of \cite{AgarwalEtAl(19)-FairRegression}, which we restate here again for convenience. Under certain conditions on $\phi$ and $g$, with probability at least $1- \delta$
\[
\left| \hat{\mathbb{E}}\left[ \phi(S, g(U)) \right] - \mathbb{E}\left[ \phi(S, g(U)) \right]   \right| \leq 
\]
\[
4 R_n(\cG) + \frac{2}{\sqrt{n}} + \sqrt{ \frac{2 \ln(2/\delta)}{n}  }.
\]

We invert the bound by setting $\epsilon  = 4 R_n(\cG) + \frac{2}{\sqrt{n}} + \sqrt{ \frac{2 \ln(2/\delta)}{n}} $ and solving for $\delta$ to get
\begin{equation}
\delta = 
    2\exp\Bigg[-\frac{n}{2}\Bigg(\epsilon  - 4 R_n(\cG) - \frac{2}{\sqrt{n}}\Bigg)^2 \Bigg]
\end{equation}
     
Now we can restate Lemma 2 of \cite{AgarwalEtAl(19)-FairRegression} as 
\begin{equation} \label{eq:lemma2_restated}
   \mathbb{P}\Big[ \left| \hat{\mathbb{E}}\left[ \phi(S, g(U)) \right] - \mathbb{E}\left[ \phi(S, g(U)) \right]   \right| > \epsilon \Big]
\end{equation}
\[
\leq 2\exp\Bigg[-\frac{n}{2}\Bigg(\epsilon  - 4 R_n(\cG) - \frac{2}{\sqrt{n}}\Bigg)^2 \Bigg]
\]

Next we revisit the quantity that we want to bound:
\begin{equation} \label{eq:target}
    \abs{\frac{\num}{\den} - \frac{\hnum}{\hden}}
\end{equation}

where $\num = \mathbb{E}[f(X) \mu(X, A = 1) | A = 1]$ and $\den = \mathbb{E}[ \mu(X, A = 1) | A = 1]$ and correspondingly for $\hnum$ and $\hden$.
We will rewrite Expression~\ref{eq:target} as a ratio of differences.
We have 
\begin{eqnarray}
   \abs{\frac{\hnum}{\hden} - \frac{\num}{\den}} &= \abs{\frac{\hnum \den - \hden \num}{\hden \den}} \\ 
   &= \abs{\frac{\den(\hnum - \num)  - \num(\hden - \den)}{\den(\hden -\den) + \den^2}} \\ 
\end{eqnarray}
By triangle inequality and union bound, we have 
\[
\mathbb{P}\Big[|\frac{\den(\hnum - \num)  - \num(\hden - \den)}{\den(\hden - \den) + \den^2} | \geq \frac{t}{\den^2/2} \Big] 
\]
\[
< \mathbb{P}\Big[|\den(\hnum - \num)|  + |\num(\hden - \den)| \geq t\Big] + \mathbb{P}\Big[|(\hden - \den) + \den^2| \leq \frac{\den^2}{2} \Big]
\]
\[
<  \mathbb{P}\big[|\den(\hnum - \num)| \geq \frac{t}{2}\big]  + \mathbb{P}\big[|\num(\hden - \den)| \geq \frac{t}{2}\big] + \mathbb{P}\Big[|\den(\hden -\den) + \den^2| \leq \frac{\den^2}{2} \Big] 
\]

Since $0 \leq \mu(X, A = 1) \leq 1$, we can use a Hoeffding bound for the quantity $|(\hden -\den)|$.
Note that $0 \leq  \num \leq \den \leq 1$.
Then applying Hoeffding's inequality gives us
\begin{eqnarray}
   \mathbb{P}\big[\abs{\num(\hden - \den)} \geq \frac{t}{2}\big] \leq 2 \exp \Big[\frac{-n t^2}{4\num^2} \Big]
\end{eqnarray}

Next we bound the third term: 
\begin{eqnarray}
   \mathbb{P}\Big[|\den(\hden -\den) + \den^2| \leq \frac{\den^2}{2} \Big] &\leq 
   \mathbb{P}\Big[|\den(\hden -\den)| \geq \frac{\den^2}{2} \Big] \\
   &=   \mathbb{P}\Big[|(\hden -\den)| \geq \frac{\den}{2} \Big] \\
   & \leq 2 \exp\Big[\frac{-n\den^2}{4} \Big]
\end{eqnarray}

where we again used Hoeffding's inequality for the last line.

We bound the first term using the restated Lemma in~\ref{eq:lemma2_restated}:
\begin{equation}
     \mathbb{P}\big[|\den(\hnum - \num)| \geq \frac{t}{2}\big] \leq 2\exp\Bigg[-\frac{n}{2}\Bigg(\frac{t}{2\den}  - 4 R_n(\cG) - \frac{2}{\sqrt{n}}\Bigg)^2 \Bigg]
\end{equation}

Now we let $\tilde \epsilon = \frac{t}{\den^2/2}$ to get
\begin{equation} \label{equation:omega general bound}
   \p \Big[  \abs{\frac{\hnum}{\hden} - \frac{\num}{\den}}  \geq \tilde \epsilon \Big] 
\end{equation}
\[
\leq 2\exp\Bigg[-\frac{n}{2}\Bigg(\frac{\tilde \epsilon \den}{4}  - 4 R_n(\cG) - \frac{2}{\sqrt{n}}\Bigg)^2 \Bigg] + \exp \Big[\frac{-n \tilde \epsilon^2 \den^4}{16\num^2} \Big] +\exp\Big[\frac{-n\den^2}{4} \Big]
\]

Now we turn to the general case. 
Recalling that we define $\beta = \max(\abs{\beta_1, \beta_0})$, we have
\[
\p\Bigg[  \abs{\frac{\beta_1 \hnuma }{\hdena} - \frac{\beta_0 \hnumo}{\hdeno} -  \Big(\frac{\beta_1 \numa}{\dena} - \frac{\beta_0 \numo}{\deno} \Big)} \geq \epsilon \Bigg] \leq
\]
\[
\p\Bigg[  \abs{\beta_1}\abs{\frac{ \hnuma }{\hdena}  -  \frac{ \numa}{\dena}} + \abs{\beta_0}\abs{\frac{\hnumo}{\hdeno} - \frac{\numo}{\deno}} \geq \epsilon \Bigg] \leq
\]
\[
\p\Bigg[ \abs{\frac{ \hnuma }{\hdena}  -  \frac{ \numa}{\dena}}  \geq \frac{\epsilon}{2\beta} \Bigg] + \p\Bigg[ \abs{\frac{\hnumo}{\hdeno} - \frac{\numo}{\deno}} \geq \frac{\epsilon}{2\beta} \Bigg] \leq
\]
\[
2\exp\Bigg[-\frac{n}{2}\Bigg(\frac{\epsilon \dena}{8\beta}  - 4 R_n(\cG) - \frac{2}{\sqrt{n}}\Bigg)^2 \Bigg] + \exp \Big[\frac{-n \epsilon^2 \dena^4}{64\beta\numa^2} \Big] +\exp\Big[\frac{-n\dena^2}{4} \Big] +
\]
\[
2\exp\Bigg[-\frac{n}{2}\Bigg(\frac{\epsilon \deno}{8\beta}  - 4 R_n(\cG) - \frac{2}{\sqrt{n}}\Bigg)^2 \Bigg] + 
\]
\[
\exp \Big[\frac{-n \epsilon^2 \deno^4}{64\beta\numo^2} \Big] +\exp\Big[\frac{-n\deno^2}{4} \Big] \leq
\]
\[
4\exp\Bigg[-\frac{n}{2}\Bigg(\frac{\epsilon \bar{\omega}_{\wedge}}{8\beta}  - 4 R_n(\cG) - \frac{2}{\sqrt{n}}\Bigg)^2 \Bigg] + 2\exp \Big[\frac{-n \epsilon^2 \bar{\omega}_{\wedge}^4}{64\beta^2\omega_{\vee}^2} \Big]
\]
\[
+2\exp\Big[\frac{-n\bar{\omega}_{\wedge}^2}{4} \Big]
\]
where the first inequality holds by triangle inequality, the second inequality holds by the union bound, the third inequality applies~(\ref{equation:omega general bound}) for $\tilde{\epsilon} = \frac{\epsilon}{2\beta}$, and the final inequality simplifies the bound using the notation $\omega_{\vee} = \max(\numa, \numo)$ and $\bar{\omega}_{\wedge} = \min(\dena, \deno)$.
\end{proof}

\subsection{Auxiliary Lemmas for the Proof of Theorem \ref{theorem: error analysis for exp grad, disp minimizing} }

\subsubsection{Iteration Complexity for Algorithm \ref{alg: exp grad for fairness frontier, disparity minimizing}}

\begin{lemma}\label{lemma: iteration complexity for disparity minimizing}
    Defining $\rho := \max_{h \in \cH, \xi \in [0, B_\xi]} \max\{\disphat(h) - \xi, -\disphat(h) - \xi, \costhat(h) - \hat{\epsilon}\}$, Algorithm \ref{alg: exp grad for fairness frontier, disparity minimizing} satisfies the inequality
        \[
            \nu_t \leq \frac{ B_\lambda \log(3) }{ \eta t } + \eta \rho^2 B.
        \]
    For $\eta = \frac{\nu}{2 \rho^2 B_\lambda}$, Algorithm \ref{alg: exp grad for fairness frontier, disparity minimizing} will return a $\nu$-approximate saddle point of $L$ in at most $\frac{ 4 \rho^2 B_\lambda^2 \log(3) }{\nu^2}$ iterations. Setting $B_{\xi} = 1$, we observe $\rho \leq 1$, and so the iteration complexity of Algorithm \ref{alg: exp grad for fairness frontier, disparity minimizing} is $\frac{ 4 B_\lambda^2 \log(3) }{\nu^2}$.
    
    \begin{proof}
     Follows immediately from the proof of Theorem 3 in \cite{AgarwalEtAl(19)-FairRegression} and the same argument given in the proof of Lemma \ref{lemma: iteration complexity for extremes}.
    \end{proof}
\end{lemma}

\subsubsection{Solution Quality for Algorithm \ref{alg: exp grad for fairness frontier, disparity minimizing}}

Let $\Lambda = \{\lambda \in \mathbb{R}_{+}^{3} : \|\lambda\| \leq B_\lambda\}$. Assume we are given $\left( \hat{\xi}, \Qhat_h, \hat{\lambda} \right)$, which is a $\nu$-approximate saddle point satisfying $L(\hat{\xi}, \Qhat_h, \hat{\lambda}) \leq L(\xi, Q_h, \hat{\lambda}) + \nu$ for all $Q_h \in \Delta(\cH), \xi \in [0, B_\xi]$ and $L(\hat{\xi}, \Qhat_h, \hat{\lambda}) \geq L(\hat{\xi}, \Qhat_h, \lambda) - \nu$ for all $\|\lambda\| \leq B_\lambda$. We extend Lemmas \ref{lemma: lemma 1 of fair classification, extremes}-\ref{lemma: lemma 3 of fair classification, extremes} to the problem of finding the absolute disparity minimizing model. 

\begin{lemma}\label{lemma: complementary slackness for disparity minimization} $\left( \hat{\xi}, \Qhat_h, \hat{\lambda} \right)$ satisfies
    \[
    \hat{\lambda}_{+} ( \disphat(\Qhat_h) - \hat{\xi} ) + \hat{\lambda}_{-} ( -\disphat(\Qhat_h) -\hat{\xi}) + \hat{\lambda}_{cost} \left( \costhat(\Qhat_h) - \hat{\epsilon} \right)
    \]
    \[
    \geq B_{\lambda} \max\{ \disphat(\Qhat_h) - \hat{\xi}, -\disphat(\Qhat_h) -\hat{\xi}, \costhat(\Qhat_h) - \hat{\epsilon}\} - \nu.
    \]
    \begin{proof}
        The argument is the same as the proof of Lemma \ref{lemma: lemma 1 of fair classification, extremes}.
    \end{proof}
\end{lemma}

\begin{lemma}\label{lemma: empirical slack variable bound for disparity minimization}
    The value $\hat{\xi}$ satisfies 
        \begin{equation*}
            \hat{\xi} \leq \xi + 2 \nu
        \end{equation*}
    for any $\xi$ such that there exists $Q_h$ satisfying $\disphat(Q_h) - \xi \leq 0$, $-\disphat(Q_h) - \xi \leq 0$ and $\costhat(Q_h) \leq \hat{\epsilon}$.
    
    \begin{proof}
        Assume the pair $(\xi, Q_h)$ satisfies $\disphat(Q_h) - \xi \leq 0$, $-\disphat(Q_h) - \xi \leq 0$ and $\costhat(Q_h) \leq \hat{\epsilon}$. Since $\hat{\lambda} \geq 0$, we have that $L(\xi, Q, \hat{\lambda}) \leq \xi$. Moreover, the $\nu$-approximate optimality conditions imply that $L(\hat \xi, \Qhat, \hat \lambda) \leq L(\xi, Q, \hat \lambda) + \nu$. Together, these inequalities imply that 
            \[
            L(\hat \xi, \Qhat, \hat \lambda) \leq \xi + \nu.
            \]
        Next, we can use Lemma \ref{lemma: complementary slackness for disparity minimization} to construct a lower bound for $L(\hat \xi, \Qhat, \hat \lambda)$. To do so, observe that 
            \begin{align*}
                & L(\hat \xi, \Qhat, \hat \lambda) \\
                &\geq \hat{\xi} + B_{\lambda} \max\{ \disphat(\Qhat) - \hat{\xi}, -\disphat(\Qhat) -\hat{\xi}, \costhat(\Qhat) - \hat{\epsilon}\} - \nu \\
                &\geq \hat{\xi} - \nu.
            \end{align*}
        By combining the inequalities, $L(\hat \xi, \Qhat, \hat \lambda) \geq \hat{\xi} - \nu$ and $L(\hat \xi, \Qhat, \hat \lambda) \leq \xi + \nu$, we arrive at the claim.
    \end{proof}
\end{lemma}

\begin{lemma}\label{lemma: empirical fairness violations, disparity minimizing} 
    Assume the empirical cost constraint $\costhat{Q_h} \leq \epsilonhat$ and the slack variable constraints $\disphat(Q_h) - \xi \leq 0$ and $-\disphat(Q_h) - \xi \leq 0$ are feasible. Then, the pair $(\hat \xi, \Qhat_h)$ satisfies
        \begin{align*}
            & \disphat(\Qhat_h) - \hat \xi \leq \frac{B_\xi + 2 \nu}{B_\lambda}, \\
            & -\disphat(\Qhat_h) - \hat \xi \leq \frac{B_\xi + 2 \nu}{B_\lambda}.
        \end{align*}

    \begin{proof}
        Let $\xi$ be a feasible value of the slack variable such that there exists $Q_h$ satisfying $\costhat(Q_h) \leq \hat{\epsilon}$ and the slack variable constraints $\disphat(Q_h) - \xi \leq 0$, $-\disphat(Q_h) - \xi \leq 0$. Recall from the Proof of Lemma \ref{lemma: empirical slack variable bound for disparity minimization}, we showed that 
            \begin{align*}
                & \hat{\xi} + B_{\lambda} \max\{ \disphat(\Qhat_h) - \hat{\xi}, -\disphat(\Qhat_h) -\hat{\xi}, \costhat(\Qhat_h) - \hat{\epsilon}\} - \nu \\
                &\leq L(\hat \xi, \Qhat_h, \hat \lambda) \leq \xi + \nu.
            \end{align*}
        Therefore, it is immediate that
            \[
                 B_{\lambda} \max\{ \disphat(\Qhat_h) - \hat{\xi}, -\disphat(\Qhat_h) -\hat{\xi}, \costhat(\Qhat_h) - \hat{\epsilon}\} \leq \left( \xi - \hat \xi \right) + 2 \nu,
            \]
         and so
            \begin{align*}
               & B_{\lambda} \left( \disphat(\Qhat_h) - \hat{\xi} \right) \leq \left( \xi - \hat \xi \right) + 2 \nu, \\
               & B_{\lambda} \left( -\disphat(\Qhat_h) -\hat{\xi} \right) \leq  \left( \xi - \hat \xi \right) + 2 \nu.
            \end{align*}
        Since $\xi \in [0, B_\xi]$, we can bound $\xi - \hat \xi$ by $B_\xi$. The result follows.
    \end{proof}
\end{lemma}

\begin{lemma}\label{lemma: empirical cost violation, disparity minimizing}
    Assume the empirical cost constraint $\costhat(Q_h) \leq \hat{\epsilon}$ and the slack variable constraints $\disphat(Q_h) - \xi \leq 0$, $-\disphat(Q_h) - \xi \leq 0$ are feasible. Then the distribution $\Qhat_h$ satisfies 
        \[
            \costhat(\Qhat_h) - \hat{\epsilon} \leq \frac{ B_\xi + 2 \nu }{ B_\lambda }.
        \]
    \begin{proof}
         The proof is analogous to the proof of Lemma \ref{lemma: empirical fairness violations, disparity minimizing}.
    \end{proof}
\end{lemma}

\begin{lemma}\label{lemma: solution quality, disparity minimizing}
    Suppose that $(\xi, Q_h)$ is a feasible solution to the empirical version of (\ref{equation: disparity minimizing, randomized, slack variable}). Then, the solution $(\hat \xi, \Qhat_h)$ returned by Algorithm \ref{alg: exp grad for fairness frontier, disparity minimizing} satisfies
        \begin{align*}
            & \hat{\xi} \leq \xi + 2 \nu, \\
            & \disphat(\Qhat_h) - \hat \xi \leq \frac{B_\xi + 2 \nu}{B_\lambda}, \\
            & -\disphat(\Qhat_h) - \hat \xi \leq \frac{B_\xi + 2 \nu}{B_\lambda} \\
            & \costhat(\Qhat_h) - \hat{\epsilon} \leq \frac{ B_\xi + 2 \nu }{ B_\lambda }.
        \end{align*}
\end{lemma}
\begin{proof}
The proof follows from Lemmas \ref{lemma: empirical fairness violations, disparity minimizing}-\ref{lemma: empirical cost violation, disparity minimizing}. If the algorithm returns $null$, then these inequalities are vacuously satisfied. 
\end{proof}

\section{Additional Details on the Consumer Lending Data}\label{section: additional details on lending data}

\subsection{Construction of IRSD for SA4 Regions}\label{section: additional details on the IRSD}
As discussed in \textsection~\ref{section: consumer lending application}, we focus our analysis on predictive disparities across SA4 geographic regions within Australia. We use the Australian Bureau of Statistics' Index of Relative Socioeconomic Disadvantage (IRSD) to define socioeconomically disadvantaged SA4 regions. The IRSD is calculated for SA2 regions, which are more granular statistical areas used by the ABS, by aggregating sixteen variables that were collected in the 2016 Australian census. These variables include, for example, the fraction of households making less than AU\$26,000, the fraction of households with no internet access, and the fraction of residents who do not speak English well. Higher scores on the IRSD are associated with less socioeconomically disadvantaged regions, and conversely, lower scores on the IRSD are associated with more socioeconomically disadvantaged regions. The full list of variables that are included in the IRSD and complete details on how the IRSD is constructed is provided in \cite{SEIFA(16)}.

\begin{figure}[htbp!]
    \center
    \includegraphics[scale = 0.17]{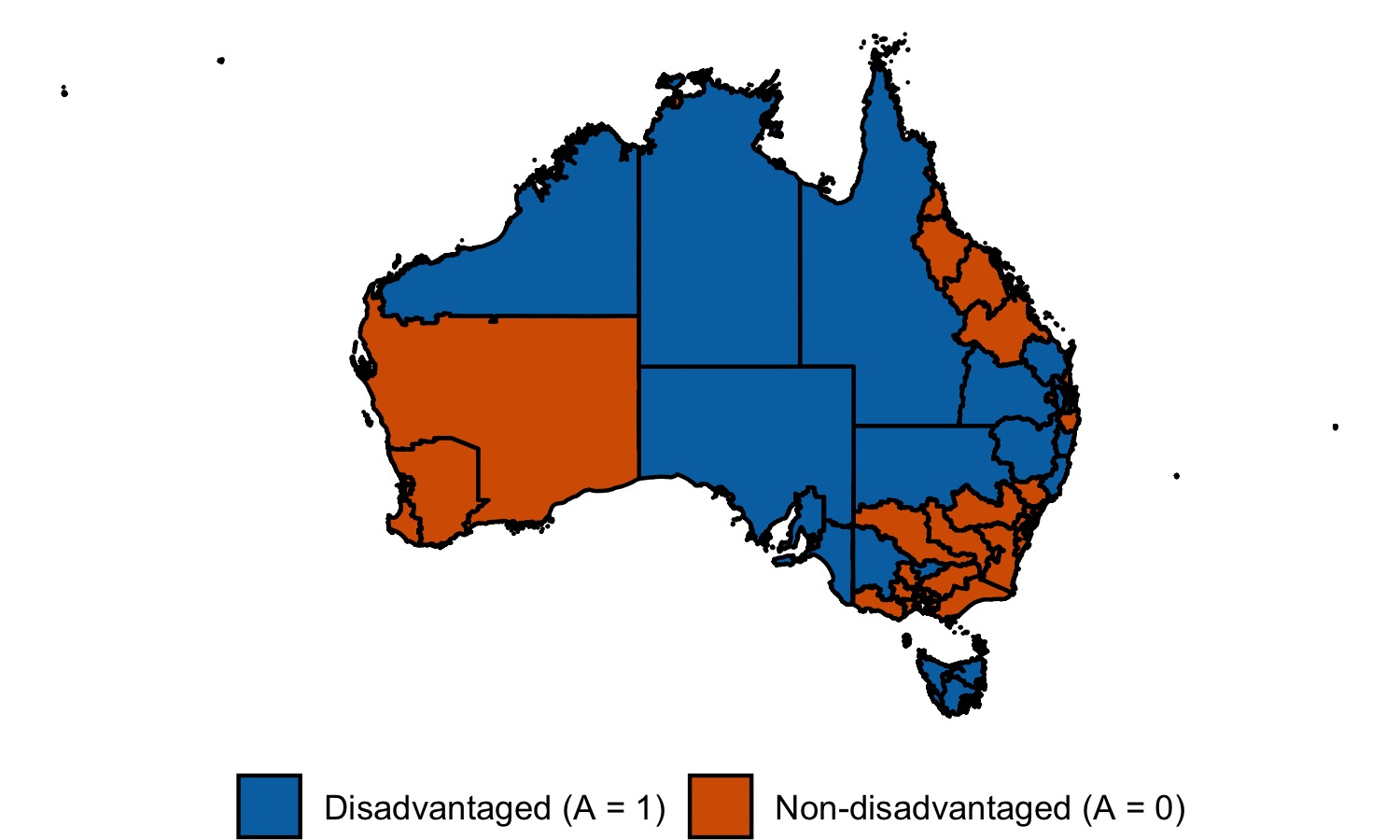}
    \caption{SA4 regions in Australia. We classify SA4 regions as being "socioeconomically disadvantaged" (red) and "non-socioeconomically disadvantaged" (blue) based on the Index of Relative Socioeconomic Disadvantage (IRSD). }
    \label{fig: map of SA4 and IRSD}
\end{figure}

Because the IRSD is constructed for SA2 regions, we first aggregate this index to SA4 regions. We construct an aggregated IRSD for each SA4 region by constructing a population-weighted average of the IRSD for all SA2 regions that fall within each SA4 region. This delivers a quantitative measure of which SA4 regions are the most and least socioeconomically disadvantaged. For example, the bottom ventile (i.e., the 20th ventile) of SA4 regions based upon the population-weighted average IRSD (i.e., the least socioeconomically disadvantaged SA4 regions) are regions associated with Sydney and Perth. The top ventile (i.e., the 1st ventile) of SA4 regions based upon the population-weighted average IRSD (i.e., the most socioeconomically disadvantaged SA4 regions) are regions associated with the Australian outback such as the Northern territory outback and the Southern Australia outback. Figure \ref{fig: map of SA4 and IRSD} provides a map of SA4 regions in Australia, in which colors SA4 regions classified as socioeconomically disadvantaged in blue.

\section{Additional Experimental Details and Results}\label{section: additional experiments}

In this section, we present additional details on our experimental setup as well as additional results for both experiments presented in the main paper.

\subsection{Consumer Lending: Additional Experimental Details} \label{section: experimental details} 

We performed experiments on a random $2\%$ sample of over 360,000 loan applications submitted from July 2017 to July 2019 by customers who did not have a prior financial relationship with CommBank, yielding our experimental sample of 7414 applications. We did a 2:1 train-test split, resulting in 4906 applications in our training set and 2508 applications in our test set. 

In order to evaluate our methods on the full population (including applications that are not funded), we generate synthetic funding decisions $D_i$ and outcomes $\tilde Y_i^*$ from the observed application features. On a $20\%$ sample of the full 360,000 applicants, we train a classifier $\pi(x) $ to predict the observed funding decision $D_i$ using the application features $X_i$, and we train a classifier $\mu(x)$ on funded applicants to predict the observed default outcome $Y_i$ using the application features $X_i$. In other words, $\pi(x)$ estimates $P(D_i = 1 | X_i = x)$ and $\mu(x)$ estimates $P(Y_i = 1 | D_i = 1, X_i = x)$. For both models we use probability forests from the R package {\tt ranger} with the default hyperparameters: $500$ trees, $mtry = \sqrt{[dim(X)]} = 6$, min node size equal $10$, and max depth equal to $0$. To learn $\mu$, we use bootstrap sampling of the $(0,1)$ classes with probabilities $(0.01, 1)$, respectively, in order to down-sample the applicants who repaid the loans because we have significant class imbalance: Only $2.0\%$ of applicants have default outcomes $=1$.

We generate synthetic funding decisions $\tilde D_i$ according to $\widetilde D_i \mid X_i \sim Bernoulli(\pi(X_i))$ and synthetic default outcomes $\widetilde Y_i^*$ according to $\widetilde Y_i^* \mid X_i \sim Bernoulli(\mu(X_i))$. 
We then proceed with our learning as if we only had access to labels $\widetilde Y_i^*$ for applicants with $\widetilde D_i = 1$. We estimate $\hat \mu(x) := \hat{P}(\widetilde Y_i = 1 | X_i = x, \widetilde D_i =1)$ using random forests with the same hyperparameters as above and use $\hat \mu(x)$ to construct the pseudo-outcomes used by the IE and RIE approaches. The KGB, IE, and RIE approaches use linear regression. Our \fairs algorithm ran the exponentiated gradient algorithm for at most 500 iterations on a fixed discretization grid, $\mathcal{Z}_\alpha = \{1/40, 2/40,\hdots, 1\}$ with parameters $B = \sqrt{n}$ and $\nu = 1/\sqrt{n}$ and $\eta = 2$. These choices were guided by our theoretical results as well as prior work \cite{AgarwalEtAl(19)-FairRegression}. The average runtime for a single error tolerance $\epsilon$ was $26.4$ minutes.
The experiments were conducted on a machine with one Intel Xeon E5-2650 v2 processor with 2.60 GHz and 16 cores. 

Our comparison against prior work used the fairlearn\footnote{See \href{https://fairlearn.github.io/v0.5.0/api_reference/fairlearn.reductions.html}{Fairlearn Github} for code.} API with logistic regression, using parameters parameters $C=10$ and maximum iterations $=10,000$. We ran fairlearn using both grid search and exponentiated gradient algorithm, but we report only the grid search algorithm since it traced out a larger fairness-performance tradeoff curve than the exponentiated gradient algorithm. We used a grid size of $41$ with a grid limit of $2$. 

We also compared against the Target-Fair Covariate Shift method in \cite{coston2019fair}. To construct our covariate shift weights, we first estimated the propensity scores  $P(D = 1 \mid X =x)$ by regressing  $D \sim X$ , yielding propensity estimates $\hat \pi(x)$. Our propensity model used  {\tt ranger} probability forests that with the default hyperparameters: $500$ trees, $mtry = \sqrt{[dim(X)]} = 6$, min node size equal $10$, and max depth equal to $0$. We used $\max(\hat \pi(X), 50)$ as covariate shift weights. We ran the method for $\lambda = \{0, 10, 1000, 20000, 50000\}$ using step size $\eta = 0.01$. We terminated the algorithm when the L1 distance in the weight vector $\leq 1e-7$ or after $500$ iterations (whichever came first). 

\subsection{Consumer Lending Risk Scores: Additional Results}\label{section: consumer lending additional results}
Figure~\ref{fig:selective_disparities extended} provides an extended version of Figure~\ref{fig:selective_disparities} that reports models over a range of hyperparameters (e.g. loss tolerance for \fairs) to show the range of possible fairness-performance combinations.

\begin{figure}[htbp!]
    \centering
    \includegraphics[scale=0.25]{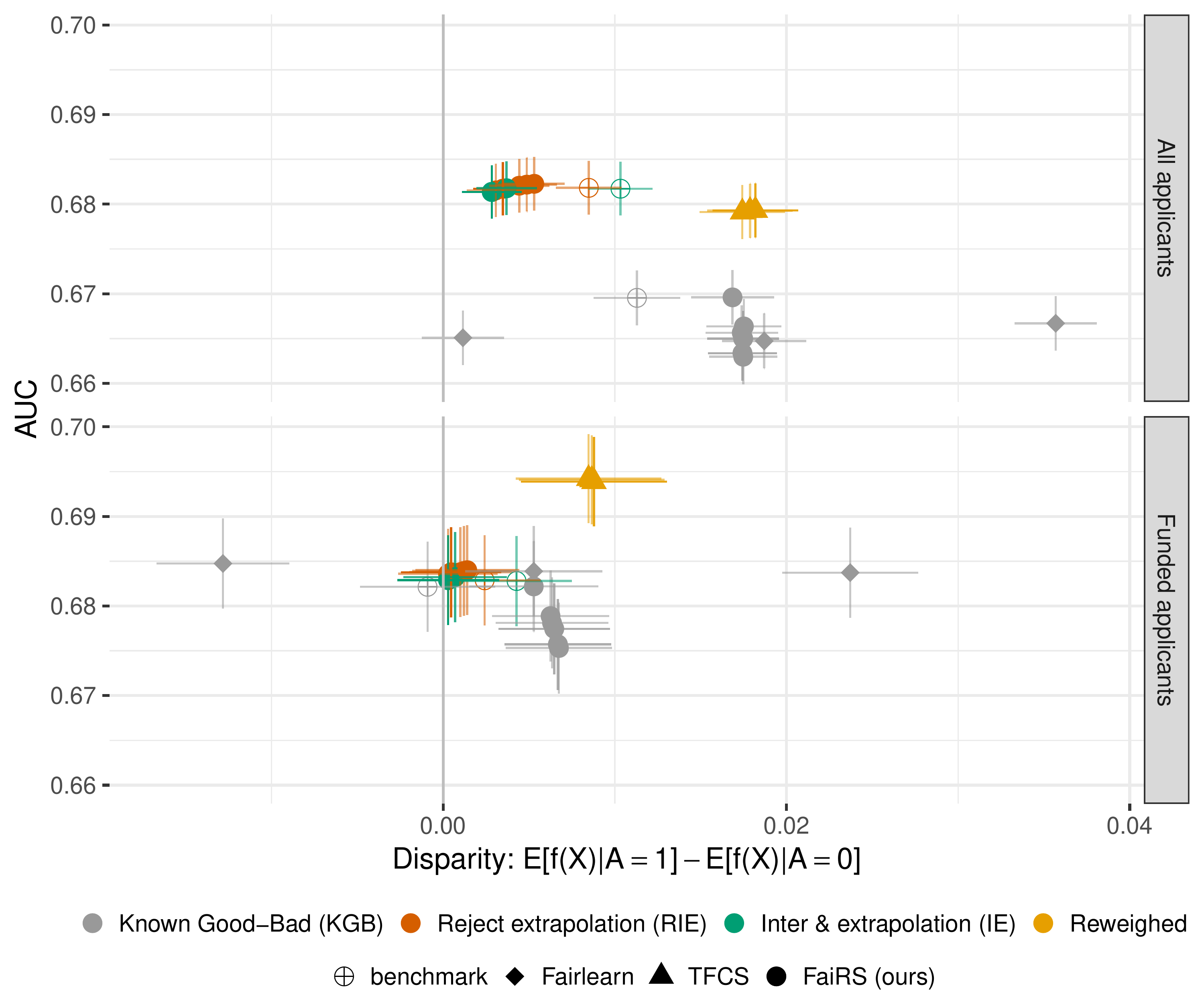}
    \caption{Area under the ROC curve (AUC) with respect to the synthetic outcome against disparity in the average risk prediction for the disadvantaged ($A_i=1$) vs advantaged ($A_i=0$) groups.
    \fairs reduces disparities for the RIE and IE approaches while maintaining AUCs comparable to the benchmark models (first row). 
    Evaluation on only funded applicants (second row) overestimates the performance of TFCS and KGB models and underestimates disparities for all models.
    Error bars show the $95\%$ confidence intervals. See \textsection~\ref{section: consumer lending application} of the main paper for details.}
    \label{fig:selective_disparities extended}
\end{figure}

We next consider the implications of \fairs for the credit applicants from the sensitive group. One might hope that encouraging statistical parity will increase access to credit for the sensitive group, and indeed we see some evidence for this.
Figure~\ref{fig:sel_scores} shows the distribution of risk scores for the disadvantaged group for the KGB, RIE, and IE models for the benchmark models (first row) and for \fairs with $1\%$ loss tolerance (second row). The $75\%$ percentile score is given as a dashed line. 
\fairs shifts the $75\%$ percentile KGB score to the left (left column). \fairs therefore reduces the predicted risk of the sensitive group, thereby expanding access to credit. We see a smaller shift for the RIE and IE approaches, which have lower risk distributions than the KGB model.
As we have seen elsewhere, evaluation on the funded only applicants (right column) lends misleading conclusions, e.g., underestimating both the difference in distributions between the KGB and RIE/IE approaches as well as differences between the benchmark and \fairs variants. 

Finally, we present results for the benchmarks and \fairs models with respect to the loss they were trained to minimize, mean-squared error. 
Figure \ref{fig:lending application disparity vs mse on test} shows the mean square error (MSE) against predictive disparity for the KGB, RIE, IE benchmarks and \fairs variants on held-out test data. The qualitative patterns are the same as Figure \ref{fig:selective_disparities} in \textsection~\ref{section: consumer lending application} of the main text. Evaluation on all applicants shows that \fairs with reject extrapolation (RIE and IE) reduces disparities without impacting MSE. The RIE and IE methods achieve lower disparity and lower MSE than the KGB model trained only on funded data,  highlighting the importance of adjusting for selective labels. We again observe that evaluation on only funded applications is misleading as it suggests that the KGB models have comparable MSE and it drastically underestimates predictive disparities for all models. 

Figure \ref{fig:selective_disparities} in \textsection~\ref{section: consumer lending application} of the main text and Figure \ref{fig:lending application disparity vs mse on test} shows that the \fairs KGB model appears to produce larger predictive disparities than the benchmark KGB model. This is likely due to generalization error on the held-out test data. To verify this hypothesis, Figure \ref{fig:lending application disparity vs mse on train} shows the MSE against predictive disparity for the KGB, RIE, IE benchmarks and \fairs variants on the training data. Indeed among funded applicants in the train data, \fairs-KGB models produce smaller absolute predictive disparities than the benchmark KGB model (second row).

\begin{figure}[htbp!]
    \centering
    \includegraphics[scale = 0.25]{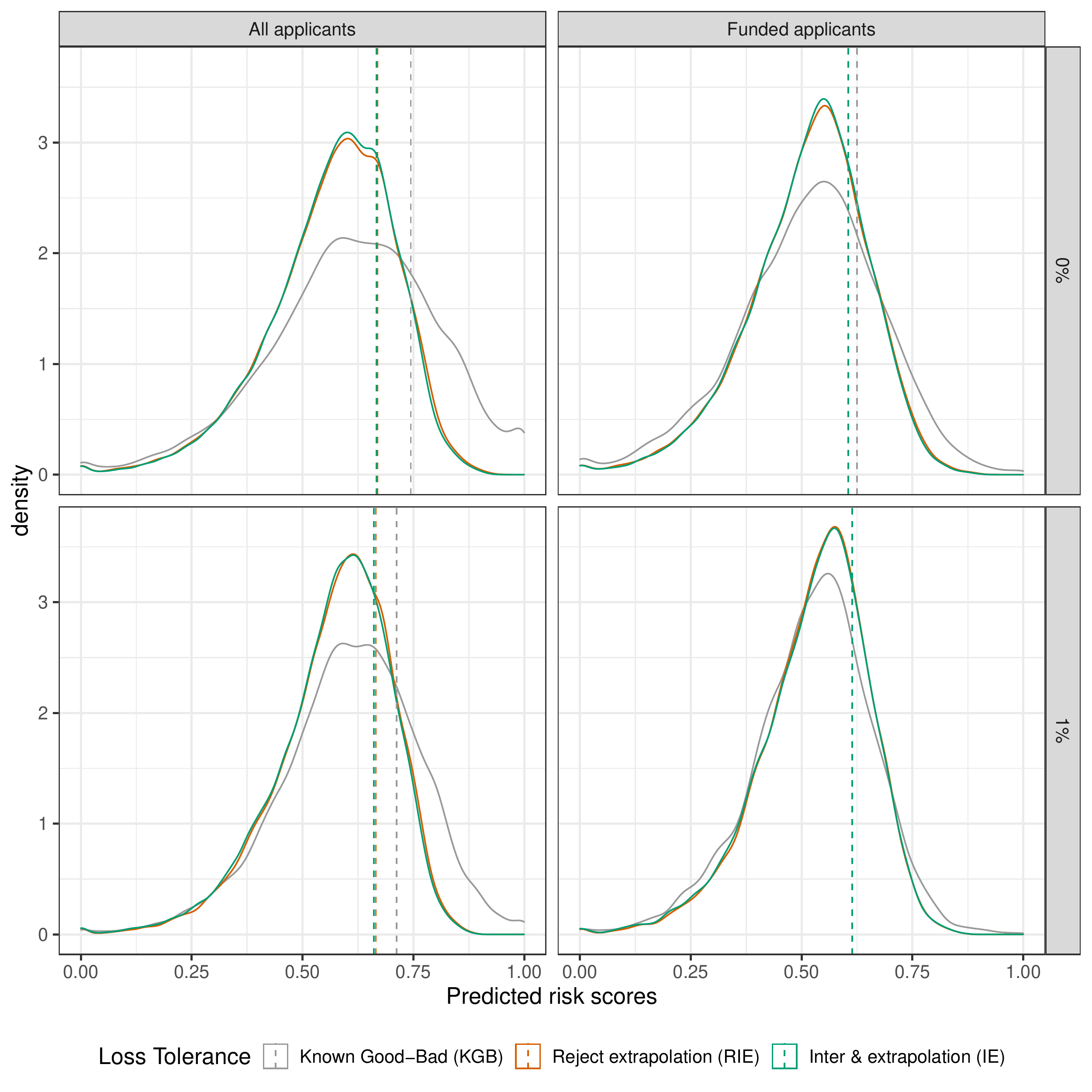}
    \caption{Predicted risk distributions for disadvantaged group $A_i = 1$ for \fairs algorithm using KGB, RIE and IE approaches. The first row shows the benchmark model risk scores. The second row shows our FaiRS's risk scores for a loss tolerance of $1\%$. The left and right columns show risk scores on all applicants and funded applicants from the disadvantaged group respectively. The dashed line indicates the $75$-percentile score. The RIE and IE methods predict lower rates of default for the disadvantaged group than the KGB method. The densities for the funded applicants (right column) underestimate the differences in risk scores across the KGB, RIE, and IE methods (compare to left column). See \textsection~\ref{section: consumer lending application} for details.}
    \label{fig:sel_scores}
\end{figure}

\begin{figure}[htbp!]
    \centering
    \includegraphics[scale=0.25]{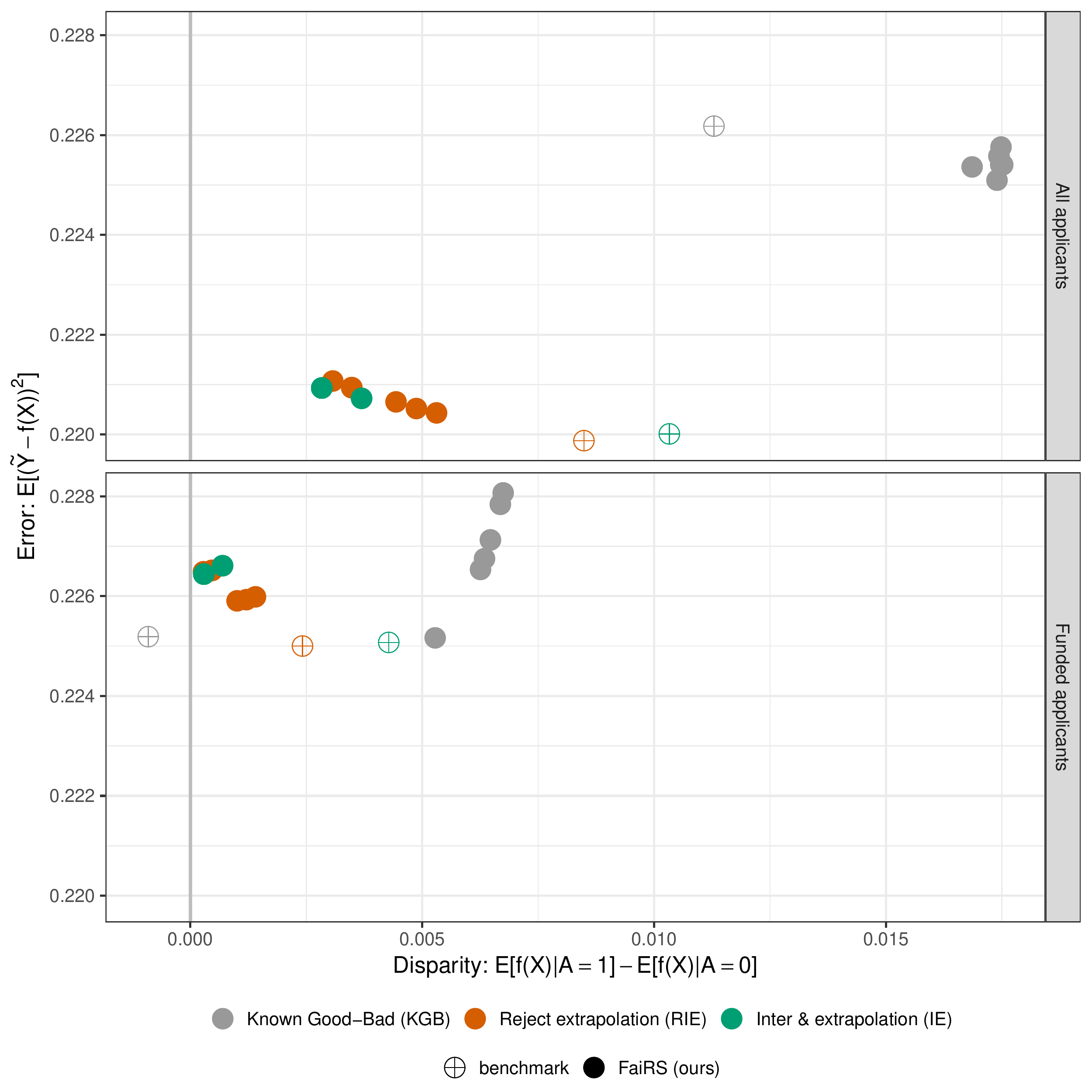}
    \caption{Mean square error (MSE) with respect to the synthetic outcome $\tilde Y_i$ against disparity in the average risk prediction for the disadvantaged ($A_i = 1$) vs. advantaged ($A_i = 0$) groups in held-out test data. The first row evaluates each method on all applicants and and the second row evaluates each method on funded  applicants only. See \textsection~\ref{section: consumer lending application} and \textsection~\ref{section: consumer lending additional results} for details.
    }
    \label{fig:lending application disparity vs mse on test}
\end{figure}

\begin{figure}[htbp!]
    \centering
    \includegraphics[scale=0.25]{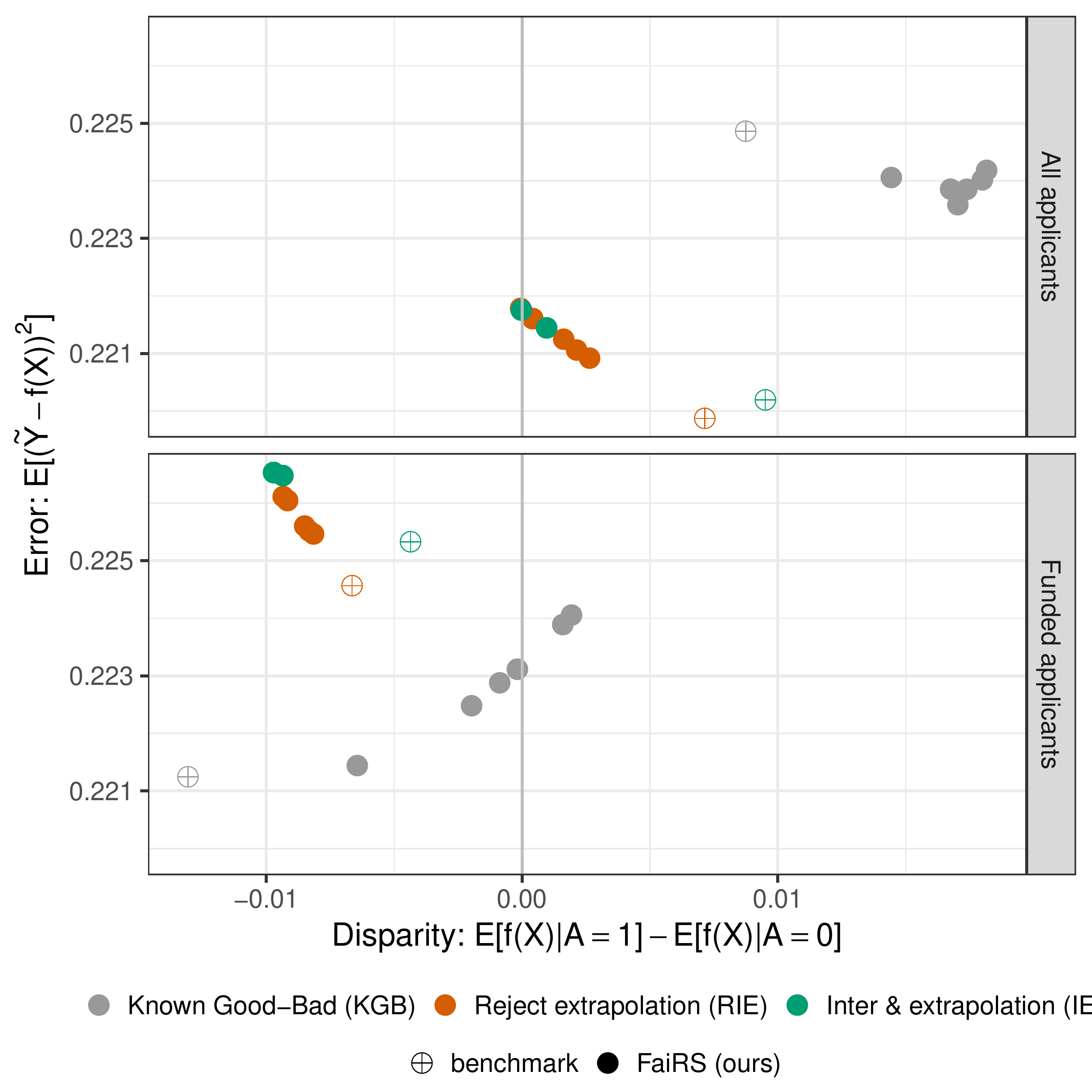}
    \caption{Mean square error (MSE) with respect to the synthetic outcome $\tilde Y_i$ against disparity in the average risk prediction for the disadvantaged ($A_i = 1$) vs. advantaged ($A_i = 0$) groups in the training data. The first row evaluates each method on all applicants and and the second row evaluates each method on funded  applicants only. See \textsection~\ref{section: consumer lending application} and \textsection~\ref{section: consumer lending additional results} for details.
    }
    \label{fig:lending application disparity vs mse on train}
\end{figure}

\subsection{Recidivism Risk Prediction: Additional Results}\label{section: recidivism additional results}

ProPublica's COMPAS recidivism data \cite{AngwinEtAl(16)} contains 7,214 examples. We randomly split this data 50\%-50\% into a train and test set. We evaluate models using logistic regression loss, defined as $l(y, f(x)) = \log(1 + e^{-C(2y-1)(2 f(x) - 1)})/(\log(1 + e^C))$ for $C = 5$. We ran the exponentiated gradient algorithm for at most 500 iterations on a fixed discretization grid, $\mathcal{Z}_\alpha = \{1/40, 2/40,\hdots, 1\}$. Letting $n = 3,607$, we set the parameters of the exponentiated gradient algorithm to be $B = \sqrt{n}/2$ for minimization problems, $B = \sqrt{n}$ for maximization problems, $\nu = 1/\sqrt{n}$ and $\eta = 2$. We report the average run time results for a single run of the exponentiated gradient algorithm to solve the minimization and maximization problems for each disparity measure in Table \ref{table: Compas Race - Traing Timing} below. These experiments were conducted on a 2012 MacBook Pro with a 2.3 GHz Quad-Core Intel Core i7.

\begin{table}[t]
\caption{Timing for the recidivism risk prediction experiment on the ProPublic COMPAS dataset. We report the average time for the exponentiated gradient algorithm to complete at most $500$ iterations on the train set ($n_{train} = 3,607$) in computing the disparity minimizing model (Min. Disp.) and the disparity maximizing model (Max. Disp.). Timing is reported in minutes. See \textsection~\ref{section: recidivism application} for details.}
\label{table: Compas Race - Traing Timing}
\vskip 0.15in
\begin{center}
\begin{small}
\begin{sc}
\begin{tabular}{c c c c }
\toprule
& & \multicolumn{2}{c}{Timing (in minutes)} \\
& & Min. Disp.  & Max. Disp. \\
\midrule
SP & & 7.29 & 24.10 \\
BFPC & & 8.45 & 24.18 \\
BFNC & & 22.24 & 23.64 \\
\bottomrule
\end{tabular}
\end{sc}
\end{small}
\end{center}
\vskip -0.1in
\end{table}

\subsubsection{Test Loss}\label{section: race test loss performance}

Table \ref{table: Compas Race - ref model compas Test, loss} reports the test loss of COMPAS and the test losses of the disparity minimizing and disparity maximizing models over the set of good models. The disparity minimizing and disparity maximizing models achieve comparable and in some cases lower test loss than COMPAS.

\begin{table}[t]
\caption{The disparity minimizing and disparity maximizing models over the set of good models (performing within $1\%$ of COMPAS's training loss) achieve comparable test loss to COMPAS. The first panel (SP) displays the test loss for the models that minimize (Min. Disp.) and maximize (Max. Disp.) the disparity in average predictions for black versus white defendants (Def. \ref{definition: Statistical parity}). The second panel (BFPC) analyzes the test loss for the models that minimize and maximize the disparity in average predictions for black versus white defendants in the positive class, and the third panel examines the test loss for the models that minimize and maximize the disparity in average predictions for black versus white defendants in the negative class (Def. \ref{definition: balance for positive and negative class}). Standard errors are reported in parentheses. See \textsection~\ref{section: recidivism application} for details.}
\label{table: Compas Race - ref model compas Test, loss}
\vskip 0.05in
\begin{center}
\begin{small}
\begin{sc}
\begin{tabular}{c cccr}
\toprule
& & Test loss & & \\
& Min. Disp. & Max. Disp.  & COMPAS \\
\midrule
SP & 0.095 & 0.067 & 0.102 \\
& (0.001) & (0.002) & (0.003) \\
\midrule 
BFPC & 0.099 & 0.085 & 0.102 \\
& (0.003) & (0.002) & (0.003) \\
\midrule
BFNC & 0.094 & 0.073 & 0.102 \\
& (0.004) & (0.001) & (0.003) \\
\bottomrule
\end{tabular}
\end{sc}
\end{small}
\end{center}
\vskip -0.1in
\end{table}

\subsubsection{Train Set Performance}\label{section: race train set performance}
Figure \ref{fig: Compas Race - ref model compas Train Plots} plots the range of predictive disparities over the train set when the parameter $\epsilon$ is calibrated using COMPAS. We report the train set performance for various choices of the loss tolerance parameter, setting $\epsilon = 1\%, 5\%, 10\%$ of COMPAS' training loss.  The blue error bars plot the relative disparities associated with the linear program reduction (\textsection~\ref{section: linear program reduction}), the green error bars plot the relative disparities associated with the stochastic prediction function returned by Algorithm \ref{alg: exp grad for fairness frontier, extremes} and the orange dashed line plots the relative disparity associated with COMPAS. The range of disparities produced by the linear program reduction closely track the range of disparities produced by the stochastic prediction function returned by Algorithm \ref{alg: exp grad for fairness frontier, extremes} in the train set, confirming the quality of the linear programming reduction.

\begin{figure}[htbp!]
     \centering
     \begin{subfigure}[b]{0.4\textwidth}
         \centering
         \includegraphics[scale=0.11]{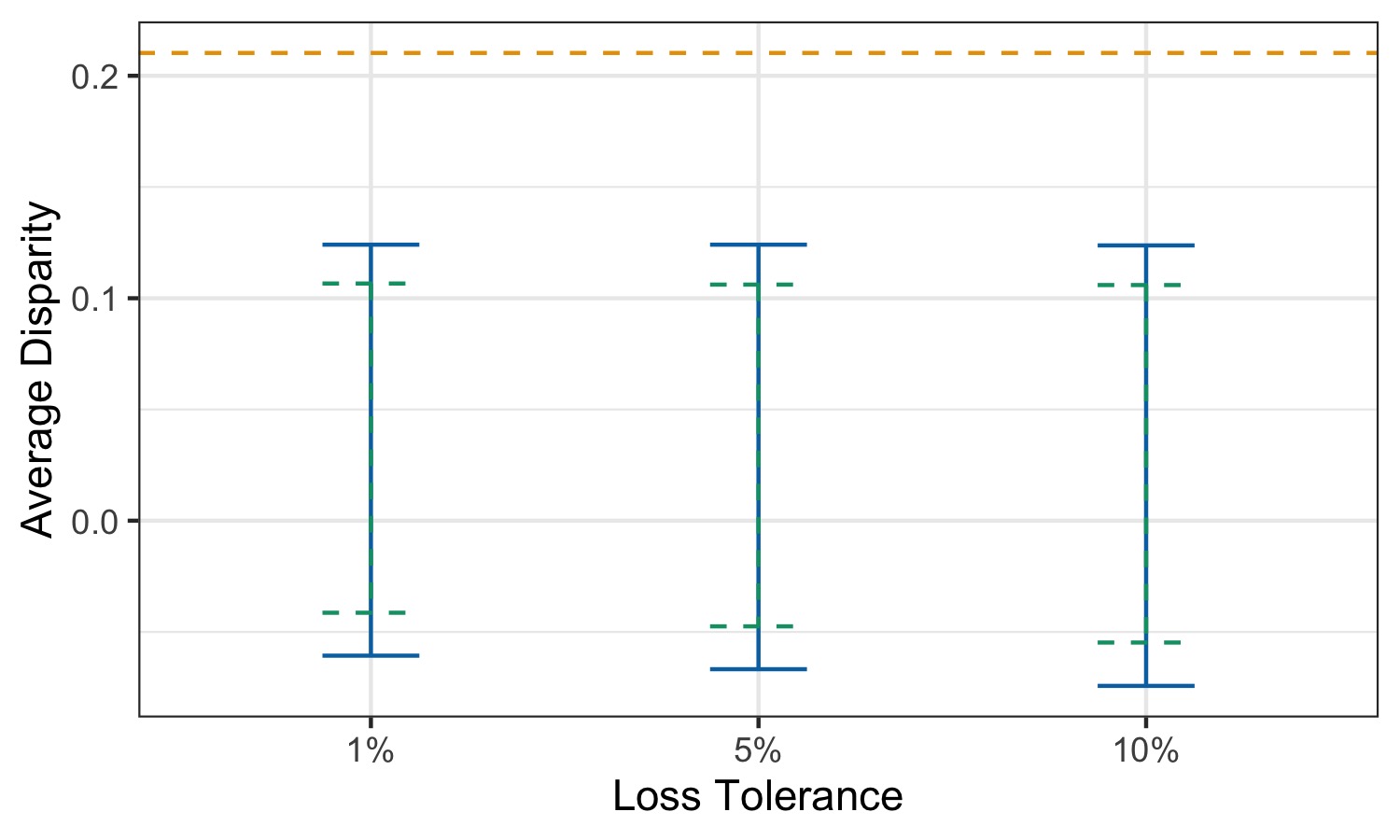}
         \caption{Statistical Parity}
         \label{fig: Compas Race - ref model compas Train SP}
     \end{subfigure}
     \hfill
     \begin{subfigure}[b]{0.4\textwidth}
         \centering
         \includegraphics[scale=0.11]{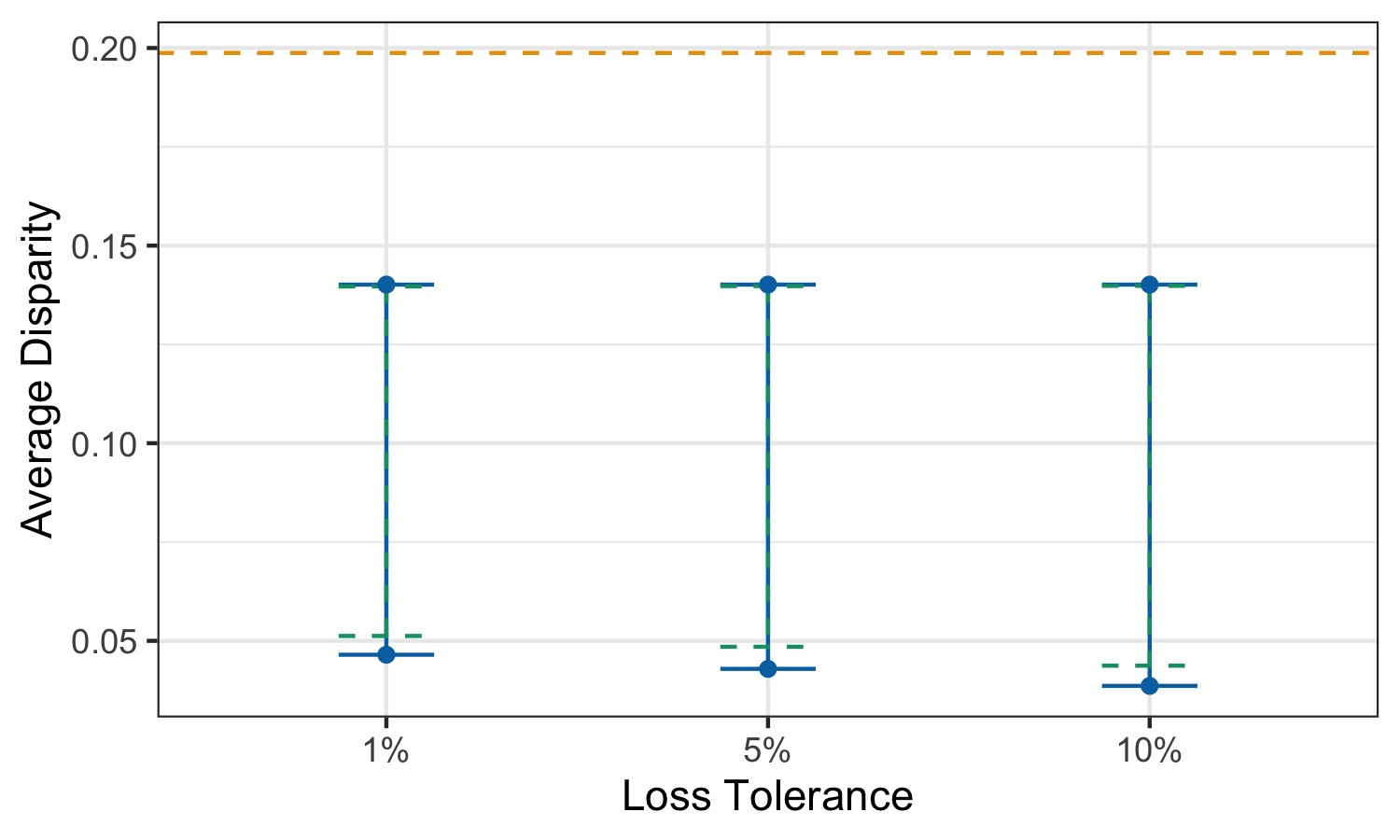}
         \caption{Balance for the Positive Class}
         \label{fig: Compas Race - ref model compas Train BFPC}
     \end{subfigure}
     \hfill
     \begin{subfigure}[b]{0.4\textwidth}
         \centering
         \includegraphics[scale=0.11]{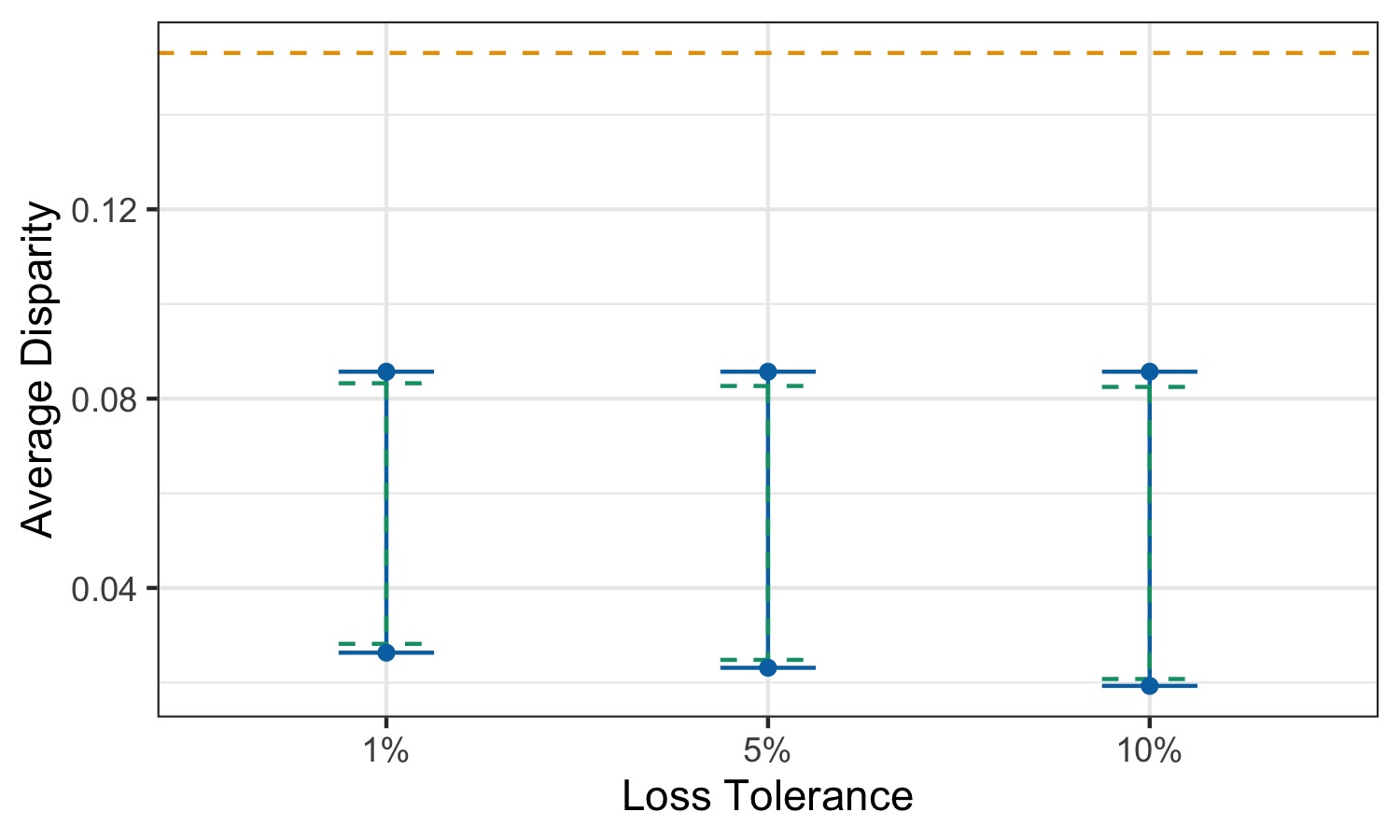}
         \caption{Balance for the Negative Class}
         \label{fig: Compas Race - ref model compas Train BFNC}
     \end{subfigure}
    \caption{The minimal and maximal predictive disparities between black defendants ($A_i = 1$) and white defendants ($A_i = 0$) over the set of good models in the train set. We set the loss tolerance as $\epsilon = 1\%, 5\%, 10\%$ of COMPAS' training loss. The blue error bars plot the relative disparities associated with the linear program reduction (\textsection~\ref{section: linear program reduction}), the green error bars plot the relative disparities associated with the stochastic prediction function returned by Algorithm \ref{alg: exp grad for fairness frontier, extremes} and the orange dashed line plots the predictive disparity associated with COMPAS. See \textsection~\ref{section: recidivism application} and \textsection~\ref{section: race train set performance} for details.}
    \label{fig: Compas Race - ref model compas Train Plots}
\end{figure}

\subsubsection{Results for Predictive Disparities across Young and Older Defendants}\label{section: COMPAS age experiments}

We also examine the range of predictive disparities between defendants that are younger than 25 years old ($A_i = 1$) and defendants older than 25 years old ($A_i = 0$), focusing on the range of predictive disparities that could be generated by a risk score that is constructed using logistic regression on a quadratic polynomial of the defendant's age and number of prior offenses. We calibrate the loss tolerance parameter $\epsilon$ such that (\ref{equation: disparity range problem}) constructs the fairness frontier over all models that achieve a logistic regression loss within $1\%$ of COMPAS's training loss. We provide the results for the statistical parity, balance for the positive class, and balance for the negative class disparity measures (Def. \ref{definition: Statistical parity} and Def \ref{definition: balance for positive and negative class}).  Table \ref{table: Compas age ref model compas test} summarizes the range of predictive disparities over the test set when the parameter $\epsilon$ is calibrated using COMPAS' training loss. While COMPAS lies within the range of possible disparities for each measure, notice that there exists a predictive model that produces strictly smaller disparities between young and older defendants than the COMPAS risk assessment at minimal cost to predictive performance. The disparity minimizing and disparity maximizing models over the set of good models achieve a test loss that is comparable to COMPAS (see Table \ref{table: Compas Age - ref model compas Test loss}).

Figure \ref{fig: Compas age ref model compas train} plots the range of predictive disparities over the train set when the parameter $\epsilon$ is calibrated using COMPAS. We report the train set performance for various choices of the loss tolerance parameter, setting $\epsilon = 1\%, 5\%, 10\%$ of COMPAS' training loss. The blue error bars plot the relative disparities associated with the linear program reduction (Section \ref{section: linear program reduction}), the green error bars plot the relative disparities associated with the stochastic prediction function returned by Algorithm \ref{alg: exp grad for fairness frontier, extremes} and the orange dashed line plots the relative disparity associated with COMPAS. We again find that the range of disparities produced by the linear program reduction closely track the range of disparities produced by the stochastic prediction function returned by Algorithm \ref{alg: exp grad for fairness frontier, extremes}.

\begin{table}[htbp!]
\caption{The minimal and maximal disparities between young defendants ($A_i = 1$) and older defendants ($A_i = 0$) over the set of good models (performing within 1\% of COMPAS' training loss) on the test set. The first panel (SP) displays the disparity in average predictions for young versus older defendants (Def. \ref{definition: Statistical parity}). The second panel (BFPC) displaces the disparity in average predictions for young versus old defendants in the positive class, and the third panel examines the disparity in average predictions for young versus older defendants in the negative class (Def. \ref{definition: balance for positive and negative class}). Standard errors are reported in parentheses. See \textsection~\ref{section: COMPAS age experiments} of the Supplement for details.}
\label{table: Compas age ref model compas test}
\vskip 0.15in
\begin{center}
\begin{small}
\begin{sc}
\begin{tabular}{ccccr}
\toprule
& Min. Disp. & Max. Disp. & COMPAS \\
\midrule
SP & -0.296 & 0.433 & 0.173 \\ 
& (0.019) & (0.008) & (0.014) \\
\midrule 
BFPC & -0.207 & 0.260 & 0.101 \\ 
& (0.010) & (0.008) & (0.019) \\
\midrule
BFNC & -0.040 & 0.329 & 0.200 \\
& (0.038) & (0.008) & (0.022) \\
\bottomrule
\end{tabular}
\end{sc}
\end{small}
\end{center}
\vskip -0.1in
\end{table}

\begin{table}[t]
\caption{The disparity minimizing and disparity maximizing models over the set of good models (performing within $1\%$ of COMPAS's training loss) achieve comparable test loss to COMPAS. The first panel (SP) displays the test loss for the models that minimize (Min. Disp.) and maximize (Max. Disp.) the disparity in average predictions for young versus older defendants (Def. \ref{definition: Statistical parity}). The second panel (BFPC) analyzes the test loss for the models that minimize and maximize the disparity in average predictions for young versus older defendants in the positive class, and the third panel examines the test loss for the models that minimize and maximize the disparity in average predictions for young versus older defendants in the negative class (Def. \ref{definition: balance for positive and negative class}). Standard errors are reported in parentheses. See \textsection~\ref{section: recidivism application} for details.}
\label{table: Compas Age - ref model compas Test loss}
\vskip 0.05in
\begin{center}
\begin{small}
\begin{sc}
\begin{tabular}{ccccr}
\toprule
& & Test loss & & \\
& Min. Disp. & Max. Disp.  & COMPAS \\
\midrule
SP & 0.096 & 0.097 & 0.102 \\
& (0.004) & (0.003) & (0.003) \\
\midrule 
BFPC & 0.098 & 0.098 & 0.102 \\
& (0.002) & (0.003) & (0.003) \\
\midrule
BFNC & 0.094 & 0.093 & 0.102 \\
& (0.016) & (0.002) & (0.003) \\
\bottomrule
\end{tabular}
\end{sc}
\end{small}
\end{center}
\vskip -0.1in
\end{table}

\begin{figure}
     \centering
     \begin{subfigure}[b]{0.4\textwidth}
         \centering
         \includegraphics[scale=0.11]{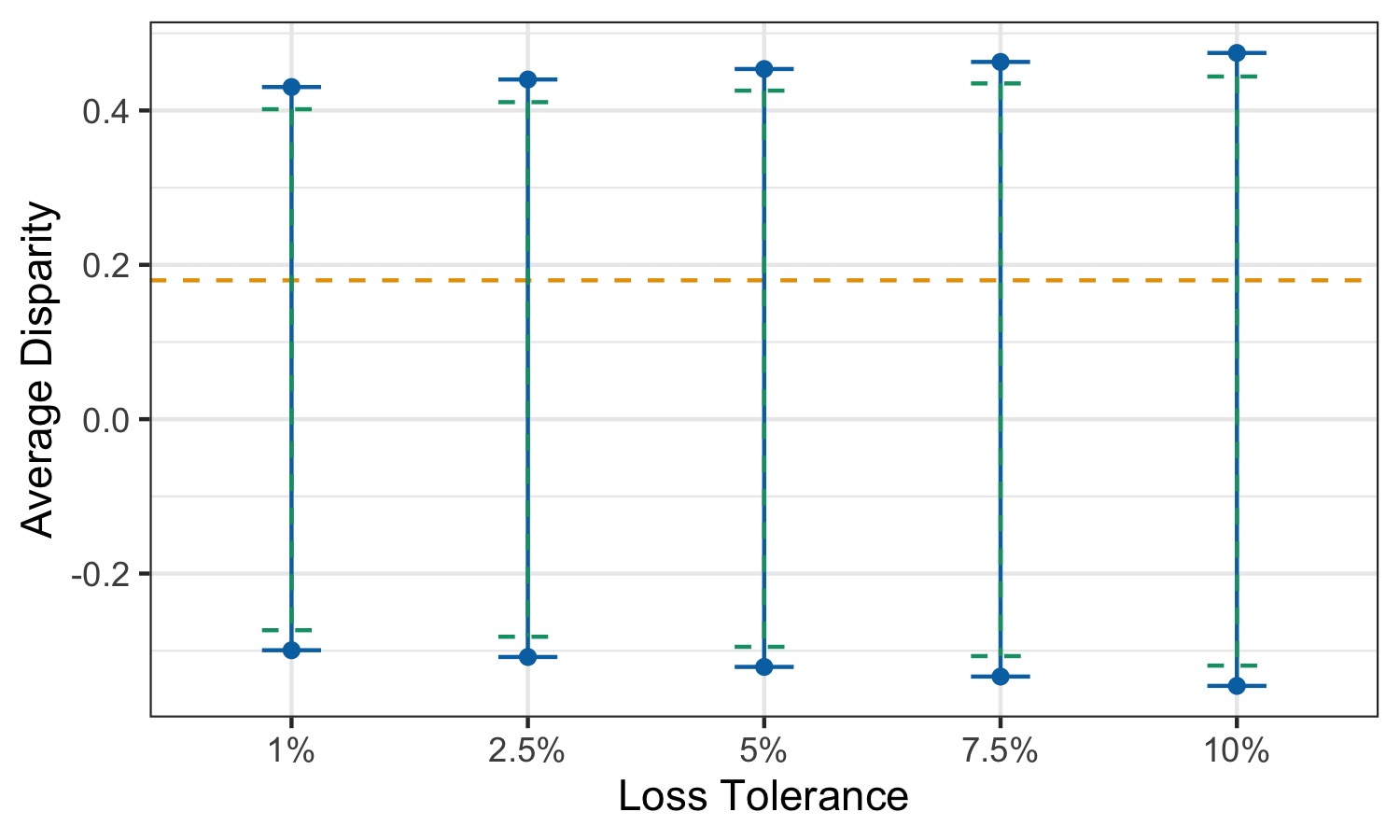}
         \caption{Statistical Parity}
         \label{fig: Compas age ref model compas SP train}
     \end{subfigure}
     \hfill
     \begin{subfigure}[b]{0.4\textwidth}
         \centering
         \includegraphics[scale=0.11]{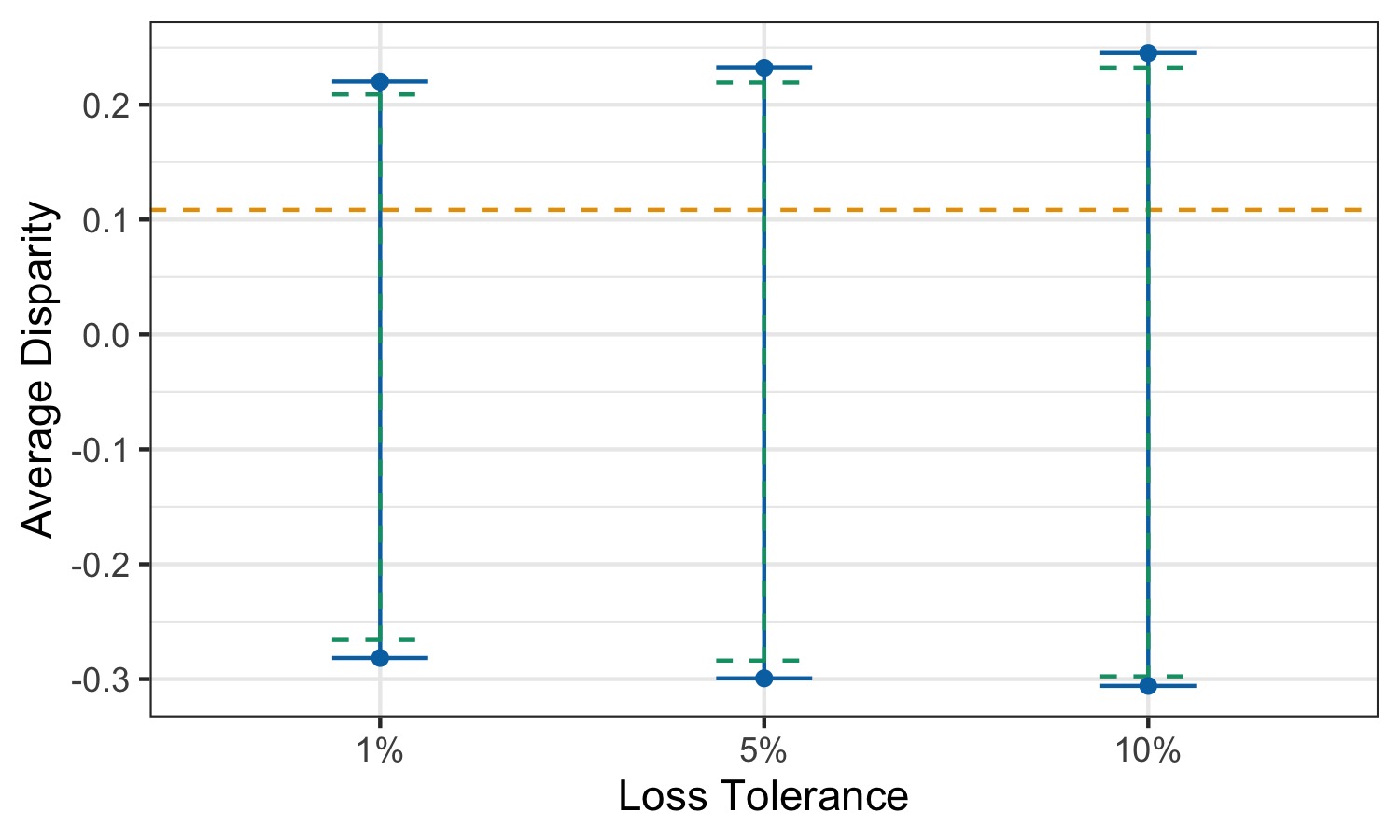}
         \caption{Balance for the Positive Class}
         \label{fig: Compas age ref model compas BFPC train}
     \end{subfigure}
     \hfill
     \begin{subfigure}[b]{0.4\textwidth}
         \centering
         \includegraphics[scale=0.11]{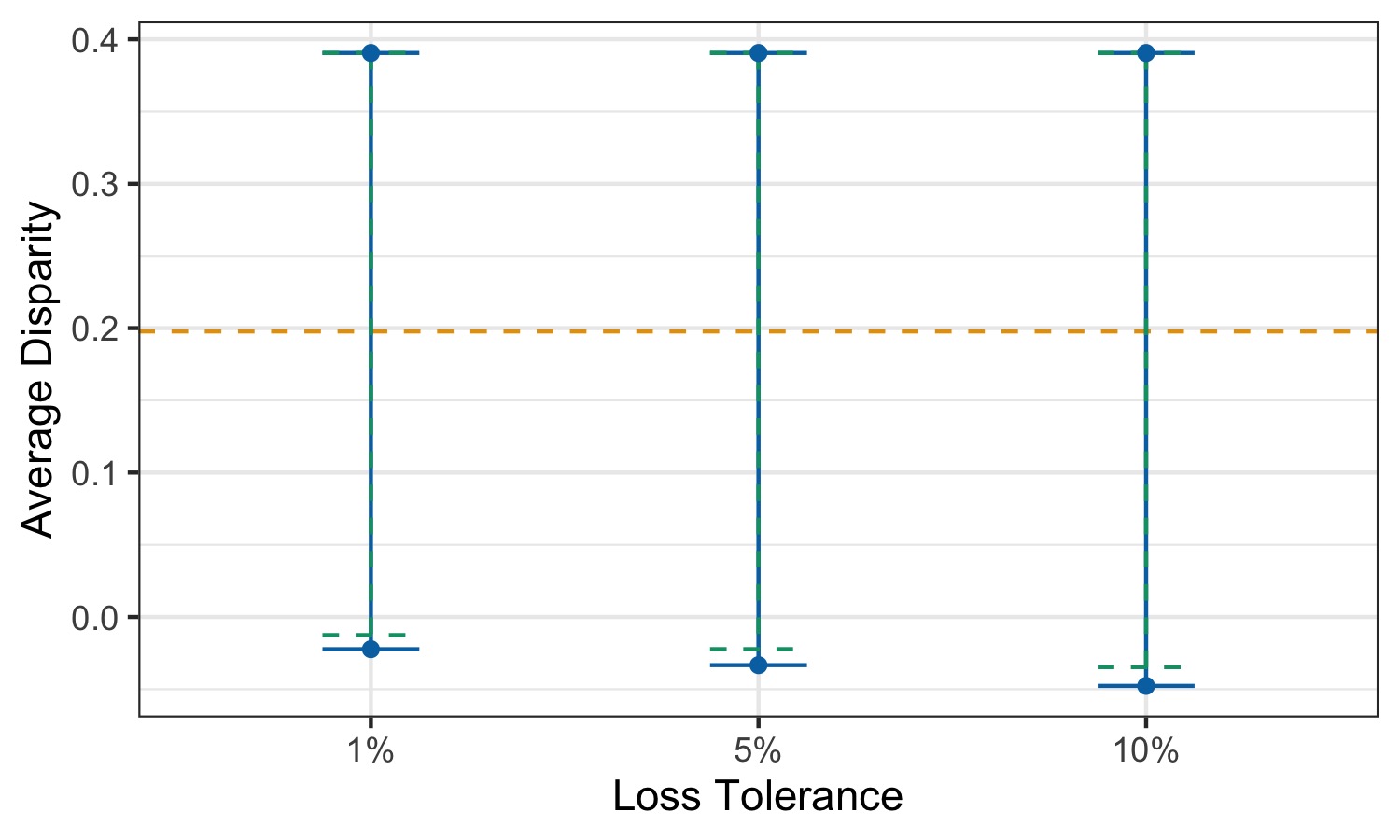}
         \caption{Balance for the Negative Class}
         \label{fig: Compas age ref model compas BFNC train}
     \end{subfigure}
\caption{The minimal and maximal disparities between young defendants ($A_i = 1$) and older defendants ($A_i = 0$) over the set of good models on the train set. We set the loss tolerance as $\epsilon = 1\%, 5\%, 10\%$ of COMPAS' training loss. The blue error bars plot the relative disparities associated with the linear program reduction (\textsection~\ref{section: linear program reduction}), the green error bars plot the relative disparities associated with the stochastic prediction function returned by Algorithm \ref{alg: exp grad for fairness frontier, extremes} and the orange dashed line plots the predictive disparity associated with COMPAS. See \textsection~\ref{section: COMPAS age experiments} of the Supplement for details.}
    \label{fig: Compas age ref model compas train}
\end{figure}

\subsection{Regression Experiments: Communities \& Crime Dataset}\label{section: regression experiments, communities & crime dataset}

The Communities \& Crime dataset \cite{Dua:2019} contains 1,994 examples. We randomly split this data 50\%-50\% into a train and test set. We train models to predict the violent crime rate within each community (the number of violent crimes per 100,000 people), which is a continuous outcome. We evaluate models using least squares loss, define the benchmark model to be the loss-minimizing linear regression and focus on the statistical parity measure of predictive disparities between communities that are majority white vs. majority non-white. We use \fairs to search for the predictive disparity minimizing linear regression that achieves a loss that is comparable to the benchmark (loss tolerance $\epsilon = 1\%, 5\%, 10\%$ of the loss-minimizing linear regression). In this dataset, there is no selective labels problem, and so we construct the \fairs model following approach detailed in \textsection~\ref{section: reductions approach} and Supplement \textsection~\ref{section: computing abs disp min model}.
 
Table \ref{table: communities and crime, disparities and loss over test set} summarizes both the predictive disparities and least squares losses over the test set of the $\fairs$ models and the benchmark linear regression. The \fairs models achieve comparable performance to the test loss of the benchmark loss-minimizing linear regression while producing lower predictive disparities. These results highlight that our proposed methods continue to perform well in regression tasks.

\begin{table}[t]
\caption{
The \fairs models over the set of good models (performing within 1\%, 5\%, and 10\% of the loss-minimizing linear regression's training loss) achieve comparable performance to the test loss of the loss-minimizing linear regression and produce lower predictive disparities. The first column reports the disparity in average predictions between majority white and majority non-white communities (Def. \ref{definition: Statistical parity}). The second column reports the test losses for each model. See \textsection~\ref{section: regression experiments, communities & crime dataset} for details.}
\label{table: communities and crime, disparities and loss over test set}
\vskip 0.05in
\begin{center}
\begin{small}
\begin{sc}
\begin{tabular}{c c c}
\toprule
& Loss  & Disp. \\
\midrule
Benchmark & 0.0101 & -0.3386 \\ 
& (0.0007) & (0.0135) \\
\midrule
\fairs \\
$\epsilon = 1\%$ & 0.0103 & -0.2989 \\
& (0.0008) & (0.0130) \\
$\epsilon = 5\%$ & 0.0105 & -0.2856 \\ 
& (0.0008) & (0.0129) \\
$\epsilon = 10\%$ & 0.0108 & -0.2658 \\ 
& (0.0008) & (0.0127) \\
\bottomrule
\end{tabular}
\end{sc}
\end{small}
\end{center}
\vskip -0.1in
\end{table}

\end{document}